\documentclass[10pt]{article}

\usepackage{mystyle}
\RequirePackage[top=2.25cm, bottom=2.5cm, left=2.5cm, right=2.5cm, columnsep=0.65cm]{geometry}

\newcommand{\variance}{\sigma(\mathscr{A})}
\newcommand{\mean}{\tilde{\mu}_{a,t}(\mathscr{A})}
\newcommand{\bonus}{\beta(\mathscr{A})}

\title{\LARGE Optimism Stabilizes Thompson Sampling for Adaptive Inference}
\author{Shunxing Yan\thanks{Author names are listed in alphabetical order. Peking University. Emails: \texttt{\{sxyan,hanzhong\}@stu.pku.edu.cn}.} \qquad \qquad Han Zhong$^*$}
\date{}
\begin{document}
\maketitle
\begingroup
\renewcommand{\thefootnote}{}
\footnotetext[0]{Accepted for presentation at the Conference on Learning Theory (COLT) 2026.}
\endgroup

\begin{abstract}
Thompson sampling (TS) is widely used for stochastic multi-armed bandits, yet its inferential properties under adaptive data collection are subtle. Classical asymptotic theory for sample means can fail because arm-specific sample sizes are random and coupled with the rewards through the action-selection rule. We study adaptive inference for Thompson sampling with Gaussian randomized indices in \(K\)-armed stochastic bandits with independent sub-Gaussian reward noises, and identify \emph{optimism} as a key mechanism for restoring \emph{stability},
meaning that each arm's pull count concentrates around a deterministic scale.
This stability yields asymptotically valid Wald inference despite adaptive
sampling. First, we prove that variance-inflated TS is stable for any $K \ge 2$, including the challenging regime where multiple arms are optimal, with
asymptotically uniform allocation over optimal arms and sharp logarithmic
pull-count asymptotics for suboptimal arms.
 This resolves the \(K\)-armed extension question raised by
\citet{halder2025stable}, using new winner-map and Lyapunov-drift techniques to
control allocation among multiple optimal arms. Second, we analyze an alternative optimistic modification that keeps the Gaussian index variance unchanged but adds an explicit mean bonus to the index center, and establish a similar stability conclusion.
In summary, suitably implemented optimism stabilizes Thompson sampling and enables asymptotically valid Wald inference in multi-armed bandits, while incurring only a mild additional regret cost.
\end{abstract}

\setcounter{tocdepth}{2}
\hypersetup{linkcolor=bluee}
\tableofcontents
\hypersetup{linkcolor=redd}

\newpage

    \section{Introduction}

Adaptive experiments increasingly use bandit algorithms not only to allocate
treatments efficiently, but also to support post-experiment inference about
unknown arm means, which underlie treatment-effect comparisons in online
experiments
\citep{zhang2020inference,hadad2021confidence,bibaut2025demystifying}.
After the experiment, practitioners often compare arms, report confidence
intervals, or decide whether a treatment should be deployed. This inferential
goal is delicate: the same reward fluctuations that guide future allocation
decisions also enter the final estimators and their normalizations.
Consequently, classical fixed-design arguments for confidence intervals and
hypothesis tests need not apply under adaptive data collection.
Multi-armed bandits~\citep{thompson1933likelihood,robbins1952some,lai1985asymptotically}
provide a canonical model for this tension between sequential decision-making
and statistical inference.

The distinction between regret and inference is fundamental. Regret evaluates
the reward loss incurred by the chosen arms. By contrast, Wald inference depends
on the realized arm-specific sample sizes, which are part of the inferential
design because they determine the normalization of the sample means. This issue
is especially pronounced when the optimal arm is not unique, the limiting case
in which regret is indifferent among several arms but inference is not.
Stability, introduced by~\citet{lai1982least}, is the property that turns an
adaptively generated data set into an asymptotically regular design: although
the allocations are random, each arm's realized pull count concentrates around a
deterministic benchmark. When stability holds, the random sample size can be
replaced by a deterministic scale, enabling asymptotic normality of studentized
sample means and valid Wald-type confidence intervals. Recent work has shown
that several UCB-type algorithms satisfy this stability condition
\citep{fan2022typical,khamaru2024inference,han2024ucb}.

We focus on Thompson sampling~\citep[TS;][]{thompson1933likelihood}, also known
as posterior sampling. TS is particularly popular due to its simplicity, strong
empirical performance, and Bayesian interpretation, and much of its theoretical
literature focuses on regret minimization and frequentist regret guarantees
\citep{agrawal2012analysis,kaufmann2012bayesian,russo2014learning,
russo2016information,agrawal2017near,daniel2018tutorial}. From an inferential viewpoint, however, TS is more subtle. Prior work has shown
that vanilla TS can lead to unstable allocation paths and can undermine
classical normal-approximation-based inference
\citep{kalvit2021closer,halder2025stable} (see also Section~\ref{sec:vanilla-feedback} for a simple two-armed equal-means example to illustrate this instability problem).
As a result, the realized pull counts may fail to concentrate around
deterministic scales, particularly when the optimal arm is not unique. To address this issue,
\citet{halder2025stable} proposed a variance-inflated TS rule and established
stability guarantees in the two-armed Gaussian bandit. Extending these
guarantees to general \(K\)-armed bandits is substantially more challenging,
especially when several arms are optimal. With more than two optimal arms,
stability is no longer a one-dimensional balancing problem; the algorithm must
stabilize an allocation vector over the optimal set while the randomized winner
probabilities remain history-dependent. This leads to the central question of
the paper:
\begin{center}
    Can Thompson sampling be modified to preserve its randomized, regret-efficient behavior while enforcing the pull-count stability needed for asymptotically valid Wald inference?
\end{center}

\subsection{Contributions}

This paper answers the central question affirmatively and identifies a unifying
principle: \textbf{optimism stabilizes Thompson sampling for adaptive
inference}. Intuitively, optimism makes empirical-mean perturbations lower
order relative to the optimistic exploration terms that drive future
allocations. We implement this idea in two ways: by inflating the variance of
the Gaussian randomized indices, or by adding an explicit upper-confidence-type
bonus to the index center. In both cases, the resulting TS rule is stable in the
sense of Definition~\ref{def:stability}, and this stability restores classical
Wald inference under adaptive sampling.

Specifically, we study fixed-\(K\) stochastic bandits with independent mean-zero
sub-Gaussian reward noises and Gaussian randomized indices. Let
\(\mathcal S^\star\) be the set of optimal arms, let
\(m=|\mathcal S^\star|\), and let \(\Delta_a=\mu^\star-\mu_a\) for
\(a\notin\mathcal S^\star\). Under the respective growth conditions on
\(\variance\) and \(\bonus\) stated below, for the two optimistic TS rules
studied here, we prove the stability law
\begin{equation}\label{eq:stability:intro}
\frac{N_{a,T}}{N^\star_{a,T}}\xrightarrow{\mathbb P}1,
\qquad
N^\star_{a,T}=
\begin{cases}
T/m, & a\in\mathcal S^\star,\\[2pt]
2\variance\log(T/\variance)/\Delta_a^2,
& a\notin\mathcal S^\star \text{ under variance inflation},\\[2pt]
2\bonus\log T/\Delta_a^2,
& a\notin\mathcal S^\star \text{ under the mean-bonus rule}.
\end{cases}
\end{equation}
Thus the optimal arms share the horizon asymptotically uniformly, while each
suboptimal arm is sampled on a sharp, gap-dependent logarithmic scale. This
deterministic pull-count normalization is precisely the stability condition
needed for classical Wald inference under adaptive sampling. Our contributions are fourfold.
\begin{itemize}[leftmargin=*]
    \item \textbf{Variance-inflated TS in general \(K\)-armed bandits.}
    We prove that variance-inflated TS is stable for general \(K\)-armed bandits,
    allowing any number of optimal arms (Theorem~\ref{thm:main}). This extends
    the variance-inflation stability theory of~\citet{halder2025stable} beyond
    the two-armed case to genuinely multidimensional regimes, where the
    algorithm must balance an allocation vector over the optimal set rather than
    a single scalar proportion. The theorem also applies under the broad growth
    condition \(\variance\to\infty\) and
    \(\variance(\log T)^2/T\to0\), and gives the sharp suboptimal centering
    \(2\variance\log(T/\variance)/\Delta_a^2\).

    \item \textbf{A second optimistic TS rule through a mean bonus.}
    We show that stabilization does not require inflating the Gaussian index
    variance. We analyze a mean-bonus TS rule that keeps the vanilla Gaussian
    sampling variance but shifts the index center by an explicit
    upper-confidence-type bonus (Theorem~\ref{thm:main:2}). Under
    \(\bonus\to\infty\) and \(\bonus\log T/T\to0\), this rule satisfies the
    same optimal-arm balancing property and pulls each suboptimal arm on the
    sharp scale \(2\bonus\log T/\Delta_a^2\). This connects optimistic sampling
    ideas to adaptive inference: here optimism regularizes the realized
    allocation path, not only exploration for regret minimization.

    \item \textbf{Wald inference and the price of stability.}
The stability results yield standard Wald inference from adaptively collected
data (Theorem~\ref{thm:inference}). For each fixed arm, the usual studentized
sample mean is asymptotically normal under either optimistic TS rule, and hence
the usual Wald confidence interval has asymptotically correct coverage. The same
stability argument also justifies standard two-sample Wald comparisons between
fixed pairs of arms. Importantly, the estimator is the ordinary sample mean with
the ordinary armwise sample variance; no additional debiasing or reweighting is
required. The price of inferential regularity is explicit and tunable: the
suboptimal exploration scale is governed by
\(\variance\log(T/\variance)\) under variance inflation and by \(\bonus\log T\)
under the mean-bonus rule (Theorem~\ref{thm:regret-main}). Since the conditions
allow \(\variance\) and \(\bonus\) to diverge arbitrarily slowly, for example at
the \(\log\log\log T\) scale, stability can be obtained with only an arbitrarily
slowly diverging multiplicative overhead relative to the usual logarithmic
suboptimal-sampling scale.

    \item \textbf{New techniques for randomized allocation dynamics.}
    The proofs require new techniques for controlling randomized, history-dependent
    allocation probabilities. For
    variance-inflated TS, we introduce a winner-map viewpoint for Gaussian index
    competition, prove a negative-feedback property on the simplex, and combine
    it with averaged perturbation control, Lyapunov drift, and geometric
    waiting-time couplings. For the mean-bonus rule, we use a corridor
    argument, an order-flip mechanism among optimal arms, and a terminal-block
    contradiction. These tools separate the geometry of randomized winner
    probabilities from empirical-mean perturbations, a decomposition that may be
    useful for studying stability of other randomized bandit algorithms. See Section~\ref{sec:proof:sketch} for more details.
\end{itemize}

To the best of our knowledge, these results provide the first Wald-validity
guarantees for Thompson sampling type algorithms in general \(K\)-armed bandits (including multiple-optimal-arm
regimes).

\subsection{Related Work}

\paragraph{Adaptive inference for bandit algorithms.}
Classical bandit theory primarily emphasizes regret and sample-complexity guarantees~\citep{lattimore2020bandit}, while comparatively less attention has been devoted to statistical inference under adaptive data collection. A recent line of work addresses this gap via the notion of \emph{stability}~\citep{lai1982least}, which provides sufficient conditions under which studentized estimators remain asymptotically normal even when effective sample sizes are random and history-dependent. For UCB algorithms~\citep{lai1985asymptotically,lai1987adaptive,auer2002finite} in the multi-armed bandit setting,~\citet{fan2022typical,khamaru2024inference,han2024ucb} establish stability and derive asymptotically valid inference.~\citet{fan2024precise} study the UCB-V algorithm and identify potential instability for this method.
 Beyond the multi-armed setting,~\citet{fan2025statistical} study linear bandits and provide inference guarantees for LinUCB~\citep{li2010contextual,abbasi2011improved}, and related stability-based approaches have also been developed in broader contextual or adversarial regimes~\citep{praharaj2025avoiding}. Our focus differs in that we study Thompson sampling, whose random Gaussian indices yield qualitatively different stability behavior.
 More broadly, our work is also related to~\citet{dimakopoulou2021online} and~\citet{howard2021time}.
~\citet{dimakopoulou2021online} modify the TS decision rule itself, using adaptive reweighting and doubly robust estimation with the goal of improving
 online decision performance.~\citet{howard2021time} develop general time-uniform, nonparametric confidence sequences based on concentration
 inequalities. We use a related time-uniform concentration idea only as a
 perturbation-control tool; our inferential object remains fixed-horizon Wald
 inference along the realized trajectory of optimistic TS.

\paragraph{Thompson sampling and stability.}
Thompson sampling dates back to~\citet{thompson1933likelihood} and has been extensively studied from a regret-minimization perspective~\citep{agrawal2012analysis,kaufmann2012bayesian,russo2014learning,russo2016information,agrawal2017near,daniel2018tutorial,dimakopoulou2021online,jin2021mots,jin2022finite,jin2023thompson}. From an adaptive inference viewpoint, however, TS can exhibit pathological behavior that undermines classical central limit theorem based inference~\citep{kalvit2021closer}. %
Our work is most closely related to~\citet{halder2025stable}, which
characterizes instability of vanilla Gaussian TS and proposes a stabilized
variant via variance inflation, establishing stability in the two-armed Gaussian
bandit.  %
For context, the current version of their paper extends the analysis to
\(K\)-armed Gaussian bandits with at most two optimal arms. The
variance-inflation part of our work overlaps with this setting, but also covers
any fixed number of optimal arms, including the regime \(|\mathcal S^\star|\ge3\), where balancing the optimal arms requires new techniques.
Our variance-inflation theorem also applies under a
much weaker growth condition and a more general sub-Gaussian reward model; see Theorem~\ref{thm:main} and the discussion
following it for the precise comparison.
We additionally analyze an optimistic mean-bonus variant that restores stability without inflating the Gaussian index variance, thereby isolating optimism as a key mechanism for stabilizing Thompson sampling algorithms. Finally, we note a concurrent and independent work by~\citet{han2026thompson}, which provides a sharp characterization of the arm-pull dynamics and adaptive inference for standard TS in multi-armed bandits. In particular, their results suggest that standard TS is not stable when the optimal arm is not unique. Our results are complementary: rather than analyzing the limiting behavior of vanilla TS, we focus on optimistic TS variants that enforce stability and thereby recover classical Wald-type asymptotic normality for studentized sample means.

\paragraph{Optimism and optimistic posterior sampling.}
Optimism is a classical exploration principle in bandits, most prominently instantiated via upper confidence bounds~\citep{auer2002finite}. In Bayesian and posterior-sampling settings, optimism can arise through posterior-quantile indexing (Bayes-UCB)~\citep{kaufmann2012bayesian}, optimistic Bayesian sampling rules~\citep{may2012optimistic}, ensemble- or multi-sample-based procedures~\citep{fonteneau2013optimistic,agrawal2017optimistic,lu2017ensemble}, and more recent Feel-Good Thompson sampling~\citep{zhang2022feel}. Our mean bonus variant TS is closely related to these optimistic posterior-sampling methods, and we show that it is stable for general $K$-armed bandits. In addition, we prove stability of the variance inflated variant TS, indicating that variance inflation is another effective way to achieve optimism and ensure stability.
The recent work of~\citet{zhu2026adaptive} is particularly close algorithmically due to the variance inflation inside Gaussian Thompson sampling. Their variance inflation scheme is adaptive and is analyzed for finite-time efficiency, safety, robustness, and related post-experiment guarantees, whereas our variance-inflated rule is studied through Lai--Wei stability and sharp pull-count asymptotics in order to recover classical Wald validity.

\paragraph{Other related work.} Broadly speaking, our work is also related to statistical inference for other online algorithms, such as stochastic approximation and stochastic gradient descent~\citep{polyak1992acceleration,chen2020statistical,li2023online,su2023higrad,shen2026sgd}. Moreover, since bandits are a special case of reinforcement learning, our work is related to statistical inference for reinforcement learning algorithms~\citep{li2023statistical,wu2024statistical,wu2025uncertainty} and to posterior sampling methods for reinforcement learning~\citep{russo2014learning,agrawal2017optimistic,osband2016generalization,lu2017ensemble,zanette2020frequentist,dann2021provably,zhong2022gec,agarwal2022model}.

\subsection{Notations}
For any integer $K$, we use $[K]:=\{1,\dots,K\}$ to denote the set of $K$ arms. We use the standard asymptotic notations $O(\cdot)$, $o(\cdot)$, and $\Theta(\cdot)$. For functions $f,g:\mathbb{N}\to\mathbb{R}_{\ge 0}$, we write $f(n)=O(g(n))$ if there exist constants $c>0$ and $n_0$ such that $f(n)\le c\,g(n)$ for all $n\ge n_0$. We write $f(n)=o(g(n))$ if $f(n)/g(n)\to 0$ as $n\to\infty$. We write $f(n)=\Theta(g(n))$ if there exist constants $c_1,c_2>0$ and $n_0$ such that $c_1 g(n)\le f(n)\le c_2 g(n)$ for all $n\ge n_0$.
Define the standard normal density, cumulative distribution function, and upper tail function by
    \[ \label{eq:def:gaussian}
    \phi(x):=\frac{1}{\sqrt{2\pi}}\exp\{-x^{2}/2\},\qquad
    \Phi(x):=\int_{-\infty}^{x}\phi(u)\,du, \qquad \overline\Phi(x):=1-\Phi(x).
    \]
    Throughout the paper, \(C,c>0\) denote generic positive constants whose values may change from line to line and may depend only on fixed problem parameters unless otherwise specified.

    \section{Preliminaries}

\subsection{Multi-armed Bandit}
We consider a stochastic $K$-armed bandit with some constant $K\ge 2$.
At each time \(t\), the reward from arm \(a\in[K]\) is
\[
R_{a,t}=\mu_a+\xi_{a,t},
\qquad
\xi_{a,t}\overset{\mathrm{i.i.d.}}{\sim}\xi,
\] The learner observes only
the selected reward \(R_{A_t,t}\) after choosing \(A_t\).
The non-degenerate noise \(\xi\) is mean-zero and \(1\)-sub-Gaussian. The mean vector
\((\mu_1,\ldots,\mu_K)\in\mathbb R^K\) is unknown to the learner and needs to be
estimated from the data. Let $\mu^\star:=\max_{a\in[K]}\mu_a$ be the maximum mean and denote the optimal arm set and its size by
\[
\mathcal S^\star:=\{a:\mu_a=\mu^\star\},\qquad m:=|\mathcal S^\star|.
\]
For each suboptimal arm $a\notin\mathcal S^\star$, define the gap $\Delta_a:=\mu^\star-\mu_a>0$.
We write
\[
N_{\mathrm{opt}}(t):=\sum_{i\in\mathcal S^\star}N_{i,t},
\qquad
N_{\mathrm{sub}}(t):=\sum_{i\notin\mathcal S^\star}N_{i,t}=t-N_{\mathrm{opt}}(t),
\]
for the total numbers of optimal-arm and suboptimal-arm pulls up to time \(t\).

An algorithm $\mathscr{A}$ interacts with the bandit for $T$ rounds by selecting arms $\{A_t\}_{t=1}^T$, observing rewards $\{R_{A_t,t}\}_{t=1}^T$, and updating its decisions based on past observations. Its performance is measured by the expected cumulative regret
\[ \label{eq:def:regret}
\mathcal{R}(T) := T\mu^\star - \mathbb{E}\Big[\sum_{t=1}^T R_{A_t,t}\Big].
\]
Let $(A_t)_{t\ge 1}$ denote the sequence of pulled arms and let $\mathcal{F}_t:=\sigma(A_1,R_{A_1,1},\dots,A_t,R_{A_t,t})$ be the filtration generated by the history up to time $t$.
For each arm $a \in [K]$ and time $t$, define the number of pulls and empirical mean by
\[ \label{eq:def:pull:time}
N_{a,t}:=\sum_{s=1}^t \mathbf{1}\{A_s=a\},\qquad
\widehat\mu_{a,t}:=
\frac{1}{N_{a,t}}\sum_{s=1}^t
\mathbf{1}\{A_s=a\}R_{A_s,s}
\quad (N_{a,t}\ge 1).
\]
When $N_{a,t}=0$ we leave $\widehat\mu_{a,t}$ undefined, but this case will not arise after initialization.

\subsection{Adaptive Inference and Stability}

In online bandit problems, data are collected adaptively: the action at time $t$ depends on past observations, and hence the samples are no longer i.i.d.. In particular, classical central limit theorems for fixed designs do not directly apply to estimators formed from adaptively chosen observations. A useful sufficient condition in this setting is \emph{stability}, introduced by the seminal work of~\citet{lai1982least}, which requires that each arm's sample size concentrates around a deterministic sequence.

\begin{definition}[Stability~\citep{lai1982least}] \label{def:stability}
A bandit algorithm $\mathscr{A}$ is \emph{stable} if there exist deterministic sequences $N^\star_{a,T}$ such that for every arm $a \in [K]$,
\[
\frac{N_{a,T}}{N^\star_{a,T}}\xrightarrow{\mathbb{P}}1 \qquad \text{and} \qquad  N_{a, T}^\star \to \infty,
\]
as $T \to \infty$. Here $N_{a,T}$ (defined in~\eqref{eq:def:pull:time}) denotes the number of times arm $a$ is pulled up to time $T$ by executing the algorithm $\mathscr{A}$.
\end{definition}

Stability provides deterministic normalizations for the random sample sizes $N_{a,T}$, a key ingredient in establishing asymptotic distributions for estimators under adaptive sampling.

\begin{proposition}[Stability implies asymptotic normality~\citep{lai1982least}] \label{prop:inference}
If a bandit algorithm $\mathscr{A}$ is stable, then for any arm $a \in [K]$,
\[
\frac{\sqrt{N_{a,T}} \cdot (\widehat\mu_{a,T}-\mu_a)}{\widehat\sigma_{a,T}}
\xrightarrow{\mathcal{D}} \mathcal{N}(0,1)
\qquad \text{as } T\to\infty,
\]
where $\widehat\mu_{a,T}$ is the sample mean defined in~\eqref{eq:def:pull:time} and $\widehat\sigma_{a,T}$ is the sample standard deviation of the rewards for arm $a$ up to time $T$:
\[ \label{eq:def:sample:variance}
\widehat{\sigma}_{a,T}
:=
\sqrt{
\frac{1}{N_{a,T}-1}
\sum_{s=1}^T
\mathbf{1}\{A_s=a\}
\left(R_{A_s,s}-\widehat{\mu}_{a,T}\right)^2
}
\]
with the convention that $\widehat{\sigma}_{a,T}=1$ when $N_{a,T}=1$.
\end{proposition}

\begin{proof}[Proof of Proposition~\ref{prop:inference}]
    Under stability (Definition~\ref{def:stability}), the centered reward sequence for a fixed arm can be written as a martingale difference array. The claim then follows from a martingale central limit theorem and Slutsky's theorem. See~\citet{lai1982least} and~\citet{khamaru2024inference} for details.
\end{proof}

    \section{Thompson Sampling and Optimism}\label{sec:ts-optimism}

To avoid zero pull counts in~\eqref{eq:def:pull:time}, we assume an initialization phase in which each arm is pulled once (so $N_{a,K}=1$ for all $a$). For \(t\ge K\), we consider TS-type algorithms that draw conditionally independent Gaussian indices and act greedily:
\begin{equation}\label{eq:ts:general}
\theta_{a,t+1}\mid\mathcal{F}_t \sim \mathcal{N}\left(\mean,\frac{\variance}{N_{a,t}}\right),
\qquad
A_{t+1}\in\operatorname*{\argmax}_{a}\theta_{a,t+1}.
\end{equation}
Here \(\mean\) and \(\variance\) specify, respectively, the index center and the index-variance multiplier used to generate the Gaussian randomized indices.

\subsection{Vanilla TS Instability and Optimism: An Illustrative Example}
\label{sec:vanilla-feedback}

Prior work has shown that vanilla TS can be unstable from an adaptive inference
viewpoint~\citep{kalvit2021closer,halder2025stable}. Here we use a simple two-armed bandit example to isolate the feedback mechanism and to illustrate why optimism may help. We focus on a
simplified two-armed Gaussian bandit with equal means \(\mu_1=\mu_2=\mu\) and
unit noise variance.

\vspace{4pt}
\noindent\textbf{Instability of Vanilla TS.} Under vanilla TS, the index for arm \(a\) is drawn as
$\theta_{a,t+1}\mid\mathcal{F}_t \sim \mathcal{N}\left(\widehat\mu_{a,t},1/N_{a,t}\right)$, and thus
$\theta_{1,t+1}-\theta_{2,t+1}\mid\mathcal F_t \sim \mathcal N\left(\widehat\mu_{1,t}-\widehat\mu_{2,t},N_{1,t}^{-1}+N_{2,t}^{-1}\right)$.
Therefore, the probability of pulling arm 1 in the next time step is
\[
\mathbb P(A_{t+1}=1\mid\mathcal F_t)
=
\Phi\left(
\frac{\widehat\mu_{1,t}-\widehat\mu_{2,t}}
{\sqrt{N_{1,t}^{-1}+N_{2,t}^{-1}}}
\right)
=
\Phi\left(
\sqrt{\frac{N_{2,t}}{t}}\,W_{1,t}
-
\sqrt{\frac{N_{1,t}}{t}}\,W_{2,t}
\right),
\]
where the second equality uses \(N_{1,t}+N_{2,t}=t\), and $W_{a,t}:=\sqrt{N_{a,t}}(\widehat\mu_{a,t}-\mu)$. This display makes the feedback explicit. Stability would require the
pull-count proportions \(N_{1,t}/t\) and \(N_{2,t}/t\) to settle around
deterministic limits. If this were the case, the coefficients multiplying
\(W_{1,t}\) and \(W_{2,t}\) in the displayed allocation probability would also
settle around deterministic constants. However, even in this simplified Gaussian
case, \(W_{1,t}\) and \(W_{2,t}\) remain order-one fluctuations rather than
shrinking to zero. Hence the standardized contrast determining the
next allocation probability need not vanish. This creates a feedback loop: the
same empirical noise that enters the final Wald statistic also steers future
sample sizes, so vanilla TS may fail to produce pull counts concentrated around
deterministic benchmarks.

\vspace{4pt}
\noindent\textbf{Why Optimism May Help.}
The same calculation suggests how optimism may weaken this feedback and stabilize the realized pull counts. Both
modifications below are \emph{optimistic} because they make uncertain or under-sampled
arms more competitive than they would be under the vanilla TS. Since the Gaussian index in~\eqref{eq:ts:general} is determined by an index center and an index-variance multiplier, there are two natural ways to modify it.

\begin{itemize}[leftmargin=*]
\item The first is to inflate the index variance. This is optimistic in a
distributional sense: a larger index variance increases the chance that an
uncertain arm generates a high index draw. In the
two-armed calculation, multiplying the sampling variance by \(\variance\) gives
\[
\mathbb P(A_{t+1}=1\mid\mathcal F_t)
=
\Phi\left(
\frac{1}{\sqrt{\variance}}
\left(
\sqrt{\frac{N_{2,t}}{t}}\,W_{1,t}
-
\sqrt{\frac{N_{1,t}}{t}}\,W_{2,t}
\right) \right),
\]
so the empirical-noise feedback is weakened as \(\variance\to\infty\).

\item The second is to shift the index center. This is optimistic in an
upper-confidence sense: if arm \(a\) receives a bonus \(B_{a,t}\), then the same
two-armed
calculation gives
\[
\mathbb P(A_{t+1}=1\mid\mathcal F_t)
=
\Phi\left(
\frac{\widehat\mu_{1,t}-\widehat\mu_{2,t}+B_{1,t}-B_{2,t}}
{\sqrt{N_{1,t}^{-1}+N_{2,t}^{-1}}}
\right).
\]
Choosing \(B_{a,t}\) larger for arms with smaller \(N_{a,t}\) gives
under-sampled arms an explicit advantage and pushes the allocation back toward
balance.
\end{itemize}

Thus, the calculation points to two natural optimistic modifications of vanilla
TS: increasing the index variance to damp empirical-noise feedback, and shifting
the index center in favor of under-sampled arms to create a balancing force.
The next subsection formalizes these two ideas as the variance-inflated and
mean-bonus TS rules.

\subsection{Optimistic Thompson Sampling}

Motivated by the last subsection, we study two complementary ways to implement optimism in TS.
\begin{enumerate}[leftmargin=*]
    \item \textbf{Variance inflation}. This is introduced by~\citet{halder2025stable}:
    \[
\mean = \widehat{\mu}_{a,t},\qquad \variance > 1.
\]
We allow $\variance$ to depend on the horizon $T$ and satisfy the growth condition:
\begin{equation}\label{eq:gamma-growth}
{\variance}\to\infty
\qquad\text{and}\qquad
\variance (\log T)^2 / T \to 0
\qquad(T\to\infty),
\end{equation}
which is milder than the growth condition in~\citet{halder2025stable}.
More precisely,~\citet{halder2025stable} assume \(\variance/\log\log T\to\infty\) and \(\sqrt{\log T}/\variance\to\infty\). Our condition~\eqref{eq:gamma-growth} only asks that \(\variance\) diverge and remain \(o(T/(\log T)^2)\), so both the lower-growth requirement and the upper-growth restriction on \(\variance\) are relaxed.

The term ``optimism'' can be made explicit here: inflating the Gaussian index variance increases the probability of
drawing an ``optimistic'' (high) index.
Indeed, if $Z\sim\mathcal{N}(0,1)$ and $c>\widehat{\mu}_{a,t}$, then
\[
\mathbb{P}\left(\widehat{\mu}_{a,t}+\sqrt{\frac{\variance}{N_{a,t}}}Z \ge c \ \middle|\ \mathcal{F}_t \right)
=
1-\Phi\left((c-\widehat{\mu}_{a,t})\sqrt{\frac{N_{a,t}}{\variance}}\right),
\]
which is increasing in $\variance$.
Thus, variance inflation implements a form of \emph{distributional optimism} by making upper-tail events more likely.

\item \textbf{Mean bonus.} The second route keeps the Gaussian index variance unchanged as in vanilla Gaussian TS but adds an explicit optimism bonus to the index center. Concretely, at time \(t\ge K\), draw the Gaussian index
$\tilde{\theta}_{a,t+1}\mid\mathcal{F}_t \sim \mathcal{N}(\widehat{\mu}_{a,t}, 1/N_{a,t})$ and then select
	\begin{equation}\label{eq:ts:mean:1}
	A_{t+1}\in\operatorname*{\argmax}_{a\in[K]}\Big\{\tilde{\theta}_{a,t+1}+B_{a,t}\Big\},
	\qquad \text{where } \,
	B_{a,t}:=\sqrt{\frac{2\bonus\log T}{N_{a,t}}},
	\end{equation}
where $\bonus$ is a deterministic \emph{bonus strength} parameter that controls the amount of optimism injected into the index. We set $\bonus$ to satisfy the following growth condition:
\[ \label{eq:bonus-growth}
\bonus \to \infty \qquad \text{and} \qquad (\bonus \log T) / T \to 0 \qquad (T \to \infty).
\]
When $\bonus=1$, the bonus $B_{a,t}$ coincides with the standard UCB bonus, and~\citet{khamaru2024inference} show that this choice is sufficient to obtain stability for deterministic UCB-type policies. Equivalently, the mean-bonus rule fits the general TS form~\eqref{eq:ts:general} with $(\mean,\variance) = (\widehat{\mu}_{a,t}+B_{a,t}, 1)$, and its action rule is~\eqref{eq:ts:mean:1}. Because these indices are randomized, in contrast to deterministic UCB-type policies, we require a slowly diverging $\bonus$ to control the additional sampling noise. This choice is crucial for stability of TS with mean bonus. Thus, the resulting indices are distributed as
	\begin{equation}\label{eq:ts:mean:2}
	\theta_{a, t+1} \mid \mathcal{F}_t \sim \mathcal{N}\left(\widehat{\mu}_{a,t}+B_{a,t}, \frac{1}{N_{a,t}}\right) \qquad
	A_{t+1}\in\operatorname*{\argmax}_{a \in [K]}\theta_{a,t+1}.
	\end{equation}
We highlight three complementary interpretations that justify this choice.

\smallskip
\noindent\emph{(i) Posterior quantile (Bayes-UCB).}
For a Gaussian posterior $\mathcal{N}(\widehat{\mu}_{a,t},1/N_{a,t})$, the $(1-\alpha)$-upper quantile equals
$\widehat{\mu}_{a,t}+(\Phi^{-1}(1-\alpha))/\sqrt{N_{a,t}}$.
Taking $\alpha=T^{-\bonus}$ yields $\Phi^{-1}(1-\alpha)\approx \sqrt{2\bonus\log T}$, so
$\widehat{\mu}_{a,t}+B_{a,t}$ matches an extreme upper posterior quantile, closely related to quantile-based
Bayesian optimism such as Bayes-UCB~\citep{kaufmann2012bayesian}.

\smallskip
\noindent\emph{(ii) Multi-sample optimism.}
A standard way to enforce optimism in posterior sampling is to draw multiple posterior samples and act optimistically with respect to them~\citep{fonteneau2013optimistic,agrawal2017optimistic,lu2017ensemble}.
If one draws $M$ i.i.d.\ posterior samples for the same arm and uses their maximum, then for Gaussian posteriors
the typical upward shift of the maximum is on the order of $\sqrt{2\log M/N_{a,t}}$.
Our choice $B_{a,t}=\sqrt{2\bonus \log T/N_{a,t}}$ corresponds to $M=T^{\bonus}$.
Thus the mean-bonus rule can be viewed as a surrogate for ``optimistic posterior sampling''
via multiple posterior draws, where optimism arises from selecting an upper-tail realization.

\smallskip
\noindent\emph{(iii) Optimistic Posterior.}
In the Gaussian family, exponential tilting of a Gaussian density by a linear factor preserves Gaussianity and
induces a mean shift.
Hence adding a positive bonus can also be interpreted as sampling from an optimistic, data-dependent
pseudo-posterior in~\eqref{eq:ts:mean:2}, connecting to optimistic TS variants such as
Optimistic Bayesian Sampling~\citep{may2012optimistic} and Feel-Good Thompson sampling~\citep{zhang2022feel}.
\end{enumerate}

Algorithm~\ref{alg:ts} gives pseudocode for the general TS framework in~\eqref{eq:ts:general}, including the different parameter choices described above.

\begin{algorithm}[t]
\caption{Thompson Sampling Framework}
\label{alg:ts}
\begin{algorithmic}[1]
\REQUIRE Horizon $T$, number of arms $K$.
\REQUIRE Choose a \emph{mode} and set parameters $(\mean, \variance)$ accordingly:
\begin{enumerate}[label=(\Alph*), leftmargin=*, nosep]
\item Vanilla TS: $\mean = \widehat{\mu}_{a,t}$ and $\variance=1$. \label{option:A}
\item TS with inflated variance: $\mean = \widehat{\mu}_{a,t}$ and $\variance$ satisfies~\eqref{eq:gamma-growth}. \label{option:B}
\item TS with mean bonus: $\mean$ is specified in~\eqref{eq:ts:mean:2} with $\bonus$ satisfying~\eqref{eq:bonus-growth} and $\variance=1$. \label{option:C}
\end{enumerate}
\STATE \textbf{Initialization:} for $a=1,\ldots,K$, pull arm $a$ once and observe the reward.
\FOR{$t = K, K+1, \ldots, T-1$}
    \FOR{$a=1,\ldots,K$}
        \STATE Compute $\mean$ based on~\eqref{eq:def:pull:time} and~\eqref{eq:ts:mean:2} for Option~\ref{option:C}.
        \STATE Sample independently $\theta_{a,t+1} \sim \mathcal{N}(\mean,\variance/N_{a,t})$.
    \ENDFOR
    \STATE Play $A_{t+1}\in\operatorname*{argmax}_{a\in[K]}\theta_{a,t+1}$.
    \STATE Observe the reward and update $\mathcal{F}_{t+1}$.
\ENDFOR
\end{algorithmic}
\end{algorithm}

\section{Theoretical Guarantees}

We present theoretical guarantees for the stability, adaptive inference, and regret of proposed TS algorithms with optimism. Recall that $\mathcal S^\star$ denotes the set of optimal arms with $m:=|\mathcal S^\star|$, and
$\Delta_a:=\mu^\star-\mu_a$ is the reward gap for each $a\notin\mathcal S^\star$.

\paragraph{Stability via variance inflation.}
We first analyze Algorithm~\ref{alg:ts} with Option~\ref{option:B}, which inflates the Gaussian index variance while keeping the index center \(\widehat\mu_{a,t}\) unchanged.

\begin{theorem}[Stability for TS with variance inflation] \label{thm:main}
    Run TS with variance inflation, i.e., Algorithm~\ref{alg:ts} with Option~\ref{option:B}, with \(\variance\) satisfying~\eqref{eq:gamma-growth}. Then, as $T\to\infty$:
\begin{enumerate}
\item[(a)] (Uniform allocation among optimal arms) For each $a \in\mathcal S^\star$,
\[
\frac{N_{a,T}}{T}\xrightarrow{\mathbb{P}}\frac{1}{m}.
\]
\item[(b)] (Sharp suboptimal asymptotics) For each $a\notin\mathcal S^\star$,
\[
\frac{N_{a,T}}{\variance\log\bigl(T/\variance\bigr)}\xrightarrow{\mathbb{P}}\frac{2}{\Delta_a^2}.
\]
\end{enumerate}
Consequently, TS with variance inflation is stable with respect to the deterministic sequences
\[
N^\star_{a,T}=
\begin{cases}
T/m, & a\in\mathcal S^\star,\\[2pt]
2\variance\log\bigl(T/\variance\bigr)/\Delta_a^2, & a\notin\mathcal S^\star.
\end{cases}
\]
\end{theorem}

\begin{proof}[Proof of Theorem~\ref{thm:main}]
    See Appendix~\ref{appendix:thm:main} for a detailed proof.
\end{proof}

Theorem~\ref{thm:main} shows that this form of optimism yields stability for any fixed \(K\ge2\), including the multiple-optimal-arm regime.
Part~(a) implies that the algorithm does not collapse onto a single optimal arm; instead, it allocates asymptotically uniformly over \(\mathcal S^\star\).
Part~(b) gives the sharp asymptotic characterization of suboptimal pulls: for every \(a\notin\mathcal S^\star\), \(N_{a,T}\sim (2\variance/\Delta_a^2)\log(T/\variance)\).
This is the precise centering for variance-inflated TS throughout the full inflation regime covered by Theorem~\ref{thm:main}.
In the narrower regime considered by~\citet{halder2025stable}, namely \(\log\log T\ll\variance\ll\sqrt{\log T}\), one has \(\log\variance=o(\log T)\) and hence \(\log(T/\variance)\sim\log T\).
Thus their \(2\variance\log T/\Delta_a^2\) scale is consistent with this centering in the overlap, while the distinction between \(\log T\) and \(\log(T/\variance)\) becomes essential in the broader regime allowed by~\eqref{eq:gamma-growth}.
Compared with~\citet{halder2025stable}, whose stability analysis covers the cases with at most two optimal arms under the narrower inflation window \(\log\log T\ll\variance\ll\sqrt{\log T}\), Theorem~\ref{thm:main} handles arbitrary fixed \(K\), including \(|\mathcal S^\star|\ge3\), under the broader condition \(\variance\to\infty\) and \(\variance(\log T)^2/T\to0\). The proof uses a winner-map viewpoint and a Lyapunov-drift argument to control the balancing of multiple optimal arms.
It also uses a geometric coupling argument to obtain the sharp suboptimal pull counts.
Section~\ref{sec:proof:sketch} gives an overview of these new proof ideas.

\paragraph{Stability via mean bonus.}
Algorithm~\ref{alg:ts} with Option~\ref{option:C} implements optimism by shifting the index center rather than inflating the index variance.
While this modification appears qualitatively different, the next theorem shows that it yields a similar stability picture to variance inflation.

\begin{theorem}[Stability for TS with mean bonus] \label{thm:main:2}
    Run TS with mean bonus, i.e.,  Algorithm~\ref{alg:ts} with Option~\ref{option:C}, with \(\bonus\) satisfying~\eqref{eq:bonus-growth}. Then, as $T\to\infty$:
\begin{enumerate}
\item[(a)] (Uniform allocation among optimal arms) For each $a \in\mathcal S^\star$,
\[
\frac{N_{a,T}}{T}\xrightarrow{\mathbb{P}}\frac{1}{m}.
\]
\item[(b)] (Sharp suboptimal asymptotics) For each $a\notin\mathcal S^\star$,
\[
\frac{N_{a,T}}{\bonus\log T}\xrightarrow{\mathbb{P}}\frac{2}{\Delta_a^2}.
\]
\end{enumerate}
Consequently, TS with mean bonus is stable with respect to the deterministic sequences
\[
N^\star_{a,T}=
\begin{cases}
T/m, & a\in\mathcal S^\star,\\[2pt]
(2\bonus\log T)/\Delta_a^2, & a\notin\mathcal S^\star.
\end{cases}
\]
\end{theorem}

\begin{proof}[Proof of Theorem~\ref{thm:main:2}]
    See Appendix~\ref{appendix:thm:main:2} for a detailed proof.
\end{proof}

Theorem~\ref{thm:main:2} shows that TS with a mean bonus exhibits a similar stability picture to variance inflation: asymptotically uniform allocation over the optimal set and sharp logarithmic sampling of each suboptimal arm. Moreover, the suboptimal sampling scale becomes $\bonus\log T$, reflecting the strength of the bonus $B_{a,t}$ in~\eqref{eq:ts:mean:1}.

\paragraph{From stability to adaptive inference.}
As pointed out by Proposition~\ref{prop:inference}, if the algorithm is stable, then standard Wald-type confidence intervals constructed from the estimated mean and variance achieve asymptotically correct coverage despite adaptive sampling.
Theorem~\ref{thm:inference} formalizes this implication for both optimistic TS variants.

\begin{theorem}[Adaptive inference]\label{thm:inference}
Run Algorithm~\ref{alg:ts} with Option~\ref{option:B} (variance inflation) or Option~\ref{option:C} (mean bonus).
Then, for any fixed arm $a\in[K]$,
\[
\frac{\sqrt{N_{a,T}}\bigl(\widehat\mu_{a,T}-\mu_a\bigr)}{\widehat\sigma_{a,T}}
\xrightarrow{\mathcal{D}} \mathcal{N}(0,1)
\qquad \text{as } T\to\infty,
\]
where $\widehat{\mu}_{a,T}$ and $\widehat{\sigma}_{a,T}$ are defined in~\eqref{eq:def:pull:time} and~\eqref{eq:def:sample:variance}, respectively.
Consequently, for any fixed $\alpha\in(0,1)$, the Wald-type confidence interval
\[
\mathrm{CI}_{a,T}(1-\alpha)
:=
\left[
\widehat{\mu}_{a,T}-\Phi^{-1}(1-\alpha/2)\frac{\widehat{\sigma}_{a,T}}{\sqrt{N_{a,T}}},
\
\widehat{\mu}_{a,T}+\Phi^{-1}(1-\alpha/2)\frac{\widehat{\sigma}_{a,T}}{\sqrt{N_{a,T}}}
\right]
\]
satisfies
\[
\lim_{T\to\infty}\mathbb{P}\left(\mu_a\in \mathrm{CI}_{a,T}(1-\alpha)\right)=1-\alpha.
\]
\end{theorem}

\begin{proof}
By Theorem~\ref{thm:main} (Option~\ref{option:B}) and Theorem~\ref{thm:main:2} (Option~\ref{option:C}), the algorithm is stable under either option. Proposition~\ref{prop:inference} then yields the stated asymptotic normality. The confidence interval coverage follows immediately.
\end{proof}

The takeaway of Theorem~\ref{thm:inference} is that, under either Option~\ref{option:B} or Option~\ref{option:C}, practitioners may form confidence intervals for each $\mu_a$ using the same studentized sample-mean recipe as in classical fixed-design settings. No additional debiasing or reweighting is required; the role of optimism is precisely to enforce the stability needed for valid inference.

\begin{remark}[Inference between two arms]
Although Theorem~\ref{thm:inference} is stated arm by arm, the same stability-to-CLT argument justifies the usual two-sample Wald statistic for comparing two fixed arms.  For any fixed \(i\ne j\),
\[
\frac{(\widehat\mu_{i,T}-\widehat\mu_{j,T})-(\mu_i-\mu_j)}
{\sqrt{\widehat\sigma_{i,T}^2/N_{i,T}+\widehat\sigma_{j,T}^2/N_{j,T}}}
\xrightarrow{\mathcal D}\mathcal N(0,1),
\]
where the denominator is positive with probability tending to one.  Hence, under \(H_0:\mu_i=\mu_j\), the standard Wald statistic
$
(\widehat\mu_{i,T}-\widehat\mu_{j,T})/
{(\widehat\sigma_{i,T}^2/N_{i,T}+\widehat\sigma_{j,T}^2/N_{j,T})^{1/2}}
$
is asymptotically standard normal. We can construct asymptotically valid confidence intervals for the difference \(\mu_i-\mu_j\) or perform asymptotically valid hypothesis tests comparing \(\mu_i\) and \(\mu_j\).
\end{remark}

\paragraph{Regret guarantees.}
We have established stability and the resulting Wald inference guarantees for both optimistic TS variants.  We now quantify the regret cost of enforcing this stability via optimism.  The next theorem shows that both optimistic TS variants retain instance-dependent logarithmic regret, with the additional factor governed by the amount of optimism.

\begin{theorem}[Regret bounds for optimistic TS]\label{thm:regret-main}
For Algorithm~\ref{alg:ts} with Option~\ref{option:B} (TS with variance inflation), it holds that for all sufficiently large \(T\),
\[
\mathcal R(T)
\le
C\sum_{a\notin\mathcal S^\star}
\frac{\variance}{\Delta_a}
\log\left(1+\frac{T\Delta_a^2}{\variance}\right).
\]
For Algorithm~\ref{alg:ts} with Option~\ref{option:C} (TS with mean bonus), it holds that for all sufficiently large \(T\),
\[
\mathcal R(T)
\le
C\sum_{a\notin\mathcal S^\star}
\frac{\bonus\log T}{\Delta_a}.
\]
Here \(C>0\) is a universal constant.
\end{theorem}

\begin{proof}
See Appendix~\ref{sec:regret-variance} for the variance-inflation bound and Appendix~\ref{sec:regret-mean-bonus} for the mean-bonus bound.
\end{proof}

For comparison, recall that the classical gap-dependent regret bound for vanilla TS \citep[e.g.,][]{agrawal2012analysis,lattimore2020bandit} is
\[
\mathcal R(T)
\le
C\sum_{a\notin\mathcal S^\star}
\frac{\log T}{\Delta_a}.
\]
Relative to this baseline, the mean-bonus bound in Theorem~\ref{thm:regret-main} incurs at most the multiplicative factor~\(\bonus\).  The variance-inflation bound carries the analogous \(\variance\)-scale multiplier, but with a sharper per-arm logarithmic factor \(\log(1+T\Delta_a^2/\variance)\).  For fixed gaps under~\eqref{eq:gamma-growth}, this factor differs from \(\log(T/\variance)\) only by an additive constant and is no larger than the classical \(\log T\) order.  Thus the regret price of optimism is controlled by \(\bonus\) or \(\variance\), respectively.

The conditions on the optimism parameters required by Theorem~\ref{thm:regret-main} are mild.  Apart from the upper-growth restrictions \(\variance=o(T/(\log T)^2)\) and \(\bonus=o(T/\log T)\), the only lower requirement for stability is \(\variance\to\infty\) and \(\bonus\to\infty\).  Hence the optimism parameters may be chosen to diverge arbitrarily slowly, for example \(\variance=\bonus=\log\log T\), or even slower iterated-logarithmic choices such as \(\log\log\log T\).

Overall, the results in this section show that optimism restores pull-count stability and yields standard Wald inference under adaptive sampling, all at only a controlled additional regret cost.

    \section{Numerical Experiments}\label{sec:numerical-experiments}

In this section, we present simulation studies on the finite-sample behavior of Thompson sampling for adaptive inference.  They examine whether the two optimistic variants of Thompson sampling induce the allocation stability required for valid inference and quantify the regret cost associated with this stabilization.

We compare the three sampling rules in Algorithm~\ref{alg:ts}: vanilla Thompson sampling (Option~\ref{option:A}), Thompson sampling with variance inflation (Option~\ref{option:B}), and Thompson sampling with a mean bonus (Option~\ref{option:C}).
We consider the reward model with Gaussian noise, i.e.
\[
    R_{a,t}=\mu_a+\xi_{a,t},\qquad
    \xi_{a,t}\stackrel{\mathrm{i.i.d.}}{\sim}\mathcal{N}(0,1).
\]
Unless otherwise specified, the main instance is
$
    (\mu_1,\ldots,\mu_6)=(1,1,1,0.8,0.5,0),
$
where the optimal set is $\mathcal S^\star=\{1,2,3\}$, with $|\mathcal S^\star|=3$.
In the following, we conduct \(5000\) Monte Carlo runs for all experiments.

\paragraph{Regret--inference tradeoff and parameter choice.}
We first study how the strength of optimism affects regret and inference.  To this end, we vary the variance inflation parameter $\variance$ in Option~\ref{option:B} and the mean-bonus parameter $\bonus$ in Option~\ref{option:C}.  Since stability predicts
$
    {N_{a,T}}/{T}\to {1}/{|\mathcal S^\star|},
$ for $ a\in\mathcal S^\star,
$
the quantity
\(
    \Var\{ {N_{a,T}}/{(T/|\mathcal S^\star|)}\}
\)
serves as a finite-sample measure of allocation variability among optimal arms. In the main instance, where $|\mathcal S^\star|=3$, we report $\Var(3N_{1,T}/T)$.

Figure~\ref{fig:regret-inference-tradeoff} presents the resulting regret--inference tradeoff. The horizontal axis reports relative regret with respect to vanilla Thompson sampling. The solid markers show $\Var(3N_{1,T}/T)$, and the hollow markers show the empirical coverage of the nominal $95\%$ Wald confidence interval for $\mu_1$.
The results reveal a clear finite-sample tradeoff between regret and allocation stability. When the optimism parameter is small, the algorithm behaves more like vanilla Thompson sampling. It incurs lower regret, but the realized allocation among the optimal arms is more variable, and the empirical coverage may fall below the nominal level. As $\variance$ or $\bonus$ increases, the allocation becomes more stable, as reflected by the decrease in $\Var(3N_{1,T}/T)$. The empirical coverage also improves.  This pattern is consistent with the theoretical role of optimism: a larger optimism parameter increases exploration of suboptimal arms and helps prevent persistent imbalance among optimal arms.
The theoretically motivated choice $(1+\log\log T)^2$ already yields empirical coverage close to the nominal $95\%$ level.  Further increasing the optimism parameter leads to only marginal improvements in allocation variability and coverage, while incurring a larger regret cost.

\begin{figure}[t]
    \centering
    \includegraphics[width=\linewidth]{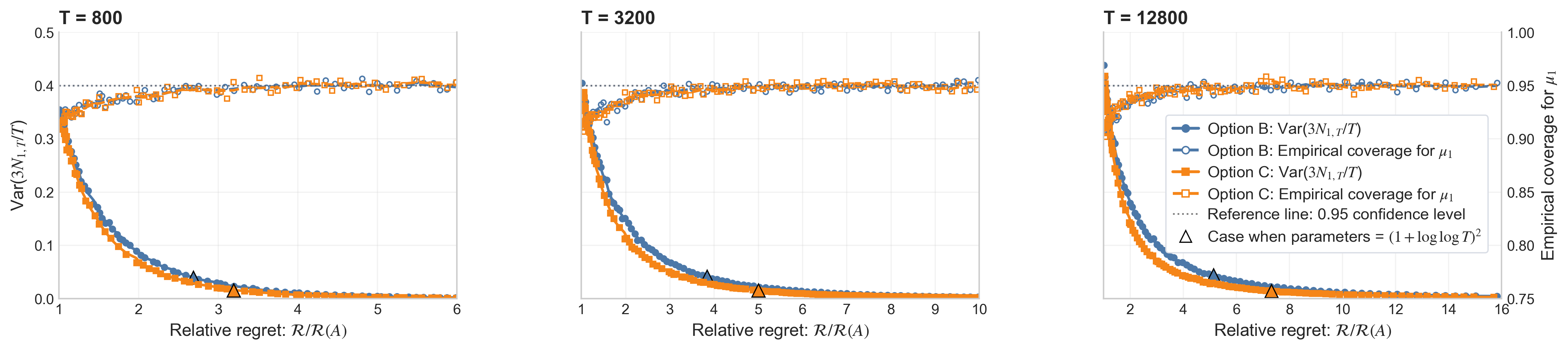}
    \caption{Regret--inference tradeoff in the instance $(1,1,1,0.8,0.5,0)$.  Solid markers show $\Var(3N_{1,T}/T)$, a finite-sample measure of allocation variability among the three optimal arms, and hollow markers show empirical coverage of the nominal $95\%$ Wald confidence interval for $\mu_1$.  The optimism parameters $\variance$ and $\bonus$ are varied for Options~\ref{option:B} and~\ref{option:C}, respectively.  The marked point corresponds to $\variance=\bonus=(1+\log\log T)^2$.}
    \label{fig:regret-inference-tradeoff}
\end{figure}

\paragraph{Regret and coverage/Gaussian-approximation across horizons.}
Figure~\ref{fig:coverage-horizon} studies how regret and coverage evolve with the horizon under the default choice $\variance=\bonus=(1+\log\log T)^2$.  The left panel reports regret as a function of $T$ on the logarithmic scale. Vanilla Thompson sampling has the smallest regret, as expected, because it is designed primarily for regret minimization and does not deliberately enforce allocation stability.  The two optimistic variants incur larger regret, but the observed increase is moderate. In particular, the finite-sample regret gap between the optimistic variants and vanilla Thompson sampling is visibly smaller than what one would expect from the theoretical upper bounds.
The right panel shows that, for \(95\%\) intervals, both optimistic variants achieve empirical coverage close to the nominal level across varying horizons, whereas vanilla Thompson sampling remains clearly below the target. A similar pattern is observed for \(80\%\) intervals, where vanilla Thompson sampling undercovers substantially.

Figure~\ref{fig:studentized-density} displays the empirical density of the studentized statistic at $T=3200$, together with the standard normal density.  Under vanilla Thompson sampling, the distribution of the studentized statistic for arm \(1\) shows a clear deviation from the $\mathcal{N}(0,1)$ benchmark, reflecting the lack of allocation stability among equally optimal arms.  In contrast, both optimistic variants are much closer to the standard normal curve.  This is also reflected in their smaller Kolmogorov--Smirnov and Cram\'er--von Mises distances.

\begin{figure}[t]
    \centering
    \includegraphics[width=0.92\linewidth]{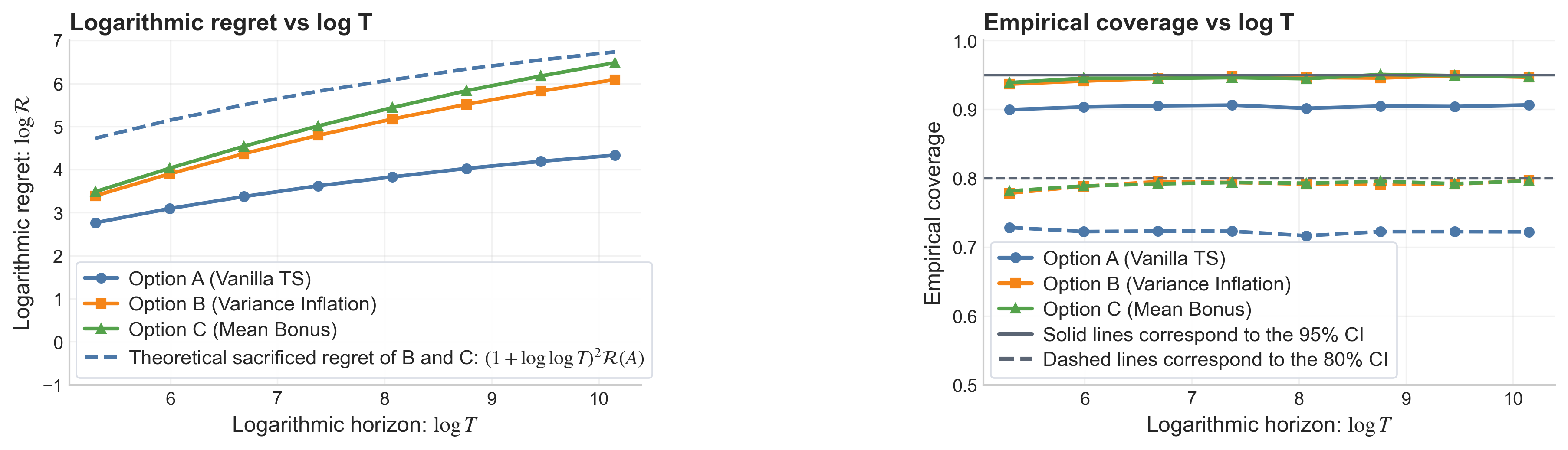}
    \caption{Regret and confidence interval coverage across horizons in the instance $(1,1,1,0.8,0.5,0)$, with $\variance=\bonus=(1+\log\log T)^2$.  The left panel reports regret as a function of the horizon, together with the corresponding theoretical upper bound.  The right panel reports empirical coverage of Wald confidence intervals.  The optimistic variants incur additional regret relative to vanilla Thompson sampling, but their empirical coverage is substantially closer to the nominal levels.}
    \label{fig:coverage-horizon}
\end{figure}

\begin{figure}[t]
    \centering
    \includegraphics[width=\linewidth]{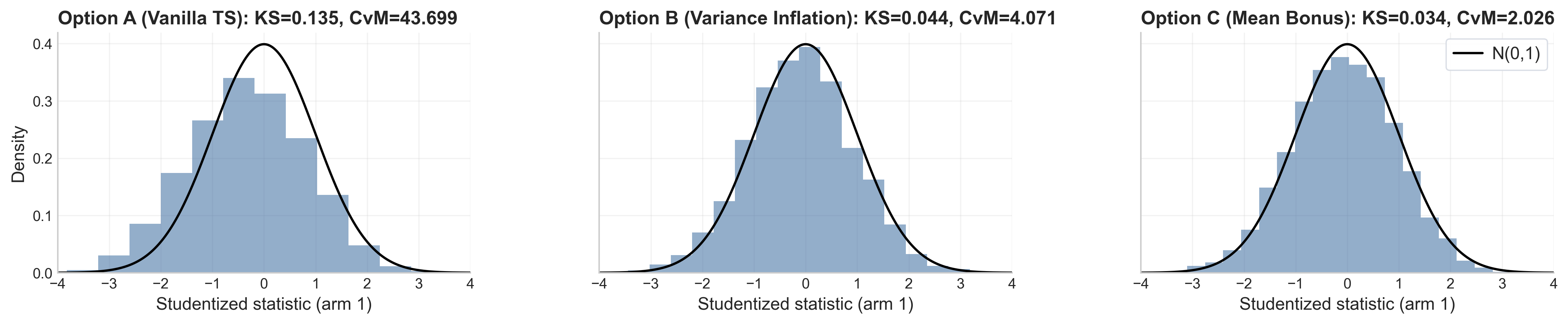}
    \caption{Studentized statistic for $\mu_1$ in the instance $(1,1,1,0.8,0.5,0)$ at $T=3200$, with $\variance=\bonus=(1+\log\log T)^2$.  The black curve is the standard normal density.  The panel titles report Kolmogorov--Smirnov and Cram\'er--von Mises distances from the standard normal distribution.  The optimistic variants yield studentized statistics closer to the normal benchmark.}
    \label{fig:studentized-density}
\end{figure}

\paragraph{Wald-type testing.}
Finally, we evaluate the effect of stabilization on hypothesis testing.  We consider the one-sided Wald test
$
    H_0:\mu_1=\mu_3
     \text{ against }
    H_1:\mu_1>\mu_3,
$
in the three-arm instance
$
    (\mu_1,\mu_2,\mu_3)=(1,1,1-\Delta),
$
at level $\alpha=0.05$. The parameter $\Delta$ controls the separation between arms~1 and~3.  When $\Delta=0$, the null hypothesis holds and all three arms are optimal.  When $\Delta>0$, arm~3 becomes suboptimal while arms~1 and~2 remain optimal, and the rejection probability measures power.
Figure~\ref{fig:hypothesis-testing} reports empirical rejection probabilities for several horizons.  At \(\Delta=0\), vanilla Thompson sampling exhibits inflated type I error, while the optimistic variants keep rejection probabilities close to the nominal \(5\%\) level. For \(\Delta>0\), the optimistic variants achieve higher rejection probabilities than vanilla Thompson sampling across the displayed horizons.

\begin{figure}[t]
    \centering
    \includegraphics[width=\linewidth]{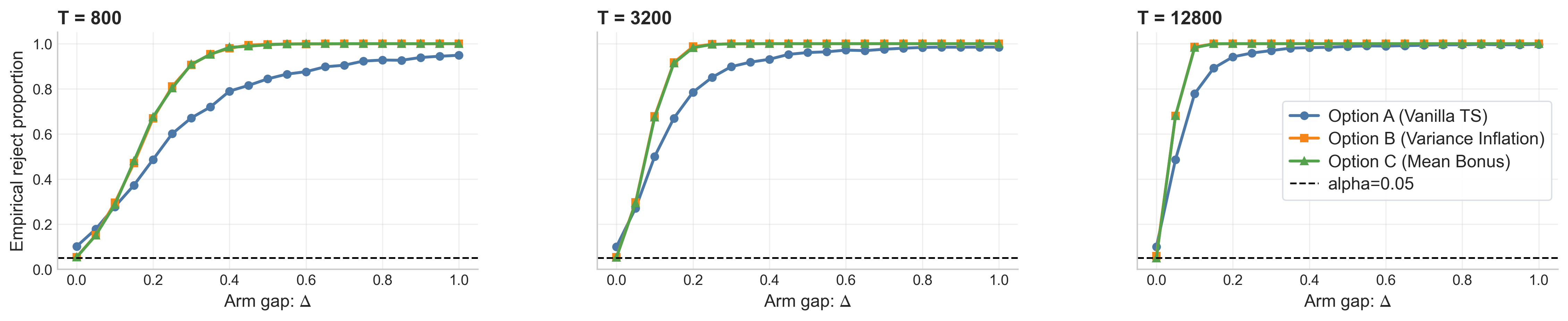}
    \caption{Empirical rejection probability for the one-sided Wald test of $H_0:\mu_1=\mu_3$ against $H_1:\mu_1>\mu_3$ in the instance $(1,1,1-\Delta)$, with level $\alpha=0.05$ and $\variance=\bonus=(1+\log\log T)^2$.  The point $\Delta=0$ corresponds to type I error, whereas $\Delta>0$ corresponds to power.  Across the displayed horizons, the optimistic variants provide improved null calibration while retaining high rejection probability under alternatives.}
    \label{fig:hypothesis-testing}
\end{figure}

    \section{Proof Overview} \label{sec:proof:sketch}

In this section, we give high-level proof sketches for Theorems~\ref{thm:main} and~\ref{thm:main:2}, emphasizing the novel technical ideas that enable the stability analysis for randomized Thompson sampling.

\subsection{Proof Overview of Theorem~\ref{thm:main}}

The proof proceeds in three steps. First, we study an equal-means bandit to isolate the allocation dynamics induced by Gaussian sampling noise. The key device is a "winner map" that exhibits a strict negative-feedback property: arms with larger empirical proportions have smaller index dispersion and thus smaller winning probability, which yields a negative Lyapunov drift toward uniform allocation. Second, we transfer this argument to the general $K$-armed bandit by working with proportions among the optimal arms and showing that suboptimal pulls are negligible. Third, we characterize suboptimal pull counts through the waiting times between successive pulls of a fixed suboptimal arm. Each new pull requires a rare Gaussian overshoot beyond the optimal-arm indices. The resulting waiting times can be compared with geometric variables, whose sums give the sharp logarithmic pull-count scale.

\paragraph{Step 1. Study $r$-armed equal-means bandits.}
We first study the bandit with $r$ \emph{equal-means arms}, i.e., the bandit has $r$ arms with the same mean $\mu_1=\cdots=\mu_r=\mu$.

\smallskip
\noindent
\textbf{(1.1) A winner-map viewpoint to handle sampling noise.}
Define the empirical proportions $x_{i,t}=N_{i,t}/t$ and $x_t=(x_{1,t},\ldots,x_{r,t})$. The next Gaussian indices satisfy
\[
\theta_{i,t+1}
= \widehat{\mu}_{i,t}
+ \sqrt{\frac{\variance}{N_{i,t}}}Z_{i,t+1}
= \widehat{\mu}_{i,t}
+ \sqrt{\frac{\variance}{t}}\frac{Z_{i,t+1}}{\sqrt{x_{i,t}}}.
\]
Heuristically, once the empirical means $\widehat{\mu}_{i,t}$ are close to the common mean $\mu$, the decision rule is mainly driven by the scaled sampling-noise terms $Z_{i,t+1}/\sqrt{x_{i,t}}$. This motivates the \emph{pure-noise winner map} $g$ defined on the simplex by
\[ \label{eq:sketch:winner-map}
g_i(x)
:=
\mathbb P\left(
\frac{Z_i}{\sqrt{x_i}}
=
\max_{1\le j\le r}\frac{Z_j}{\sqrt{x_j}}
\right),
\qquad Z_1,\ldots,Z_r\stackrel{\mathrm{i.i.d.}}{\sim}\mathcal N(0,1).
\]
Here the probability is over the fresh Gaussian sampling noises conditional on the current empirical proportion vector. A careful analysis of $g$ yields a negative-feedback monotonicity property (Lemma~\ref{lem:mono}): for any $r\ge3$ and $i\ne j$, if $x_i>x_j$ then $g_i(x)<g_j(x)$. Together with the rearrangement inequality, this yields the key inequality (Lemma~\ref{lem:dotprod})
\begin{equation}\label{eq:sketch:dotprod}
\sum_{i=1}^r x_i g_i(x)\le \frac1r \qquad \forall r\ge2,
\end{equation}
with equality if and only if \(x=(1/r,\ldots,1/r)\) for \(r\ge3\). When \(r=2\), the inequality is an identity. This rigorously characterizes the negative-feedback mechanism of $g$ for the Lyapunov analysis.

\smallskip
\noindent
\textbf{(1.2) Robustness to estimation errors at the expectation level.}
The pure-noise winner map \(g\) in~\eqref{eq:sketch:winner-map} describes the idealized allocation rule obtained by ignoring empirical-mean errors. In the actual TS rule, the conditional selection probabilities $p_{i,t}:=\mathbb P(A_{t+1}=i\mid\mathcal F_t)$ also depend on the estimation errors $\widehat\mu_{i,t}-\mu$. Thus the discrepancy between \(p_{i,t}\) and \(g_i(x_t)\) is the perturbation caused by empirical-mean errors.

To quantify this perturbation, let \(W_{i,t}:=\sqrt{N_{i,t}}(\widehat\mu_{i,t}-\mu)\) denote the Wald-scale empirical error. Concretely, Lemma~\ref{lem:weighted-perturb} gives the weighted perturbation bound
\begin{equation}\label{eq:sketch:weighted-perturb}
\sum_i x_{i,t}|p_{i,t}-g_i(x_t)|
\le
\frac{C_r}{\sqrt{\variance}}
\sum_i|W_{i,t}|.
\end{equation}
This follows because empirical-mean errors can change the winner only near pairwise Gaussian comparison boundaries, whose probabilities are controlled by anti-concentration. Here and below, \(C_r\) may change from line to line and depends only on \(r\).
Furthermore, Lemma~\ref{lem:eq-exp-tight} establishes that, for some \(\kappa>0\),
\begin{equation}\label{eq:sketch:moment-control}
\sup_i\sup_t
\mathbb E\left[\exp\{\kappa W_{i,t}^2\}\right]\le C_r,
\qquad
\sup_i\sup_t\mathbb E|W_{i,t}|\le C_r.
\end{equation}
The exponential-moment bound is the substantive part of Lemma~\ref{lem:eq-exp-tight}. The first-moment bound follows from \(|x|\le C_\kappa\exp\{\kappa x^2\}\). Together with~\eqref{eq:sketch:weighted-perturb}, it makes the expected perturbation contribution \(O(\variance^{-1/2})\).

The proof of Lemma~\ref{lem:eq-exp-tight} contains a separate technical novelty:
it separates control of the empirical error from control of small empirical
proportions. The two estimates are
\[
\sup_i\sup_t
\mathbb E\left[x_{i,t}\exp\{\kappa_0 W_{i,t}^2\}\right]\le C_r,
\qquad
\sup_i\sup_t\mathbb E[x_{i,t}^{-\alpha}]\le C_r.
\]
Here \(0<\alpha\le1\) is a tuning exponent. The first estimate controls the
exponential moment of \(W_{i,t}\) after weighting by the sampling proportion
\(x_{i,t}\). When arm \(i\) is not sampled, \(W_{i,t}\) stays unchanged, and the
factor \(x_{i,t}\) downweights its contribution. When arm \(i\) is sampled, the
new observation gives a sub-Gaussian one-step bound for the updated standardized
error. The second estimate shows that \(x_{i,t}\) has enough inverse-moment
control. The intuition is that a smaller \(x_{i,t}\) gives arm \(i\) a larger
index variance, which increases its chance of being selected again. H\"older's
inequality combines the two estimates and removes the factor \(x_{i,t}\),
yielding the uniform exponential-moment bound in~\eqref{eq:sketch:moment-control}.

\smallskip
\noindent
\textbf{(1.3) Lyapunov drift on the equal-means bandit.}
Define Lyapunov function \(
V_t := \sum_{i=1}^r (x_{i,t}-1/r)^2,
\)
which measures deviation from uniformity.
Recalling $x_{i,t+1}
    =x_{i,t}+\frac{1}{t+1}(\mathbf{1}\{A_{t+1}=i\}-x_{i,t})$, a direct calculation yields
\begin{equation}\label{eq:sketch:lyap-one-step}
\begin{aligned}
\mathbb E[V_{t+1}\mid\mathcal F_t]
&\le
V_t+\frac{2}{t+1}\left(\sum_i x_{i,t}p_{i,t}-\sum_i x_{i,t}^2\right)
+\frac{2}{(t+1)^2}
\\
&=
V_t+\frac{2}{t+1}\Bigg[
\left(\sum_i x_{i,t}g_i(x_t)-\sum_i x_{i,t}^2\right)
+
\sum_i x_{i,t}\bigl(p_{i,t}-g_i(x_t)\bigr)
\Bigg]
+\frac{2}{(t+1)^2}
\\
&\le
\left(1-\frac{2}{t+1}\right)V_t
+\frac{2}{t+1}\sum_i x_{i,t}\bigl|p_{i,t}-g_i(x_t)\bigr|
+\frac{2}{(t+1)^2}.
\end{aligned}
\end{equation}
The equality in~\eqref{eq:sketch:lyap-one-step} inserts and subtracts the pure-noise probabilities \(g_i(x_t)\). The final inequality uses~\eqref{eq:sketch:dotprod} to bound the pure-noise term and leaves only the perturbation between \(p_{i,t}\) and \(g_i(x_t)\).
Taking expectations in~\eqref{eq:sketch:lyap-one-step} and applying~\eqref{eq:sketch:weighted-perturb} and~\eqref{eq:sketch:moment-control} gives
\begin{equation}\label{eq:sketch:eqmeans-recursion}
\mathbb E[V_{t+1}]
\le
\left(1-\frac{2}{t+1}\right)\mathbb E[V_t]
+\frac{C_r}{(t+1)\sqrt{\variance}}
+\frac{2}{(t+1)^{2}}.
\end{equation}
Iterating~\eqref{eq:sketch:eqmeans-recursion} gives \(V_T\to0\) in probability under the growth condition \(\variance\to\infty\). This proves asymptotically uniform allocation in the simplified $r$-armed equal-means bandit under $\variance\to\infty$.

\smallskip
\noindent
\textbf{Step 2. General $K$-armed bandits: control the optimal-arm proportions.}
We now consider a general \(K\)-armed bandit with multiple optimal arms and suboptimal arms.
Lemma~\ref{lem:lem3} gives the suboptimal-pull estimate
\begin{equation}\label{eq:sketch:mean:1}
\mathbb E[N_{\mathrm{sub}}(T)]=O\bigl(\variance\log(T/\variance)\bigr),
\qquad
N_{\mathrm{sub}}(T)/T\xrightarrow{\mathbb{P}}0,
\qquad
N_{\mathrm{opt}}(T)/T\xrightarrow{\mathbb{P}}1.
\end{equation}
The optimal-arm dynamics are not a closed equal-means process, so the proof does not simply embed into the subsequence of optimal-arm pulls. Instead, it tracks the within-optimal proportions \(y_{i,t}:={N_{i,t}}/{N_{\mathrm{opt}}(t)}\) and the deviation \(V_t^\star:=\sum_{i\in\mathcal S^\star}(y_{i,t}-1/|\mathcal S^\star|)^2\). The one-step update for \(V_t^\star\) has denominator \(N_{\mathrm{opt}}(t)+1\), not \(t+1\). To keep this denominator comparable to \(t\), define
\[
\zeta:=\inf\{t\ge K:N_{\mathrm{sub}}(t)>\variance\log^2(T/\variance)\}.
\]
Since \(N_{\mathrm{sub}}(t)\) is nondecreasing, \eqref{eq:sketch:mean:1} and Markov's inequality give \(\mathbb P(\zeta\le T)=o(1)\). On \(\{\zeta>t\}\), after a burn-in of order \(\variance\log^2(T/\variance)\), we have \(N_{\mathrm{opt}}(t)\ge t/2\). Thus the stopped Lyapunov recursion has the same \(1/t\) scale as the equal-means recursion~\eqref{eq:sketch:eqmeans-recursion}. Taking expectations and using the optimal-arm analogues of~\eqref{eq:sketch:weighted-perturb} and~\eqref{eq:sketch:moment-control} gives
\begin{equation}\label{eq:sketch:opt-local-recursion}
\mathbb E[V_{t+1}^\star\mathbf 1\{\zeta>t+1\}]
\le
\left(1-\frac{2}{t+1}\right)
\mathbb E[V_t^\star\mathbf 1\{\zeta>t\}]
+
\frac{C}{t\sqrt{\variance}}
+
\frac{C}{t}\mathbb E\left[
\mathbb P(A_{t+1}\notin\mathcal S^\star\mid\mathcal F_t)
\right]
+
\frac{C}{t^2}.
\end{equation}
The contraction in~\eqref{eq:sketch:opt-local-recursion} comes from the winner-map drift restricted to \(\mathcal S^\star\). The empirical-mean perturbation is handled by the same perturbation and moment bounds and contributes \(O(\variance^{-1/2})\) after summation. The only additional error comes from rounds in which the full comparison selects a suboptimal arm. After multiplying~\eqref{eq:sketch:opt-local-recursion} by the \(t\)-weight that telescopes the contraction and summing up to \(T\), this contribution is at most \(C\mathbb E[N_{\mathrm{sub}}(T)]/T=o(1)\) by~\eqref{eq:sketch:mean:1}. The \(C/t^2\) term contributes only \(O(\log T/T)\). The burn-in and stopping errors are \(O(\variance\log^2(T/\variance)/T)\) and \(\mathbb P(\zeta\le T)=o(1)\), respectively. Under~\eqref{eq:gamma-growth}, all these terms vanish, so \(V_T^\star\to0\) in probability and Theorem~\ref{thm:main}(a) follows.

\begin{remark}
The two requirements in~\eqref{eq:gamma-growth} enter at different points of the proof.
The lower requirement \(\variance\to\infty\) controls the empirical-mean perturbation in the expected Lyapunov recursion. Rather than bounding this perturbation uniformly over sample paths, the proof combines Lemmas~\ref{lem:eq-exp-tight} and~\ref{lem:weighted-perturb} to make its expected contribution \(O(\variance^{-1/2})\), as reflected in~\eqref{eq:sketch:eqmeans-recursion}. The pathwise approach of~\citet{halder2025stable} instead controls the Wald-scale empirical errors uniformly at the law-of-the-iterated-logarithm scale \(\sqrt{\log\log T}\). It therefore needs the inflated index standard deviation \(\sqrt{\variance}\) to dominate this scale, leading to the stronger lower condition \(\variance/\log\log T\to\infty\). The upper requirement \(\variance(\log T)^2/T\to0\) is used in the stopped recursion~\eqref{eq:sketch:opt-local-recursion}. It makes the localization scale \(\variance\log^2(T/\variance)\), and the expected suboptimal exploration scale \(\variance\log(T/\variance)\), negligible relative to \(T\). On the localized event, \(N_{\mathrm{opt}}(t)\) is then comparable to \(t\), which gives the within-optimal-arm Lyapunov recursion its \(1/t\) scale.
\end{remark}

\smallskip
\noindent
\textbf{Step 3. Sharp suboptimal asymptotics via rare-event control and geometric coupling}.
To obtain sharp suboptimal pull-count asymptotics, fix a suboptimal arm $a$ with gap $\Delta_a>0$. After arm $a$ has been pulled $k$ times, pulling it again requires a rare Gaussian overshoot: $\theta_{a,t+1}$ must exceed the maximum Thompson sample among optimal arms. On good events, after burn-in and in the relevant count window, Lemmas~\ref{lem:theta-star-lower-tail} and~\ref{lem:one-step-bounds} give the conditional one-step bounds
\[
p^-_{k,T}(\varepsilon)
\le
\mathbb P(A_{t+1}=a\mid\mathcal F_t)
\le
p^+_{k,T}(\varepsilon),
\quad
\text{where}\quad
\left\{
\begin{aligned}
p^+_{k,T}(\varepsilon)
&=C\sqrt{\frac{\variance}{k}}
\exp\left\{-\alpha_+(\varepsilon)\frac{k}{\variance}\right\},
\\
p^-_{k,T}(\varepsilon)
&=c\exp\{-r_T\}
\sqrt{\frac{\variance}{k}}
\exp\left\{-\alpha_-(\varepsilon)\frac{k}{\variance}\right\}.
\end{aligned}
\right.
\]
with $\alpha_+(\varepsilon)=(\Delta_a-3\varepsilon)^2/2$, $\alpha_-(\varepsilon)=(\Delta_a+3\varepsilon)^2/2$, and \(r_T=o(\log(T/\variance))\).
The exponents \((\Delta_a\pm3\varepsilon)^2/2\) are the Gaussian overshoot rates after allowing slack for empirical-mean errors. The factor \(\exp\{-r_T\}\) appears only in the lower bound and is sub-leading because \(r_T=o(\log(T/\variance))\).
Then Lemma~\ref{lem:geom-coupling} converts these adaptive Bernoulli bounds into a truncated geometric coupling for the waiting times.
The tail log-law for variance-scale geometric sums (Lemma~\ref{lem:geom-tail-sum-loglaw}) then gives lower and upper count brackets around the scale \(\variance\log(T/\variance)\). Letting the slack parameters tend to zero yields
\[
N_{a,T}
\sim
\frac{2\variance}{\Delta_a^2}
\log\left(\frac{T}{\variance}\right),
\]
which proves Theorem~\ref{thm:main}(b).

\subsection{Proof Overview of Theorem~\ref{thm:main:2}}

The proof of Theorem~\ref{thm:main:2} follows a different route. The mean bonus gives under-sampled arms a deterministic advantage through \(B_{a,t}=\sqrt{2\bonus\log T/N_{a,t}}\), allowing pathwise index comparisons on a high-probability event rather than the averaged Lyapunov drift used for variance inflation.

Let \(\mathcal G_T\) denote the high-probability event from Lemma~\ref{lem:mean-bonus-good-event}, on which empirical means and all Gaussian index draws are controlled uniformly over time. Recall that $N_{a,T}^\star=T/m$ for any optimal arm $a\in\mathcal S^\star$ and $N_{a,T}^\star=(2\bonus\log T)/\Delta_a^2$ for any suboptimal arm $a\notin\mathcal S^\star$. Here $\bonus \to \infty$ and $\bonus \log T = o(T)$. Set
\[
\varepsilon_T:=\bonus^{-1/4}+\left(\frac{\bonus\log T}{T}\right)^{1/4}.
\]
For each suboptimal arm $a\notin\mathcal S^\star$, define the deterministic bracketing corridor
\[
N^-_{a,T}:=\left\lfloor(1-\varepsilon_T)N^\star_{a,T}\right\rfloor,
\qquad
N^+_{a,T}:=\left\lceil(1+\varepsilon_T)N^\star_{a,T}\right\rceil.
\]
The goal is to prove \(N^-_{a,T}\le N_{a,T}\le N^+_{a,T}\) on \(\mathcal G_T\). Since \(\mathbb P(\mathcal G_T)\to1\) and \(\varepsilon_T\to0\), this yields the desired convergence.
For Algorithm~\ref{alg:ts} with Option~\ref{option:C}, the decision rule can be written as
\begin{equation}\label{equ:mean:rule:0}
A_{t+1}\in\argmax_{a\in[K]} \Upsilon_{a,t+1},
\qquad
\Upsilon_{a,t+1}:=\widehat\mu_{a,t}+B_{a,t}+\frac{Z_{a,t+1}}{\sqrt{N_{a,t}}},
\qquad
B_{a,t}:=\sqrt{\frac{2\bonus\log T}{N_{a,t}}},
\end{equation}
where $Z_{a,t+1}\sim\mathcal{N}(0,1)$ are conditionally independent given $\mathcal F_t$.

\paragraph{Step 1: Suboptimal upper bounds.}
We prove $N_{a,T}\le N^+_{a,T}$ for every suboptimal arm $a\notin\mathcal S^\star$ (Lemma~\ref{lem:subopt-negligible}). The argument is a stopping-time comparison: once $N_{a,t}$ reaches $N^+_{a,T}$, the bonus term $B_{a,t}$ in~\eqref{equ:mean:rule:0} becomes small, and on \(\mathcal G_T\) the resulting indices for arm $a$ are dominated by those of optimal arms thereafter, preventing further pulls. Summing over suboptimal arms gives the pathwise bound
\begin{equation}\label{eq:sketch:mean-bonus-subopt}
N_{\mathrm{sub}}(T)
\le
\sum_{b\notin\mathcal S^\star}N^+_{b,T}
=O(\bonus\log T)=o(T)
\quad\text{on }\mathcal G_T.
\end{equation}
Since \(\mathbb P(\mathcal G_T)\to1\), this also gives \(N_{\mathrm{opt}}(T)/T\xrightarrow{\mathbb P}1\).

\paragraph{Step 2: Balancing among optimal arms.}
We establish a pathwise comparison lemma for optimal arms (Lemma~\ref{lem:opt-order}): on \(\mathcal G_T\), if \(N_{i,t}\ge (1+\varepsilon_T)N_{j,t}\) for \(i,j\in\mathcal S^\star\), then \(\Upsilon_{i,t+1}<\Upsilon_{j,t+1}\). Writing \(i_{\max}\) and \(i_{\min}\) for optimal arms with maximal and minimal terminal counts, applying this comparison to the last pull of \(i_{\max}\) gives
\begin{equation}\label{eq:sketch:mean-bonus-opt-balance}
N_{i_{\max},T}
\le
(1+\varepsilon_T)N_{i_{\min},T}+1.
\end{equation}
Combining~\eqref{eq:sketch:mean-bonus-subopt} and~\eqref{eq:sketch:mean-bonus-opt-balance} yields asymptotically uniform allocation over \(\mathcal S^\star\):
\begin{equation}\label{eq:sketch:mean:2}
N_{a,T}/(T/m)\xrightarrow{\mathbb P}1,
\qquad \forall a\in\mathcal S^\star.
\end{equation}

\paragraph{Step 3: Suboptimal lower bounds.}
Finally, we prove \(N_{a,T}\ge N^-_{a,T}\) for every suboptimal arm \(a\notin\mathcal S^\star\) via a horizon-end contradiction (Lemma~\ref{lem:subopt-lower}). Suppose instead that \(N_{a,T}\le N^-_{a,T}\). Then the bonus for arm \(a\) remains large throughout the horizon. Meanwhile, \eqref{eq:sketch:mean-bonus-subopt} and~\eqref{eq:sketch:mean-bonus-opt-balance} imply, for all large \(T\), that every optimal arm has terminal count at least \(T/(4m)\), and hence \(N_{i,t}\ge T/(8m)\) for every \(i\in\mathcal S^\star\) in the terminal block. The resulting index comparison gives
\[
\Upsilon_{a,t+1}
>
\max_{i\in\mathcal S^\star}\Upsilon_{i,t+1},
\qquad
t=T-\lfloor T/(8m)\rfloor,\ldots,T-1.
\]
Therefore the algorithm pulls only suboptimal arms in that block, which implies \(N_{\mathrm{sub}}(T)\ge \lfloor T/(8m)\rfloor\). This contradicts the \(o(T)\) bound in~\eqref{eq:sketch:mean-bonus-subopt}. Together with the upper corridor, we obtain \(N_{a,T}\in[N^-_{a,T},N^+_{a,T}]\) for all \(a\notin\mathcal S^\star\). Since \(\varepsilon_T\to0\), it follows that \(N_{a,T}/N^\star_{a,T}\xrightarrow{\mathbb P}1\) for each \(a\notin\mathcal S^\star\).

    \section{Conclusion}
We study statistical inference for Thompson sampling (TS) under adaptive data collection in stochastic multi-armed bandits. Our main message is that \emph{optimism stabilizes TS}: injecting optimism either by inflating the variance of the Gaussian indices or through an explicit \emph{mean bonus} yields \emph{stability}, in the sense that each arm's pull count concentrates around a deterministic scale. Concretely, the algorithm allocates asymptotically uniformly over the optimal set, while each suboptimal arm is pulled on a sharp, gap-dependent logarithmic scale. As a consequence, classical Wald-type inference based on studentized sample means is asymptotically valid under these optimistic TS procedures. Importantly, this inferential stability comes at only a mild additional regret cost, quantifying a small \emph{price of stability}. Several directions remain open, including extending our results beyond multi-armed bandits to more general settings like linear bandits and Markov decision processes.

\bibliographystyle{ims}
\bibliography{ref}

\newpage

\appendix

\section{Proofs for Variance-Inflated Thompson Sampling} \label{appendix:thm:main}
This section focuses on the analysis of variance-inflated Thompson sampling. Section~\ref{sec:regret-variance} proves the regret bound in Theorem~\ref{thm:regret-main}, while Sections~\ref{appendix:thm:main:a} and~\ref{appendix:thm:main:b} prove the two stability statements in Theorem~\ref{thm:main}. Recall that the index-variance multiplier may depend on the horizon and is assumed to satisfy
\[
    \variance\to\infty,
    \qquad
    \frac{\variance(\log T)^2}{T}\to0.
\]

\subsection{Proof of the Variance-Inflation Part of Theorem~\ref{thm:regret-main}} %
\label{sec:regret-variance}

We first control the expected number of pulls of a fixed suboptimal arm.

\begin{lemma}[Suboptimal pulls under variance inflation]
	\label{lem:lem3}
	Assume that \(\variance\ge 4\). There exists a universal constant
	\(C>0\) such that, for every fixed suboptimal arm
	\(a\notin\mathcal S^\star\), for all sufficiently large \(T\),
	\begin{equation}
		\label{eq:variance-arm-pull-time-uniform}
		\mathbb{E}[N_{a,u}]
		\le
		1+
		C\frac{\variance}{\Delta_a^2}
		\left\{
		1+\log\left(1+\frac{u\Delta_a^2}{\variance}\right)
		\right\}
	\end{equation}
	holds uniformly over every deterministic \(u\in\{1,\ldots,T\}\).
\end{lemma}

\begin{proof}
	Fix a deterministic \(u\in\{1,\ldots,T\}\).  If \(u<K\), then \(N_{a,u}\le1\), so the desired bound is trivial.  Hence assume \(u\ge K\).
	Consider a suboptimal arm $a\notin\mathcal S^\star$, and fix one optimal arm, relabeled
	as arm $1$, such that
	\(\mu_1=\mu^\star\) and \(\Delta_a=\mu^\star-\mu_a>0\).
	Recall that the empirical mean of arm \(a\) at time \(t\) is
	$
	\widehat\mu_{a,t}
	=
	\mu_a+(1/N_{a,t})\sum_{s\le t:\,A_s=a}\xi_{a,s}.
	$
	Because each arm is pulled once during initialization, \(N_{a,u}=1+\sum_{t=K}^{u-1}\mathbf 1\{A_{t+1}=a\}\).
		For \(t\in\{K,\dots,u-1\}\), let \(\mathcal{E}_{\mu,a}^{(t)}:=\{\widehat\mu_{a,t}\le \mu_a+\Delta_a/3\}\) and \(\mathcal{E}_{\theta,a}^{(t+1)}:=\{\theta_{a,t+1}\le \mu^\star-\Delta_a/3\}\).
		We decompose the expectation as
		\begin{equation}\label{eq:variance-pull-decomposition}
			\begin{aligned}
			\mathbb{E}[N_{a,u}]
			\le
			1+
			\underbrace{
			\sum_{t=K}^{u-1}
			\mathbb P \left(
			A_{t+1}=a, \bigl(\mathcal{E}_{\mu,a}^{(t)}\bigr)^c
			\right)
			}_{=:\mathrm{I}_{a,u}}
			+
			\underbrace{
			\sum_{t=K}^{u-1}
			\mathbb P \left(
			A_{t+1}=a, \mathcal{E}_{\mu,a}^{(t)}, \bigl(\mathcal{E}_{\theta,a}^{(t+1)}\bigr)^c
			\right)
			}_{=:\mathrm{II}_{a,u}}
			+
			\underbrace{
			\sum_{t=K}^{u-1}
			\mathbb P \left(
			A_{t+1}=a, \mathcal{E}_{\mu,a}^{(t)}, \mathcal{E}_{\theta,a}^{(t+1)}
			\right)
			}_{=:\mathrm{III}_{a,u}} .
			\end{aligned}
		\end{equation}
		We will bound \(\mathrm{I}_{a,u},\mathrm{II}_{a,u},\mathrm{III}_{a,u}\) separately.

	\medskip
	\noindent
		\textbf{Step 1: bound on \(\mathrm{I}_{a,u}\).}
	We have
		\begin{align} \label{eq:term-I-bound}
			\mathrm{I}_{a,u}
			&\le
			{\sum_{k=1}^{u}}
			\mathbb P \left(
			{\exists\,t\leq u-1:\,
			A_{t+1}=a,\ N_{a,t}=k,\
			\widehat\mu_{a,t}>\mu_a+\frac{\Delta_a}{3}}
			\right) \le
			{\sum_{k=1}^{u}}
			\exp\left\{-\frac{k\Delta_a^2}{18}\right\}
			\le
			\frac{C}{\Delta_a^2}.
		\end{align}
	For each \(k\), such a time can occur at most once on a sample path; the middle inequality is the sub-Gaussian tail bound for the selected-sum representation of \(\widehat\mu_{a,t}\) above.

	\medskip
	\noindent
	\textbf{Step 2: bound on \(\mathrm{II}_{a,u}\).}
	Write, only in this paragraph,
	\[
	L:=
	\left\lfloor
		\frac{18\variance}{\Delta_a^2}
		\log \left(1+\frac{u\Delta_a^2}{\variance}\right)
	\right\rfloor .
	\]
	We split the summation according to whether \(N_{a,t}\le L\).  The
	contribution from the first part is at most \(L\), because, along each
	sample path, whenever arm \(a\) is selected its pre-selection count
	\(N_{a,t}\) increases from one integer value to the next.  Hence the
	values \(1,\ldots,L\) can each contribute at most once.

	For the remaining part, we do not use a selected-once-per-\(k\) summation.
	Instead, for each \(t\),
	\[
	\begin{aligned}
	\mathbb P\Bigl(
	A_{t+1}=a,\ \mathcal{E}_{\mu,a}^{(t)},\bigl(\mathcal{E}_{\theta,a}^{(t+1)}\bigr)^c,\ N_{a,t}>L
	\Bigr)
	& =
	\mathbb E\Bigl[
	\mathbf 1\{\mathcal{E}_{\mu,a}^{(t)},\,N_{a,t}>L\}
	\mathbb P\bigl(
	A_{t+1}=a,\bigl(\mathcal{E}_{\theta,a}^{(t+1)}\bigr)^c
	\mid \mathcal F_t
	\bigr)
	\Bigr]\\
	& \le
	\mathbb E\Bigl[
	\mathbf 1\{\mathcal{E}_{\mu,a}^{(t)},\,N_{a,t}>L\}
	\mathbb P\bigl(
	\bigl(\mathcal{E}_{\theta,a}^{(t+1)}\bigr)^c
	\mid \mathcal F_t
	\bigr)
	\Bigr].
	\end{aligned}
	\]
	On \(\mathcal{E}_{\mu,a}^{(t)}\), conditionally on \(\mathcal F_t\) and with
	\(k=N_{a,t}\),
	\[
	\mathbb P\bigl(
	\bigl(\mathcal{E}_{\theta,a}^{(t+1)}\bigr)^c
	\mid \mathcal F_t
	\bigr)
	=
	\mathbb P\left(
	\theta_{a,t+1}> \mu^\star-\frac{\Delta_a}{3}
	\middle| \mathcal F_t
	\right)
	\le
	\exp\left\{-\frac{k\Delta_a^2}{18\variance} \right\} \le
	\frac{\variance}{\variance+u\Delta_a^2}.
	\]
	where the first inequality uses the Gaussian tail bound, and the second uses the definition of \(L\) and the fact that \(N_{a,t}>L\).
	Summing over \(t=K,\ldots,u-1\) therefore gives
	\[
	\sum_{t=K}^{u-1}
	\mathbb P\Bigl(
	A_{t+1}=a,\ \mathcal{E}_{\mu,a}^{(t)},\bigl(\mathcal{E}_{\theta,a}^{(t+1)}\bigr)^c,\ N_{a,t}>L
	\Bigr)
	\le
	\frac{u\variance}{\variance+u\Delta_a^2}
	\le
	\frac{\variance}{\Delta_a^2}.
	\]
	Combining the two parts, and using \(L\le 18\variance\Delta_a^{-2}
	\log(1+u\Delta_a^2/\variance)\), we obtain
	\begin{equation}
		\label{eq:term-II-bound}
		\mathrm{II}_{a,u}
		\le
		\frac{18\variance}{\Delta_a^2}
		\log \left(1+\frac{u\Delta_a^2}{\variance}\right)
		+
		\frac{\variance}{\Delta_a^2}
		\le
		\frac{C\variance}{\Delta_a^2}
		\left[
		1+
		\log \left(1+\frac{u\Delta_a^2}{\variance}\right)
		\right].
	\end{equation}

	\medskip
	\noindent
	\textbf{Step 3: comparison bound for \(\mathrm{III}_{a,u}\).}
	Define
	\[
	Q_t^{(a)}
	:=
	\mathbb P \left(
	\theta_{1,t+1}>\mu^\star-\frac{\Delta_a}{3}
	\middle|
	\mathcal F_t
	\right).
	\]
	Condition on $\mathcal F_t$ and on
	$\{\theta_{j,t+1}:j\neq 1\}$. Since the Gaussian indices have continuous
	conditional distributions, ties occur with conditional probability zero. On the event
	$
	\{A_{t+1}=a\}\cap \mathcal{E}_{\mu,a}^{(t)}\cap \mathcal{E}_{\theta,a}^{(t+1)},
	$
	we have
	$
	a=\argmax_{j\neq 1}\theta_{j,t+1},
	$ and $
	\max_{j\neq 1}\theta_{j,t+1}
	=
	\theta_{a,t+1}
	\le
	\mu^\star-{\Delta_a}/{3}.
	$
	Conditioning on \(\mathcal F_t\) and \(\{\theta_{j,t+1}:j\neq 1\}\) gives
	\begin{align}
		&\mathbb P \left(
		A_{t+1}=a, \mathcal{E}_{\mu,a}^{(t)}, \mathcal{E}_{\theta,a}^{(t+1)}
		\middle|
		\mathcal F_t,\{\theta_{j,t+1}:j\neq 1\}
		\right)\\
		&\le
		\mathbf 1 \left\{\mathcal{E}_{\mu,a}^{(t)}\right\}
		\mathbf 1 \left\{
		a=\argmax_{j\neq 1}\theta_{j,t+1},\
		\max_{j\neq 1}\theta_{j,t+1}\le \mu^\star-\frac{\Delta_a}{3}
		\right\} \cdot
		\mathbb P \left(
		\theta_{1,t+1}\le \max_{j\neq 1}\theta_{j,t+1}
		\middle|
		\mathcal F_t
		\right)\\
		&\le
		\mathbf 1 \left\{\mathcal{E}_{\mu,a}^{(t)}\right\}
		\mathbf 1 \left\{
		a=\argmax_{j\neq 1}\theta_{j,t+1},\
		\max_{j\neq 1}\theta_{j,t+1}\le \mu^\star-\frac{\Delta_a}{3}
		\right\}
		\left(1-Q_t^{(a)}\right).
		\label{eq:term-III-cond-subopt}
	\end{align}
	On the event
	$
	\{
	\max_{j\neq 1}\theta_{j,t+1}\le \mu^\star-{\Delta_a}/{3}
	\},
	$
	we also have
		\begin{align}
		&\mathbb P \left(
		A_{t+1}=1
		\middle|
		\mathcal F_t,\{\theta_{j,t+1}:j\neq 1\}
		\right)=
		\mathbb P \left(
		\theta_{1,t+1}> \max_{j\neq 1}\theta_{j,t+1}
		\middle|
		\mathcal F_t
		\right)\ge
		\mathbb P \left(
		\theta_{1,t+1}> \mu^\star-\frac{\Delta_a}{3}
		\middle|
		\mathcal F_t
		\right)
		=
		Q_t^{(a)}.
		\label{eq:term-III-cond-opt-lower}
	\end{align}
	Combining~\eqref{eq:term-III-cond-subopt} and~\eqref{eq:term-III-cond-opt-lower} gives
	\begin{align*}
		&\mathbb P \left(
		A_{t+1}=a, \mathcal{E}_{\mu,a}^{(t)}, \mathcal{E}_{\theta,a}^{(t+1)}
		\middle|
		\mathcal F_t,\{\theta_{j,t+1}:j\neq 1\}
		\right)\le
		\frac{1-Q_t^{(a)}}{Q_t^{(a)}}
		\mathbf 1 \left\{\mathcal{E}_{\mu,a}^{(t)}\right\}
		\mathbb P \left(
		A_{t+1}=1
		\middle|
		\mathcal F_t,\{\theta_{j,t+1}:j\neq 1\}
		\right).
	\end{align*}
	Taking conditional expectation with respect to $\mathcal F_t$ and using
	$\mathbf 1\{\mathcal{E}_{\mu,a}^{(t)}\}\le 1$, we have
	\begin{equation}
			\label{eq:key-comparison-III}
		\mathbb P \left(
		A_{t+1}=a, \mathcal{E}_{\mu,a}^{(t)}, \mathcal{E}_{\theta,a}^{(t+1)}
		\middle|
		\mathcal F_t
		\right)
		\le
		\frac{1-Q_t^{(a)}}{Q_t^{(a)}}
		\mathbb P \left(A_{t+1}=1\mid \mathcal F_t\right).
	\end{equation}
	{Write
	\[
	X_{1,t}^{(a)}
	:=
	\sqrt{\frac{N_{1,t}}{\variance}}
	\left(\widehat\mu_{1,t}-\mu^\star+\frac{\Delta_a}{3}\right),
	\qquad
	Q_t^{(a)}=\Phi(X_{1,t}^{(a)}).
	\]}
		Summing~\eqref{eq:key-comparison-III} over \(t=K,\dots,u-1\) and using the tower property, we get
	\[
	\begin{aligned}
	\mathrm{III}_{a,u}
	\le
	\sum_{t=K}^{u-1}
	\mathbb{E}\left[
	\frac{1-Q_t^{(a)}}{Q_t^{(a)}}
	\mathbf 1\{A_{t+1}=1\}
	\right]
	=
	\sum_{t=K}^{u-1}
	\mathbb{E}\left[
	\frac{\Phi(-X_{1,t}^{(a)})}{\Phi(X_{1,t}^{(a)})}
	\mathbf 1\{A_{t+1}=1\}
	\right].
	\end{aligned}
	\]
	We now group this last sum by \(k=N_{1,t}\).  On the event
	\(\{A_{t+1}=1\}\), each value of \(k\) can occur at most once.  For a fixed
	\(k\), with \(t\) in the summation range, the selected-sum concentration for
	\(\widehat\mu_{1,t}\) with \(N_{1,t}=k\) gives, for every \(z\ge0\),
	$$
	\mathbb P \left(
	\exists\,t:\,
	A_{t+1}=1,\ N_{1,t}=k,\
	(-X_{1,t}^{(a)})_+>z
	\right)
	\le
	\exp\left\{
	-\frac{k}{2}
	\left(
	\frac{\Delta_a}{3}
	+\sqrt{\frac{\variance}{k}} z
	\right)^2
	\right\}
	\le
	\exp\left\{-\frac{\variance z^2}{2}\right\}.
	$$
	A standard Mills-ratio bound yields
	a universal constant $C_0>0$ such that
	$$
		\left(
		\frac{\Phi(-X_{1,t}^{(a)})}{\Phi(X_{1,t}^{(a)})}
		\right)^2
		\le
		C_0^2
		\left(1+\bigl((-X_{1,t}^{(a)})_+\bigr)^2\right)
		\exp\left\{\bigl((-X_{1,t}^{(a)})_+\bigr)^2\right\}.
	$$
	Since \(\variance\ge4\), applying the tail-integral formula to the above two results implies
	\begin{equation}
		\label{eq:R2-uniform}
		\sum_{t=K}^{u-1}\mathbb{E}\left[
		\left(
		\frac{\Phi(-X_{1,t}^{(a)})}{\Phi(X_{1,t}^{(a)})}
		\right)^2
		\mathbf 1\{A_{t+1}=1,\ N_{1,t}=k\}
		\right]
		\le C.
	\end{equation}
	On \(N_{1,t}=k\) and
	\(\widehat\mu_{1,t}\ge \mu_1-\Delta_a/12\), we have
	\(X_{1,t}^{(a)}\ge\sqrt{k}\Delta_a/(4\sqrt{\variance})\ge0\).  Hence
	\begin{equation}
		\label{eq:R-good}
	\begin{aligned}
		&
		\sum_{t=K}^{u-1}\mathbb{E}\left[
		\frac{\Phi(-X_{1,t}^{(a)})}{\Phi(X_{1,t}^{(a)})}
		\mathbf 1 \left\{
		A_{t+1}=1,\ N_{1,t}=k,
		\widehat\mu_{1,t}\ge\mu_1-\Delta_a/12
		\right\}
		\right]
		\le
		2\exp\left\{-\frac{k\Delta_a^2}{32\variance}\right\}.
	\end{aligned}
	\end{equation}
    The same selected-sum concentration gives
	\[
	\mathbb P\left(
	\exists\,t:\ A_{t+1}=1,\ N_{1,t}=k,\
	\widehat\mu_{1,t}<\mu_1-\frac{\Delta_a}{12}
	\right)
	\le
	\exp\left\{-\frac{k\Delta_a^2}{288}\right\}.
	\]
	On the complementary event, Cauchy's inequality and~\eqref{eq:R2-uniform} give
	\begin{equation}
		\label{eq:R-bad}
	\begin{aligned}
		&
		\sum_{t=K}^{u-1}\mathbb{E}\left[
		\frac{\Phi(-X_{1,t}^{(a)})}{\Phi(X_{1,t}^{(a)})}
		\mathbf 1 \left\{A_{t+1}=1,\ N_{1,t}=k,
		\widehat\mu_{1,t}<\mu_1-\Delta_a/12
		\right\}
		\right] \le
		C\exp\left\{-\frac{k\Delta_a^2}{576}\right\}.
	\end{aligned}
	\end{equation}
	Combining the above bounds, then summing over \(k\ge1\), gives
	\begin{equation}
			\label{eq:term-III-bound}
	\begin{aligned}
		\mathrm{III}_{a,u}
		&\le
		C\left[{\sum_{k=1}^{u}}
		\exp\left\{-\frac{k\Delta_a^2}{32\variance}\right\}
		+
		{\sum_{k=1}^{u}}
		\exp\left\{-\frac{k\Delta_a^2}{576}\right\} \right] \le
		C \left[ \frac{\variance}{\Delta_a^2}
		+
	  \frac{1}{\Delta_a^2} \right].
	\end{aligned}
	\end{equation}

	\medskip
	\noindent
	\textbf{Step 4: combine the bounds.}
		Combining~\eqref{eq:term-I-bound},~\eqref{eq:term-II-bound}, and~\eqref{eq:term-III-bound}, we obtain
	\[
	\mathbb{E}[N_{a,u}]
	\le
	1
	+
	C \left[
	\frac{\variance}{\Delta_a^2}
	\log \left(1+\frac{u\Delta_a^2}{\variance}\right)
	+
	 \frac{\variance}{\Delta_a^2}
	+
	 \frac{1}{\Delta_a^2}\right].
	\]
	Since \(\variance\ge4\), the last term is absorbed into \(C\variance/\Delta_a^2\) after increasing \(C\).  The constants do not depend on \(u\), so~\eqref{eq:variance-arm-pull-time-uniform} holds uniformly over deterministic \(u\in\{1,\ldots,T\}\).
	This completes the proof.
\end{proof}

\begin{proof}[Proof of the variance-inflation part of Theorem~\ref{thm:regret-main}]
	By the definition of regret and the fact that
	$\mathbb{E}[R_{A_t,t}\mid A_t]=\mu_{A_t}$,
	\[
	\mathcal R(T)
	=
	T\mu^\star-\mathbb{E}\left[\sum_{t=1}^T R_{A_t,t}\right]
	=
	\sum_{a\notin\mathcal S^\star}\Delta_a \mathbb{E}[N_{a,T}].
	\]
	Applying Lemma~\ref{lem:lem3} arm by arm gives \(\Delta_a\mathbb{E}[N_{a,T}]\le \Delta_a+C(\variance/\Delta_a)\{1+\log(1+T\Delta_a^2/\variance)\}\).  Since the instance is fixed and \(T\Delta_a^2/\variance\to\infty\) under \(\variance(\log T)^2/T\to0\), the additive terms are absorbed into the logarithmic term for all sufficiently large \(T\).  Thus
	\[
	\mathcal R(T)
	\le
	C\sum_{a\notin\mathcal S^\star}
	\frac{\variance}{\Delta_a}
	\log \left(1+\frac{T\Delta_a^2}{\variance}\right),
	\]
	which proves the variance-inflation part of Theorem~\ref{thm:regret-main}.
\end{proof}

\subsection{Proof of Theorem~\ref{thm:main}(a)} \label{appendix:thm:main:a}
Define the variance-inflated suboptimal time scale
\begin{equation}\label{eq:LT-def}
   L_T:=\log\bigl(T/\variance\bigr),\quad \text{ satisfying } L_T\to\infty,\quad \frac{\variance L_T^2}{T}=L_T^2\exp\{-L_T\}\to 0.
\end{equation}
For the stability proof, we use two high-probability events.  For some universal constant $C$ determined later, define
\begin{equation}\label{equ:event}
\mathcal{E}_{T}:=  \left\{
\left|\widehat{\mu}_{a, t}-\mu_a\right| \leq \sqrt{\frac{C \log \log T}{N_{a, t}}}, \text{ for all }t\leq T \text{ with } N_{a,t}\ge1 \text{ and all arms } a \in [K]
\right\}.
\end{equation}
{For each fixed arm \(a\),
\[
\widehat\mu_{a,t}-\mu_a
=
\frac{1}{N_{a,t}}\sum_{s\le t:\,A_s=a}\xi_{a,s},
\qquad 1\le N_{a,t}\le T.
\]
Thus \(\mathcal E_T^c\) implies that, for some \(a\in[K]\) and some
\(t\le T\) with \(N_{a,t}\ge1\),
\[
\left|\frac{1}{N_{a,t}}\sum_{s\le t:\,A_s=a}\xi_{a,s}\right|
\ge
\sqrt{\frac{C\log\log T}{N_{a,t}}}.
\]
Applying Lemma~\ref{lem:TUB} arm by arm with the predictable indicators
\(I_s=\mathbf 1\{A_s=a\}\)
and \(\delta=(\log T)^{-2}\), union bounding over
arms, and absorbing the finitely many cases \(N_{a,t}<3\) by enlarging \(C\),
we have, for \(C\) large enough,}
\begin{equation}\label{equ:highprob}
\mathbb{P}\left( \mathcal{E}_{T} \right)
=1 - O\left( \frac{K}{(\log T)^{2}} \right) \to 1.
\end{equation}
Moreover, let \(\Omega_T\) denote the event that the total number of suboptimal pulls is at most \(\variance L_T^2\). Lemma~\ref{lem:lem3} gives
\(\mathbb{E}[N_{\mathrm{sub}}(T)]=O(\variance L_T)\), and hence Markov's inequality implies that the probability of \(\Omega_T\) goes to 1 as \(T\to\infty\). Equivalently, define
\begin{equation}\label{equ:209}
\begin{aligned}
  \Omega_T
  := \Bigl\{ N_{\mathrm{sub}}(T) =\sum_{b\notin \mathcal{S}^\star}N_{b,T}\ \le\ \variance L_T^2\Bigr\} \quad \text{ and thus }\
  \mathbb P(\Omega_T)\to 1.
\end{aligned}
\end{equation}
Before proving Theorem~\ref{thm:main}(a) in full generality, it is instructive to isolate a simplified regime
in which \emph{all arms have identical means}. When all arms have the same mean, any non-uniform allocation can only come from index uncertainty.
This isolates the algorithm's intrinsic self-correcting mechanism (negative feedback) without confounding by mean gaps, which would help us to understand the more general bandit cases.

\subsubsection{Step 1. Simplified case: all arms have identical means.}
We first introduce the basic setup for the simplified case.
\paragraph{Simplified setup (equal-means bandit).}
Consider \(r\ge 2\) arms, with a common mean \(\mu\) and independent mean-zero 1-sub-Gaussian reward noise. Run inflated
Thompson Sampling with parameter \(\variance\) as in Algorithm~\ref{alg:ts}(B), using conditionally independent Gaussian sampling indices. Define the empirical proportions
\[
x_{i,t} := \frac{N_{i,t}}{t},\qquad x_t := (x_{1,t},\dots,x_{r,t}) \in \Delta_r,
\]
where \(\Delta_r:=\{x\in\mathbb{R}_+^r:\sum_i x_i=1\}\) is the simplex.
Conditionally on the history \(\mathcal{F}_t\), the next Gaussian sampling indices satisfy
\[
\theta_{i,t+1} = \widehat{\mu}_{i,t} + \sqrt{\frac{\variance}{N_{i,t}}} Z_{i,t+1},
\]
with independent \(Z_{i,t+1}\sim\mathcal{N}(0,1)\). Since \(\mu\) is common across arms, the only systematic
difference across \(\theta_{i,t+1}\) arises from the index standard deviation
\(\sqrt{\variance/N_{i,t}}\), i.e., from uncertainty.

\paragraph{An ideal pure-noise winner map and negative feedback.}
If we temporarily ignore the small estimation errors \(\widehat{\mu}_{i,t}-\mu\), the next chosen arm is
the maximizer of \(Z_{i,t+1}/\sqrt{x_{i,t}}\). This motivates the following \emph{pure-noise} winner map:
for \(x\in\Delta_r\), let \(Z_1,\dots,Z_r\stackrel{\mathrm{i.i.d.}}{\sim}\mathcal{N}(0,1)\), and define
\begin{equation}\label{eq:def:g:variance}
g_i(x) := \mathbb{P} \left(\frac{Z_i}{\sqrt{x_i}}=\max_{1\le j\le r} \frac{Z_j}{\sqrt{x_j}}\right).
\end{equation}
so that \(g(x)=(g_1(x),g_2(x),...,g_r(x))\in\Delta_r\).
In the monotonicity proof below, we also view the same winning probability as a function of the scale vector \(s=(s_1,\ldots,s_r)\):
\[
\Gamma_i(s):=\mathbb P\left(s_i Z_i\ge \max_{\ell\ne i}s_\ell Z_\ell\right),
\qquad
g_i(x)=\Gamma_i\bigl(x_1^{-1/2},\ldots,x_r^{-1/2}\bigr).
\]
Thus derivatives with respect to \(s_i\) are derivatives of the scale-parametrized map \(\Gamma_i\), not of \(g_i\) as a fixed scalar evaluated at a fixed \(x\).
The following facts characterize some useful monotonicity property of $g(x)$.
\begin{lemma}[Monotonicity]\label{lem:mono}
Assume $r\ge 3$ and $x\in \mathrm{int}(\Delta_r)$. For any distinct $i\neq j$, if $x_i>x_j$ then $g_i(x)<g_j(x)$.
\end{lemma}

\begin{proof}[Proof of Lemma~\ref{lem:mono}]
With the scale vector \(s=(x_1^{-1/2},\ldots,x_r^{-1/2})\), \(Z_k/\sqrt{x_k}=x_k^{-1/2}Z_k\) for each \(k\), with i.i.d.\ \(Z_k\sim\mathcal{N}(0,1)\).
Fix $i\neq j$ and let $x^{(ij)}$ be the vector obtained from $x$ by swapping coordinates $i$ and $j$. By exchangeability of \((Z_1,\dots,Z_r)\), \(g_j(x)=g_i(x^{(ij)})\).
Hence it suffices to show that if $s_i<s_j$, then
$g_i(x^{(ij)})>g_i(x)$.
Now, we establish two strict monotonicity properties.

\paragraph{Property 1. \(\Gamma_i\) strictly increases in \(s_i\).}
Fix an index $i$ and hold $(s_\ell)_{\ell\neq i}$ fixed.
Conditioning on $Z_i=z$ yields
\[
\Gamma_i(s)
=\mathbb P\Bigl(s_i Z_i \ge \max_{\ell\neq i}s_\ell Z_\ell\Bigr)
=\int_{\mathbb R}\phi(z)\prod_{\ell\neq i}\Phi \Bigl(\frac{s_i}{s_\ell}z\Bigr) dz.
\]
	 Since \(x\in\mathrm{int}(\Delta_r)\) we have \(s_\ell\in(0,\infty)\) for all \(\ell\); moreover the integrand is smooth in \(s_i\) and its derivative is dominated by an integrable envelope, so differentiating under the integral sign is justified by dominated convergence, giving
	\[
	\frac{\partial}{\partial s_i}\Gamma_i(s)
	=\int_{\mathbb R}\phi(z)\sum_{\ell\neq i}\Bigl(\frac{z}{s_\ell}\Bigr)\phi\left(\frac{s_i}{s_\ell}z\right)
\prod_{q\neq i,\ell}\Phi\left(\frac{s_i}{s_q}z\right) dz.
\]
We now pair $z$ and $-z$. Using $\overline{\Phi}(-u)=\Phi(u)$ and $\phi(-z)=\phi(z)$, for each $\ell\neq i$ the contribution
of $z$ and $-z$ equals, for $z>0$,
\[
\phi(z)\Bigl(\frac{z}{s_\ell}\Bigr)\phi\left(\frac{s_i}{s_\ell}z\right)
\Biggl[
\prod_{q\neq i,\ell}\Phi\left(\frac{s_i}{s_q}z\right)\;-\;\prod_{q\neq i,\ell}\overline{\Phi}\left(\frac{s_i}{s_q}z\right)
\Biggr].
\]
For $z>0$ and any $u>0$, $\Phi(u)>1/2$ and $\overline{\Phi}(u)<1/2$. Since $r\ge 3$, the product
$\prod_{q\neq i,\ell}$ contains at least one factor, hence the difference above is strictly positive. All remaining terms are strictly positive for $z>0$.
Therefore the integrand is strictly positive on a set of positive measure, implying \(\partial_{s_i}\Gamma_i(s)>0\).
That is, holding competitors fixed, \(\Gamma_i\) is strictly increasing in its own scale \(s_i\).

\paragraph{Property 2. \(\Gamma_i\) strictly decreases in \(s_j\) for \(j\neq i\).}
Fix $j\neq i$ and hold $s_i$ and $(s_\ell)_{\ell\neq i,j}$ fixed.
We have
\[
\Gamma_i(s)
=
\mathbb{E}\left[\overline{\Phi}\left(\frac{\max\{s_jZ_j,\max_{\ell\neq i,j}s_\ell Z_\ell\}}{s_i}\right)\right].
\]
For each fixed \(m\in\mathbb R\), the map
\[
u\mapsto \int_{\mathbb R}\overline{\Phi}\left(\frac{\max\{uv,m\}}{s_i}\right)\phi(v) dv
\]
is strictly decreasing on \((0,\infty)\).
   Indeed, differentiating under the integral sign, one obtains
	\[
	-\frac{1}{s_i}\int_{v>m/u} v \phi \left(\frac{uv}{s_i}\right)\phi(v) dv,
	\]
whose integrand is an odd function multiplied by a strictly positive even weight, so the integral is strictly
positive for every threshold $m/u\in\mathbb R$, and therefore the derivative is strictly negative for all \(u>0\).
Taking expectation over the remaining Gaussian variables yields that \(s_j\mapsto \Gamma_i(s)\) is strictly decreasing.

Assume $s_i<s_j$. Passing from $x$ to $x^{(ij)}$ increases $s_i$ and decreases $s_j$ while
leaving all other scales fixed. By the above two steps, this strictly increases \(\Gamma_i\), i.e.,
$g_i(x^{(ij)})>g_i(x)$. This completes the proof of Lemma~\ref{lem:mono}.
\end{proof}

\begin{lemma}[Dot-product inequality]\label{lem:dotprod}
Fix $r\ge 2$ and $x\in\mathrm{int}(\Delta_r)$. Then \(\sum_{i=1}^r x_i g_i(x)\le 1/r\). If $r=2$, equality holds for every $x\in\Delta_2$. If \(r\ge3\), equality holds in this inequality if and only if \(x=(1/r,\dots,1/r)\).
\end{lemma}

\begin{proof}[Proof of Lemma~\ref{lem:dotprod}]
We prove this lemma by considering the following two cases.

\vspace{4pt}
\noindent\textbf{Case 1.}
The case $r=2$ is immediate: since \(Z_1/\sqrt{x_1}-Z_2/\sqrt{x_2}\) is a centered Gaussian with continuous law, symmetry gives \(g_1(x)=g_2(x)=1/2\) for every \(x\in\Delta_2\), so the claimed inequality holds with equality.

\vspace{4pt}
\noindent\textbf{Case 2.}
Now assume $r\ge 3$. Let $x_{(1)}\le\cdots\le x_{(r)}$ denote the ordered coordinates of $x$, and permute indices accordingly.
By Lemma~\ref{lem:mono}, the corresponding values satisfy the opposite ordering \(g_{(1)}(x)\ge g_{(2)}(x)\ge\cdots\ge g_{(r)}(x)\).
A standard consequence of the rearrangement inequality (equivalently, the ``reversed'' Chebyshev sum inequality) states that for nonnegative sequences
$a_1\le\cdots\le a_r$ and $b_1\ge\cdots\ge b_r$,
\(\sum_{i=1}^r a_i b_i\le r^{-1}(\sum_i a_i)(\sum_i b_i)\).
Applying this with \(a_i=x_{(i)}\) and \(b_i=g_{(i)}(x)\), and using \(\sum_i x_i=1\) and \(\sum_i g_i(x)=1\), yields the claimed inequality.

If $x$ is not uniform, there exist $i\neq j$ with $x_i>x_j$. Lemma~\ref{lem:mono} then gives $g_i(x)<g_j(x)$, so the two sequences are not constant.
In that case the reversed Chebyshev inequality is strict, so $\sum_i x_i g_i(x)<1/r$.
Conversely, at the uniform point $x=(1/r,\dots,1/r)$ we have $g_i(x)=1/r$ by symmetry, giving equality.
This proves the claim.
\end{proof}

Intuitively, over-pulled arms are penalized because they generate less variable Gaussian indices, while under-pulled arms are rewarded by larger index variance.
The above two conclusions characterize the tendency that the proportions toward balancing.

\paragraph{Weighted perturbation and exponential tightness.}
 We now prove uniform exponential tightness of the standardized empirical means under inflated Thompson sampling and combine it with a weighted perturbation bound in the quadratic Lyapunov recursion.  Throughout this paragraph, the reward noise is independent mean-zero 1-sub-Gaussian, while the Thompson sampling indices use conditionally independent standard Gaussian sampling noise.  Since \(\variance\to\infty\), all estimates below are used only for horizons large enough that \(\variance\) exceeds a constant depending on \(r\).
The next lemma is a key input.

\begin{lemma}[Uniform exponential tightness in the equal-means case]\label{lem:eq-exp-tight}
Assume \(\mu_1=\cdots=\mu_r=\mu\).  Suppose the reward noises are independent mean-zero 1-sub-Gaussian variables, and suppose the Thompson sampling indices use conditionally independent standard Gaussian sampling noises.  If \(\variance\) is sufficiently large, depending only on \(r\), then there exist constants \(\kappa_r>0\) and \(C_r<\infty\), independent of \(t\) and \(T\), such that
\[
\sup_{1\le i\le r}\sup_{t\ge r}
\mathbb{E}\left[\exp\left\{\kappa_r N_{i,t}(\widehat\mu_{i,t}-\mu)^2\right\}\right]
\le C_r.
\]
In particular, \[\sup_{1\le i\le r}\sup_{t\ge r}\mathbb{E}[\sqrt{N_{i,t}}|\widehat\mu_{i,t}-\mu|]\le C_r.\]
\end{lemma}

\begin{proof}
We prove the lemma by coupling two Lyapunov estimates.  We first record the sub-Gaussian replacement for the Gaussian quadratic-moment calculation.  For a mean-zero 1-sub-Gaussian random variable \(\xi'\), there are constants \(C,\lambda_0>0\) such that
\[
\mathbb{E}[\exp\{s\xi'+\lambda(\xi')^2\}]
\le(1+C\lambda)\exp\{Cs^2\},
\qquad s\in\mathbb R,\quad 0\le\lambda\le\lambda_0 .
\]
For \(0<\eta\le \lambda_0\), expanding the square and applying this bound with
\(s=2\eta\sqrt n\,w/(n+1)\) and \(\lambda=\eta/(n+1)\) gives
\[
\begin{aligned}
&\mathbb{E}\left[\exp\left\{
\eta\left(\sqrt{\frac n{n+1}}w+\frac{\xi'}{\sqrt{n+1}}\right)^2
\right\}\right] \le
\left(1+\frac{C\eta}{n+1}\right)
\exp\left\{
\eta\frac n{n+1}w^2+
C\eta^2\frac{n}{(n+1)^2}w^2
\right\}.
\end{aligned}
\]
Choose a universal \(\kappa_0>0\) small enough that \(\kappa_0\le\lambda_0\) and \(C\kappa_0\le1/2\). Then, with \(\eta=\kappa_0\),
\[
\kappa_0\frac n{n+1}w^2+
C\kappa_0^2\frac{n}{(n+1)^2}w^2
\le
\kappa_0 w^2-\frac{\kappa_0 w^2}{2(n+1)}.
\]
Thus, after setting \(\kappa_1:=\kappa_0/2\) and increasing \(\kappa_2\), we obtain
\begin{equation}\label{eq:single-step-exp-contraction}
		\mathbb{E}\left[\exp\left\{
		\kappa_0\left(\sqrt{\frac n{n+1}}w+\frac{\xi'}{\sqrt{n+1}}\right)^2
		\right\}\right]
		\le
			\left(1+\frac{\kappa_2}{n+1}\right)
			\exp\left\{\kappa_0 w^2-\frac{\kappa_1 w^2}{n+1}\right\}.
		\end{equation}
We set \(\gamma:=\kappa_0/(4r)\).  For all sufficiently large values of \(\variance\), we have \(\gamma>1/(2\variance)\).
All constants denoted by \(c,C\) below may change from line to line and may depend on \(r,\kappa_0,\kappa_1,\kappa_2,\gamma\), and on the fixed choice of \(\alpha\) after Step~2, but not on \(t\), \(T\), or \(\variance\) once \(\variance\) is large enough.
In the following, we write \(W_{i,t}:=\sqrt{N_{i,t}}(\widehat\mu_{i,t}-\mu)\) and recall that \(x_{i,t}=N_{i,t}/t\), so that \(\sum_{i=1}^r x_{i,t}=1\).

\vspace{3pt}
\noindent\textbf{Step 1: control of \(x_{i,t}\exp\{\kappa_0 W_{i,t}^2\}\).}
Fix \(i\le r\) and condition on \(\mathcal F_t\).
If arm \(i\) is not selected at time \(t+1\), then
\[
x_{i,t+1}\exp\{\kappa_0 W_{i,t+1}^2\}
=\frac{t}{t+1}x_{i,t}\exp\{\kappa_0 W_{i,t}^2\}.
\]
If arm \(i\) is selected, applying~\eqref{eq:single-step-exp-contraction} with \(n=N_{i,t}\) and \(w=W_{i,t}\) gives
\[
\mathbb{E}\left[x_{i,t+1}\exp\{\kappa_0 W_{i,t+1}^2\}\mid\mathcal F_t,A_{t+1}=i\right]
\le
\frac{N_{i,t}+1+\kappa_2}{t+1}\exp\left\{\kappa_0 W_{i,t}^2-\frac{\kappa_1 W_{i,t}^2}{N_{i,t}+1}\right\}.
\]
Comparing this with \(N_{i,t}(t+1)^{-1}\exp\{\kappa_0 W_{i,t}^2\}\), the difference is
\[
\frac{\exp\{\kappa_0 W_{i,t}^2\}}{t+1}
\left[(N_{i,t}+1+\kappa_2)\exp\{-\kappa_1 W_{i,t}^2/(N_{i,t}+1)\}-N_{i,t}\right].
\]
When the bracket is positive, \(\kappa_1 W_{i,t}^2/(N_{i,t}+1)\le \log((N_{i,t}+1+\kappa_2)/N_{i,t})\le C/(N_{i,t}+1)\), so \(W_{i,t}^2\le C\) and hence the positive part of the difference is at most \(C/(t+1)\).  Therefore, regardless of the conditional probability that arm \(i\) is selected,
	\begin{equation}\label{eq:eq-means-weighted-one-step}
	\mathbb{E}\left[x_{i,t+1}\exp\{\kappa_0 W_{i,t+1}^2\}\mid\mathcal F_t\right]
	\le
	\frac{t}{t+1}x_{i,t}\exp\{\kappa_0 W_{i,t}^2\}
	+\frac{C}{t+1}.
	\end{equation}
Taking expectations and iterating gives
\begin{equation}\label{eq:q-exp-tight}
\sup_{1\le i\le r}\sup_{t\ge r}
\mathbb{E}\left[x_{i,t}\exp\{\kappa_0 W_{i,t}^2\}\right]
\le C.
\end{equation}

\vspace{3pt}
\noindent\textbf{Step 2: control of inverse proportions.}
Fix once and for all a universal \(0<\alpha\le1\). We first lower bound the Gaussian TS probability of selecting arm \(i\). Conditional on \(\mathcal F_t\), the index comparison is equivalent to comparing \((W_{j,t}+\sqrt{\variance}Z_{j,t+1})/\sqrt{x_{j,t}}\), \(j=1,\dots,r\).
Since \(\gamma>1/(2\variance)\), there exists \(c>0\) such that
\begin{equation}\label{eq:anti-starvation-lower}
p_{i,t}:=\mathbb P(A_{t+1}=i\mid\mathcal F_t)
\ge
c\exp\left\{-\gamma\sum_{j=1}^r W_{j,t}^2\right\}.
\end{equation}
Indeed, conditional on \(\mathcal F_t\), the Thompson index of arm \(j\) can be written as
\[
\widehat\mu_{j,t}+\sqrt{\frac{\variance}{N_{j,t}}}Z_j
=
\mu+
\frac{W_{j,t}+\sqrt{\variance}Z_j}{\sqrt{N_{j,t}}}.
\]
Consider the event
\[
\left\{W_{i,t}+\sqrt{\variance}Z_i\ge1\right\}
\cap
\left\{\bigcap_{j\ne i}
\left\{W_{j,t}+\sqrt{\variance}Z_j\le 0 \right\}\right\}.
\]
On this event, the index of arm \(i\) is at least
\(\mu+1/\sqrt{N_{i,t}}\), while every other index is at most \(\mu\). Hence this event is contained in \(\{A_{t+1}=i\}\).
Since \(\gamma\) is fixed and \(\variance\to\infty\), for all sufficiently large \(T\) the slack in \(\gamma>1/(2\variance)\) absorbs the polynomial prefactor in the Mills lower bound from Lemma~\ref{lem:normal-tail}. Thus, after decreasing \(c\) if necessary,
\[
\mathbb P(W_{i,t}+\sqrt{\variance}Z_i\ge1\mid\mathcal F_t)
\ge c\exp\{-\gamma W_{i,t}^2\},
\]
and, for \(j\ne i\),
\[
\mathbb P(W_{j,t}+\sqrt{\variance}Z_j\le0\mid\mathcal F_t)
\ge c\exp\{-\gamma W_{j,t}^2\}.
\]
The Gaussian sampling noises are conditionally independent given \(\mathcal F_t\), so multiplying these bounds and absorbing the resulting constant into \(c\) proves~\eqref{eq:anti-starvation-lower}.

We next compute the drift of \(x_{i,t}^{-\alpha}\). Note that
\[ x_{i,t+1}^{-\alpha} =
  \begin{cases}
    x_{i,t}^{-\alpha}(1+1/t)^\alpha\le x_{i,t}^{-\alpha}(1+C/t) & \text{if arm } i \text{ is not selected} \\
    x_{i,t}^{-\alpha}(1+1/t)^\alpha(1+1/N_{i,t})^{-\alpha}\le x_{i,t}^{-\alpha}(1+C/t-c/N_{i,t}) & \text{if arm } i \text{ is selected}.
\end{cases}
\]
Taking the conditional expectation and using \(N_{i,t}=t x_{i,t}\), we get
\begin{equation}\label{eq:eq-means-inv-one-step}
\mathbb{E}[x_{i,t+1}^{-\alpha}\mid\mathcal F_t]
\le
x_{i,t}^{-\alpha}
+\frac1t\left(Cx_{i,t}^{-\alpha}
-c x_{i,t}^{-\alpha}\frac{p_{i,t}}{x_{i,t}}\right)\le
x_{i,t}^{-\alpha}
+\frac1t\left(Cx_{i,t}^{-\alpha}
-c x_{i,t}^{-\alpha-1}\exp\left\{-\gamma\sum_{j=1}^r W_{j,t}^2\right\}\right),
\end{equation}
The lower bound~\eqref{eq:anti-starvation-lower} on \(p_{i,t}\) is used in the second inequality of~\eqref{eq:eq-means-inv-one-step}.
	For the fixed \(\alpha\), there exist positive constants \(C_\alpha,c_\alpha\) such that
	\(Cz^\alpha-cz^{\alpha+1}\le C_\alpha-c_\alpha z^\alpha\) for all \(z\ge0\).  Applying this with \(z=x_{i,t}^{-1}\exp\{-\gamma\sum_{j=1}^r W_{j,t}^2\}\) gives
	$Cx_{i,t}^{-\alpha}-c x_{i,t}^{-\alpha-1}\exp\{-\gamma\sum_{j=1}^r W_{j,t}^2\}
	\le C_\alpha\exp\{\alpha\gamma\sum_{j=1}^r W_{j,t}^2\}-c_\alpha x_{i,t}^{-\alpha}$.
	Using this fact with~\eqref{eq:eq-means-inv-one-step}, and then renaming constants, gives
	\begin{equation}\label{eq:q-inv-drift}
\mathbb{E}[x_{i,t+1}^{-\alpha}\mid\mathcal F_t]
\le
\left(1-\frac{c}{t}\right)x_{i,t}^{-\alpha}
+\frac{C}{t}\exp\left\{\alpha\gamma\sum_{j=1}^r W_{j,t}^2\right\},
\end{equation}

\vspace{3pt}
\noindent\textbf{Step 3: closing the two Lyapunov estimates.}
By the choice \(\gamma=\kappa_0/(4r)\) and \(0<\alpha\le1\),
\[
r\alpha\gamma=\frac{\alpha\kappa_0}{4}<\kappa_0,
\qquad
\frac{r\alpha\gamma}{\kappa_0-r\alpha\gamma}
=
\frac{\alpha}{4-\alpha}
\le \alpha.
\]
By H\"older's inequality and~\eqref{eq:q-exp-tight}, for each \(j\le r\),
\[
\mathbb{E}[\exp\{r\alpha\gamma W_{j,t}^2\}]
=
\mathbb{E}\left[
	\left(x_{j,t}\exp\{\kappa_0 W_{j,t}^2\}\right)^{r\alpha\gamma/\kappa_0}\cdot
	x_{j,t}^{-r\alpha\gamma/\kappa_0}
\right]
\le
C\left(\max_{1\le i\le r}\mathbb{E}[1+x_{i,t}^{-\alpha}]\right)^{(\kappa_0-r\alpha\gamma)/\kappa_0}.
\]
	Applying H\"older once more over the \(r\) arms, with all H\"older exponents equal to \(r\), gives
	\begin{equation}\label{eq:eq-means-sum-exp-bound}
	\begin{aligned}
	\mathbb{E}\left[\exp\left\{\alpha\gamma\sum_{j=1}^r W_{j,t}^2\right\}\right]
	\le
	\prod_{j=1}^r
	\left(\mathbb{E}\left[\exp\{r\alpha\gamma W_{j,t}^2\}\right]\right)^{1/r} \le
	C\left(\max_{1\le i\le r}\mathbb{E}[1+x_{i,t}^{-\alpha}]\right)^{(\kappa_0-r\alpha\gamma)/\kappa_0}.
	\end{aligned}
	\end{equation}
	Taking expectations in~\eqref{eq:q-inv-drift}, using~\eqref{eq:eq-means-sum-exp-bound}, and maximizing over \(i\) yields
	\[
	\max_{1\le i\le r}\mathbb{E}[1+x_{i,t+1}^{-\alpha}]
	\le
	\left(1-\frac{c}{t}\right)
	\max_{1\le i\le r}\mathbb{E}[1+x_{i,t}^{-\alpha}]
	+\frac{C}{t}
\left(\max_{1\le i\le r}\mathbb{E}[1+x_{i,t}^{-\alpha}]\right)^{(\kappa_0-r\alpha\gamma)/\kappa_0}
+\frac{C}{t}.
\]
After decreasing \(c\) if necessary, assume \(0<c\le1\). Since \((\kappa_0-r\alpha\gamma)/\kappa_0<1\), choose \(M<\infty\) so large that
$
Cz^{(\kappa_0-r\alpha\gamma)/\kappa_0}+C\le  c z/2,
$ for $z\ge M.$
If
$
\max_{1\le i\le r}\mathbb{E}[1+x_{i,t}^{-\alpha}]\ge M,
$
then the preceding recursion gives
\[
\max_{1\le i\le r}\mathbb{E}[1+x_{i,t+1}^{-\alpha}]
\le
\left(1-\frac{c}{2t}\right)
\max_{1\le i\le r}\mathbb{E}[1+x_{i,t}^{-\alpha}]
\le
\max_{1\le i\le r}\mathbb{E}[1+x_{i,t}^{-\alpha}].
\]
If instead
$
\max_{1\le i\le r}\mathbb{E}[1+x_{i,t}^{-\alpha}]<M,
$
then the same recursion gives
\[
\max_{1\le i\le r}\mathbb{E}[1+x_{i,t+1}^{-\alpha}]
\le
M+C M^{(\kappa_0-r\alpha\gamma)/\kappa_0}+C.
\]
	Thus the deterministic recursion implies
	\begin{equation}\label{eq:eq-means-inv-moment}
	\sup_{t\ge r}\max_{1\le i\le r}
	\mathbb{E}[1+x_{i,t}^{-\alpha}]<\infty .
	\end{equation}
	Combining~\eqref{eq:q-exp-tight} with~\eqref{eq:eq-means-inv-moment}, for any sufficiently small \(0<\kappa_r<\kappa_0\) satisfying
	\(\kappa_r/(\kappa_0-\kappa_r)\le\alpha\), H\"older's inequality gives, uniformly over \(1\le i\le r\) and \(t\ge r\),
\[
\begin{aligned}
\mathbb{E}[\exp\{\kappa_r W_{i,t}^2\}]
&=
\mathbb{E}\left[
\left(x_{i,t}\exp\{\kappa_0 W_{i,t}^2\}\right)^{\kappa_r/\kappa_0}
x_{i,t}^{-\kappa_r/\kappa_0}
\right]\\
&\le
\left(\mathbb{E}\left[x_{i,t}\exp\{\kappa_0 W_{i,t}^2\}\right]\right)^{\kappa_r/\kappa_0} \cdot
\left(\mathbb{E}\left[x_{i,t}^{-\kappa_r/(\kappa_0-\kappa_r)}\right]\right)^{(\kappa_0-\kappa_r)/\kappa_0}
\le C_r,
\end{aligned}
\]
	where the last step uses \(x_{i,t}\le1\), \(\kappa_r/(\kappa_0-\kappa_r)\le\alpha\), and~\eqref{eq:eq-means-inv-moment}.

	Here \(\kappa_r\) and \(C_r\) depend only on \(r\), since \(\kappa_0,\kappa_1,\kappa_2\) are universal and \(\alpha\) was fixed once and for all.  The first-moment bound follows from \(|x|\le C_{\kappa_r}\exp\{\kappa_r x^2\}\), where \(C_{\kappa_r}\) depends only on \(\kappa_r\).
\end{proof}

\begin{lemma}[Weighted perturbation bound]\label{lem:weighted-perturb}
Let \(x\in\mathrm{int}(\Delta_r)\), let \(Z_1,\dots,Z_r\stackrel{\mathrm{i.i.d.}}{\sim}\mathcal{N}(0,1)\), and define
\[
g_i(x):=\mathbb P\left(\frac{Z_i}{\sqrt{x_i}}=\max_{1\le j\le r}\frac{Z_j}{\sqrt{x_j}}\right).
\]
For \(\rho=(\rho_1,\dots,\rho_r)\in\mathbb R^r\), define
\[
p_i(x,\rho):=\mathbb P\left(\frac{Z_i+\rho_i}{\sqrt{x_i}}=\max_{1\le j\le r}\frac{Z_j+\rho_j}{\sqrt{x_j}}\right).
\]
Then there exists \(C_r<\infty\), depending only on \(r\), such that
\[
\sum_{i=1}^r x_i|p_i(x,\rho)-g_i(x)|
\le C_r\sum_{i=1}^r |\rho_i|.
\]
  Consequently, in the equal-means TS process,
\begin{equation}\label{eq:weighted-perturb-ts}
\sum_{i=1}^r x_{i,t}|p_{i,t}-g_i(x_t)|
\le \frac{C_r}{\sqrt{\variance}}
\sum_{i=1}^r \sqrt{N_{i,t}}|\widehat\mu_{i,t}-\mu|.
\end{equation}
  \end{lemma}

\begin{proof}
Let \(J_0:=\argmax_i Z_i/\sqrt{x_i}\) and \(J_\rho:=\argmax_i (Z_i+\rho_i)/\sqrt{x_i}\).  The Gaussian differences are nondegenerate, so ties have probability zero.  Since \(\mathbb P(J_0=i)=g_i(x)\) and \(\mathbb P(J_\rho=i)=p_i(x,\rho)\),
\[
\sum_i x_i|p_i(x,\rho)-g_i(x)|
\le \sum_i |p_i(x,\rho)-g_i(x)|
\le 2\mathbb P(J_\rho\ne J_0).
\]
If \(J_\rho\ne J_0\), then for some pair \(i<j\) the relative ordering of \(Z_i/\sqrt{x_i}\) and \(Z_j/\sqrt{x_j}\) changes after adding \(\rho_i/\sqrt{x_i}\) and \(\rho_j/\sqrt{x_j}\). Therefore, by a union bound,
\[
\mathbb P(J_\rho\ne J_0)
\le
\sum_{i<j}
\mathbb P\left(
\left|\frac{Z_i}{\sqrt{x_i}}-\frac{Z_j}{\sqrt{x_j}}\right|
\le
\left|\frac{\rho_i}{\sqrt{x_i}}-\frac{\rho_j}{\sqrt{x_j}}\right|
\right).
\]
For a fixed pair \(i<j\), \(Z_i/\sqrt{x_i}-Z_j/\sqrt{x_j}\) is centered Gaussian with variance \((x_i+x_j)/(x_ix_j)\), so its density is bounded by \(C\sqrt{x_ix_j/(x_i+x_j)}\). Hence
\[
\mathbb P\left(
\left|\frac{Z_i}{\sqrt{x_i}}-\frac{Z_j}{\sqrt{x_j}}\right|
\le
\left|\frac{\rho_i}{\sqrt{x_i}}-\frac{\rho_j}{\sqrt{x_j}}\right|
\right)
\le
C\frac{|\rho_i\sqrt{x_j}-\rho_j\sqrt{x_i}|}{\sqrt{x_i+x_j}}
\le C(|\rho_i|+|\rho_j|).
\]
Summing this fixed-pair bound over \(i<j\) gives
\(\mathbb P(J_\rho\ne J_0)\le C_r\sum_i|\rho_i|\).
Taking the union bound over all pairs yields \(\mathbb P(J_\rho\ne J_0)\le C_r\sum_i|\rho_i|\), proving the first claim.  In the equal-means TS process, after multiplying all indices by \(\sqrt{t/\variance}\), arm \(i\) is selected by comparing \((Z_i+\sqrt{N_{i,t}}(\widehat\mu_{i,t}-\mu)/\sqrt{\variance})/\sqrt{x_{i,t}}\).  Thus \(\rho_{i,t}=\sqrt{N_{i,t}}(\widehat\mu_{i,t}-\mu)/\sqrt{\variance}\), which gives~\eqref{eq:weighted-perturb-ts}.
\end{proof}

\paragraph{Lyapunov drift and convergence to uniform allocation under \(\variance\to\infty\).}
Define
\[
V_t:=\sum_{i=1}^r\left(x_{i,t}-\frac{1}r\right)^2
=\sum_{i=1}^r x_{i,t}^2-\frac{1}r.
\]
A direct expansion of the update gives
\begin{align*}
	\sum_{i=1}^r x_{i,t+1}^2
	&=\sum_{i=1}^r\left(x_{i,t}+\frac{1}{t+1}(\mathbf{1}\{A_{t+1}=i\}-x_{i,t})\right)^2\\
	&=\sum_{i=1}^r x_{i,t}^2
	+\frac{2}{t+1}\sum_{i=1}^r x_{i,t}(\mathbf{1}\{A_{t+1}=i\}-x_{i,t})
	+\frac{1}{(t+1)^2}\sum_{i=1}^r (\mathbf{1}\{A_{t+1}=i\}-x_{i,t})^2.
\end{align*}
Using \(\mathbb{E}[\mathbf{1}\{A_{t+1}=i\}\mid\mathcal F_t]=p_{i,t}\), we obtain
\begin{equation}
\label{eq:eqmeans-refined-drift0}
	\begin{aligned}
		\mathbb{E}[V_{t+1}\mid\mathcal F_t]
		&=V_t+\frac{2}{t+1}\Big(\sum_{i=1}^r x_{i,t}p_{i,t}-\sum_{i=1}^r x_{i,t}^2\Big)
		+\frac{1}{(t+1)^2}\Big(1-2\sum_{i=1}^r x_{i,t}p_{i,t}+\sum_{i=1}^r x_{i,t}^2\Big)\\
		&\le
		V_t+\frac{2}{t+1}\Big(\sum_{i=1}^r x_{i,t}p_{i,t}-\sum_{i=1}^r x_{i,t}^2\Big)
		+\frac{2}{(t+1)^2},
	\end{aligned}
\end{equation}
The final line of~\eqref{eq:eqmeans-refined-drift0} uses \(0\le \sum_i x_{i,t}p_{i,t}\le 1\) and \(\sum_i x_{i,t}^2\le 1\).
Decompose the drift into the pure-noise part and the perturbation part.  Lemma~\ref{lem:dotprod} gives
\[
\sum_i x_{i,t}g_i(x_t)-\sum_i x_{i,t}^2
\le -V_t.
\]
By Lemma~\ref{lem:weighted-perturb},
\[
\sum_i x_{i,t}(p_{i,t}-g_i(x_t))
\le
\sum_i x_{i,t}|p_{i,t}-g_i(x_t)|
\le
\frac{C_r}{\sqrt{\variance}}
\sum_{i=1}^r \sqrt{N_{i,t}}|\widehat\mu_{i,t}-\mu|.
\]
Substituting into~\eqref{eq:eqmeans-refined-drift0} gives
\[
\mathbb{E}[V_{t+1}\mid\mathcal F_t]
\le
\left(1-\frac{2}{t+1}\right)V_t
+
\frac{C_r}{(t+1)\sqrt{\variance}}
\sum_{i=1}^r\sqrt{N_{i,t}}|\widehat\mu_{i,t}-\mu|
+
\frac{2}{(t+1)^2}.
\]
Taking expectations and using Lemma~\ref{lem:eq-exp-tight}, we obtain, for all sufficiently large horizons,
\[
\mathbb{E}[V_{t+1}]
\le
\left(1-\frac{2}{t+1}\right)\mathbb{E}[V_t]
+
\frac{C_r}{(t+1)\sqrt{\variance}}
+
\frac{2}{(t+1)^2}.
\]
Multiplying the last recursion by \(t(t+1)\) and summing from \(t=r\) to \(T-1\) gives
\[
T(T-1)\mathbb{E}[V_T]
\le
r(r-1)\mathbb{E}[V_r]
+
\frac{C_r}{\sqrt{\variance}}\sum_{t=r}^{T-1}t
+
2\sum_{t=r}^{T-1}\frac{t}{t+1}.
\]
Since \(0\le V_t\le1\), it follows that
\[
\mathbb{E}[V_T]
\le
\frac{r(r-1)+2T}{T(T-1)}
+
\frac{C_r}{T(T-1)\sqrt{\variance}}\sum_{t=r}^{T-1}t
=
O(T^{-1})+O(\variance^{-1/2}).
\]
Since \(\variance\to\infty\), we have \(\mathbb{E}[V_T]\to0\), and therefore \(V_T\xrightarrow{\mathbb P}0\).
Equivalently,
\[
\frac{N_{i,T}}{T}=x_{i,T}\xrightarrow{\mathbb P}\frac{1}r,
\qquad i=1,\dots,r.
\]
This proves the desired result.

 \subsubsection{Step 2. Finishing the proof for Theorem~\ref{thm:main}(a).}

  We now extend the refined argument to the full bandit.  The key point is that suboptimal arms need not have inverse-proportion moments, since their proportions vanish.  Instead, for suboptimal arms we only control the positive standardized overshoot above the optimal level.  This yields a full-bandit exponential-tightness estimate for the optimal arms.

 By~\eqref{equ:209} and \(\variance L_T^2/T\to0\), we have \(N_{\mathrm{sub}}(T)/T\xrightarrow{\mathbb P}0\) and \(N_{\mathrm{opt}}(T)/T\xrightarrow{\mathbb P}1\), where \(N_{\mathrm{opt}}(T):=\sum_{i\in\mathcal S^\star}N_{i,T}=T-N_{\mathrm{sub}}(T)\).
 If \(m:=|\mathcal S^\star|=1\), then this already gives
\(N_{i,T}/T\to_p1\) for the unique optimal arm.  Hence assume below that \(m\ge2\).

  \begin{lemma}[Full-bandit exponential tightness for optimal standardized means]\label{lem:full-opt-exp-tight}
Assume the reward noises are independent mean-zero 1-sub-Gaussian variables, and suppose the Thompson sampling indices use conditionally independent standard Gaussian sampling noises.  If \(\variance\to\infty\), then for all sufficiently large \(T\) there exist constants \(\kappa_m>0\) and \(C_m<\infty\), depending on the fixed instance but not on \(t\) or \(T\), such that
\[
\sup_{i\in\mathcal S^\star}\sup_{K\le t\le T}
\mathbb{E}\left[\exp\left\{\kappa_m N_{i,t}(\widehat\mu_{i,t}-\mu^\star)^2\right\}\right]
\le C_m.
\]
	The same constants also give
	\begin{equation}\label{eq:full-opt-W-tight-new}
\sup_{K\le t\le T}\sum_{i\in\mathcal S^\star}
\mathbb{E}\left[\sqrt{N_{i,t}}\left|(\widehat\mu_{i,t}-\mu^\star)\right|\right]
\le C_m.
\end{equation}
\end{lemma}

\begin{proof}
We shift rewards by \(-\mu^\star\).  Thus optimal arms have mean zero, while each suboptimal arm \(j\notin\mathcal S^\star\) has mean \(-\Delta_j\), where \(\Delta_j=\mu^\star-\mu_j>0\).

Use the universal constants \(\kappa_0,\kappa_1,\kappa_2\) from~\eqref{eq:single-step-exp-contraction}, and set \(\gamma:=\kappa_0/(4m)\).  Since \(\variance\to\infty\), for all sufficiently large \(T\) we have \(\gamma>1/(2\variance)\).  All constants denoted by \(c,C\) below may change from line to line and may depend on the fixed bandit instance and on \(\kappa_0,\kappa_1,\kappa_2,\gamma\), and on the fixed choice of \(\alpha\) after Step~4, but not on \(t\), \(T\), or \(\variance\) once \(\variance\) is large enough.

\vspace{3pt}
\noindent\textbf{Step 1: weighted exponential moment for optimal arms.}
For each \(i\in\mathcal S^\star\), based on~\eqref{eq:single-step-exp-contraction}, repeating the derivation of the one-step recursion~\eqref{eq:eq-means-weighted-one-step} gives
$
\mathbb{E}[x_{i,t+1}\exp\{\kappa_0 W_{i,t+1}^2\}\mid\mathcal F_t]
\le
\{t/(t+1)\}x_{i,t}\exp\{\kappa_0 W_{i,t}^2\}
+{C}/(t+1).
$
Taking expectations and iterating over \(K\le t\le T\) further yields
\begin{equation}\label{eq:full-q-exp-tight}
\sup_{K\le t\le T}
\mathbb{E}\left[x_{i,t}\exp\{\kappa_0 W_{i,t}^2\}\right]
\le C.
\end{equation}

\vspace{3pt}
\noindent\textbf{Step 2: exponential tails for suboptimal overshoots.}
For each \(j\notin\mathcal S^\star\), pathwisely we have
\[
\left(\sqrt{N_{j,t}}(\widehat\mu_{j,t}-\mu^\star)\right)_+
\le
\sup_{\substack{r\le T:\\ N_{j,r}\ge1}}
\left(
\frac{\sum_{s\le r:\,A_s=j}\xi_{j,s}-\Delta_j N_{j,r}}
{\sqrt{N_{j,r}}}
\right)_+.
\]
{For \(x\ge0\), note that \(\Delta_j>0\) is fixed.  Grouping the event by
\(n=N_{j,r}\), applying the fixed-\(n\) optional-stopping sub-Gaussian bound
to \(\sum_{s\le r:A_s=j}\xi_{j,s}\), and then union bounding over \(n\), gives,
with constants that may depend on \(\Delta_j\),}
\[
\mathbb P\left(
\sup_{\substack{r\le T:\\ N_{j,r}\ge1}}
\frac{\sum_{s\le r:\,A_s=j}\xi_{j,s}-\Delta_j N_{j,r}}
{\sqrt{N_{j,r}}}
\ge x
\right)
\le
\sum_{n\ge1}\exp\left\{-\frac{(\Delta_j\sqrt n+x)^2}{2}\right\}
\le C\exp\{-cx^2\}.
\]
Therefore, if \([K]\setminus\mathcal S^\star\) is nonempty, the finiteness of this set gives constants \(\lambda_\star>0\) and \(C<\infty\), depending only on the fixed bandit instance, such that, for every \(0<\lambda\le\lambda_\star\),
\begin{equation}\label{eq:subopt-overshoot-exp}
\sup_{j\notin\mathcal S^\star}
\sup_{K\le t\le T}
\mathbb{E}\left[
\exp\left\{\lambda\left(\sqrt{N_{j,t}}(\widehat\mu_{j,t}-\mu^\star)\right)_+^2\right\}
\right]
\le C.
\end{equation}
When \([K]\setminus\mathcal S^\star\) is empty, set \(\lambda_\star:=1\); all sums over suboptimal arms below are empty and~\eqref{eq:subopt-overshoot-exp} is not used.

\vspace{3pt}
\noindent\textbf{Step 3: global lower bound for the probability of sampling an optimal arm.}
For \(i\in\mathcal S^\star\), let \(p_{i,t}:=\mathbb P(A_{t+1}=i\mid\mathcal F_t)\).
For all sufficiently large \(T\), there exists \(c>0\) such that
\begin{equation}\label{eq:full-opt-global-lower}
p_{i,t}
\ge
c\exp\left\{-\gamma\left(
\sum_{\ell\in\mathcal S^\star}W_{\ell,t}^2
+
\sum_{j\notin\mathcal S^\star}
\left(\sqrt{N_{j,t}}(\widehat\mu_{j,t}-\mu^\star)\right)_+^2
\right)\right\},
\qquad i\in\mathcal S^\star.
\end{equation}
Indeed, conditional on \(\mathcal F_t\), the Thompson index of any arm \(a\) can be written as
$
\widehat\mu_{a,t}+\sqrt{{\variance}/{N_{a,t}}}Z_a
=
\mu^\star+
(
\sqrt{N_{a,t}}(\widehat\mu_{a,t}-\mu^\star)+\sqrt{\variance}Z_a
)/{
\sqrt{N_{a,t}}
}.
$
Consider the event that \(W_{i,t}+\sqrt{\variance}Z_i\ge1\), that
\(W_{\ell,t}+\sqrt{\variance}Z_\ell\le0\) for every
\(\ell\in\mathcal S^\star\) with \(\ell\ne i\), and that
\[
\sqrt{N_{j,t}}(\widehat\mu_{j,t}-\mu^\star)+\sqrt{\variance}Z_j\le0,
\qquad j\notin\mathcal S^\star.
\]
On this event, the index of arm \(i\) is at least
\(\mu^\star+1/\sqrt{N_{i,t}}\), while the index of every other optimal arm
and every suboptimal arm is at most \(\mu^\star\). Hence this event is
contained in \(\{A_{t+1}=i\}\).

We next lower bound the conditional probability of this event. Since
\(\gamma\) is fixed and \(\variance\to\infty\), for all sufficiently large
\(T\) the gap between \(\gamma\) and \(1/(2\variance)\) is large enough to
absorb the polynomial prefactor in the Mills lower bound from
Lemma~\ref{lem:normal-tail}. Thus, after decreasing \(c\) if necessary,
\[
\mathbb P(W_{i,t}+\sqrt{\variance}Z_i\ge1\mid\mathcal F_t)
\ge
c\exp\{-\gamma W_{i,t}^2\}.
\]
For the other optimal arms, the same Gaussian tail lower bound gives
\[
\mathbb P(W_{\ell,t}+\sqrt{\variance}Z_\ell\le0\mid\mathcal F_t)
\ge
c\exp\{-\gamma W_{\ell,t}^2\},
\qquad \ell\in\mathcal S^\star,\ \ell\ne i.
\]
For a suboptimal arm \(j\notin\mathcal S^\star\), if
\(\sqrt{N_{j,t}}(\widehat\mu_{j,t}-\mu^\star)\le0\), then the left-hand side
below is at least \(1/2\); otherwise Mills' lower bound applies at
\(\sqrt{N_{j,t}}(\widehat\mu_{j,t}-\mu^\star)/\sqrt{\variance}\). Hence,
again after decreasing \(c\) if necessary,
\[
\mathbb P\left(
\sqrt{N_{j,t}}(\widehat\mu_{j,t}-\mu^\star)+\sqrt{\variance}Z_j\le0
\mid\mathcal F_t
\right)
\ge
c\exp\left\{
-\gamma
\left(\sqrt{N_{j,t}}(\widehat\mu_{j,t}-\mu^\star)\right)_+^2
\right\}.
\]
Conditional on \(\mathcal F_t\), the quantities involving empirical means
are fixed and the Gaussian sampling noises are independent across arms.
Multiplying the preceding lower bounds and absorbing the resulting constant
into \(c\) proves~\eqref{eq:full-opt-global-lower}.

\vspace{3pt}
\noindent\textbf{Step 4: inverse moments of optimal proportions.}
Choose \(0<\alpha\le1\) small enough that \(2\alpha\gamma |[K]\setminus\mathcal S^\star|<\lambda_\star\).  For \(i\in\mathcal S^\star\), the first inequality in the one-step drift calculation~\eqref{eq:eq-means-inv-one-step} gives
\[
\mathbb{E}[x_{i,t+1}^{-\alpha}\mid\mathcal F_t]
\le
x_{i,t}^{-\alpha}
+\frac1t\left(Cx_{i,t}^{-\alpha}
-c x_{i,t}^{-\alpha}\frac{p_{i,t}}{x_{i,t}}\right).
\]
Combining~\eqref{eq:full-opt-global-lower} with the preceding display and then using \(C_0z^\alpha-c_0z^{\alpha+1}\le C_1-c_1z^\alpha\) for suitable positive constants \(C_1,c_1\) and all \(z\ge0\), with \(z\) equal to \(x_{i,t}^{-1}\) times the exponential factor in~\eqref{eq:full-opt-global-lower}, gives the following bound after renaming constants:
\begin{equation}\label{eq:full-inv-drift}
\mathbb{E}[x_{i,t+1}^{-\alpha}\mid\mathcal F_t]
\le
\left(1-\frac{c}{t}\right)x_{i,t}^{-\alpha}
+\frac{C}{t}
\exp\left\{\alpha\gamma\left(
\sum_{\ell\in\mathcal S^\star}W_{\ell,t}^2
+
\sum_{j\notin\mathcal S^\star}
\left(\sqrt{N_{j,t}}(\widehat\mu_{j,t}-\mu^\star)\right)_+^2
\right)\right\}.
\end{equation}
By Cauchy--Schwarz inequality, we have
		\[
		\begin{aligned}
		&\mathbb{E}\left[
		\exp\left\{\alpha\gamma\left(
		\sum_{\ell\in\mathcal S^\star}W_{\ell,t}^2
		+
		\sum_{j\notin\mathcal S^\star}
		\left(\sqrt{N_{j,t}}(\widehat\mu_{j,t}-\mu^\star)\right)_+^2
		\right)\right\}
		\right] \\
		&\qquad\le
		\left(\mathbb{E}\left[\exp\left\{2\alpha\gamma
		\sum_{\ell\in\mathcal S^\star}W_{\ell,t}^2\right\}\right]\right)^{1/2}
		\left(\mathbb{E}\left[\exp\left\{2\alpha\gamma
		\sum_{j\notin\mathcal S^\star}
		\left(\sqrt{N_{j,t}}(\widehat\mu_{j,t}-\mu^\star)\right)_+^2
		\right\}\right]\right)^{1/2}.
		\end{aligned}
		\]
			If there is no suboptimal arm, the second factor equals one. Otherwise, H\"older's inequality over the finite set \([K]\setminus\mathcal S^\star\), the choice \(2\alpha\gamma |[K]\setminus\mathcal S^\star|<\lambda_\star\), and~\eqref{eq:subopt-overshoot-exp} bound this second factor by a constant. Hence
		\begin{equation}\label{eq:full-exp-reduction-to-opt}
		\begin{aligned}
		\mathbb{E}\left[
	\exp\left\{\alpha\gamma\left(
\sum_{\ell\in\mathcal S^\star}W_{\ell,t}^2
+
\sum_{j\notin\mathcal S^\star}
\left(\sqrt{N_{j,t}}(\widehat\mu_{j,t}-\mu^\star)\right)_+^2
\right)\right\}
\right]\le
C\left(\mathbb{E}\left[\exp\left\{2\alpha\gamma
	\sum_{\ell\in\mathcal S^\star}W_{\ell,t}^2
	\right\}\right]\right)^{1/2}.
			\end{aligned}
			\end{equation}
			The choice \(\gamma=\kappa_0/(4m)\) and the bound \(\alpha\le1\) imply
		\[
		2m\alpha\gamma=\frac{\alpha\kappa_0}{2}<\kappa_0,
\qquad
\frac{2m\alpha\gamma}{\kappa_0-2m\alpha\gamma}
		=
		\frac{\alpha}{2-\alpha}
		\le\alpha.
		\]
			For each \(\ell\in\mathcal S^\star\), H\"older's inequality and~\eqref{eq:full-q-exp-tight} give
			\[
			\begin{aligned}
		\mathbb{E}\left[\exp\{2m\alpha\gamma W_{\ell,t}^2\}\right]
		=
		\mathbb{E}\left[
		\left(x_{\ell,t}\exp\{\kappa_0 W_{\ell,t}^2\}\right)^{2m\alpha\gamma/\kappa_0}
		x_{\ell,t}^{-2m\alpha\gamma/\kappa_0}
		\right]\le
			C\left(\max_{h\in\mathcal S^\star}\mathbb{E}[1+x_{h,t}^{-\alpha}]\right)^{(\kappa_0-2m\alpha\gamma)/\kappa_0},
			\end{aligned}
			\]
				where we used \(2m\alpha\gamma/(\kappa_0-2m\alpha\gamma)\le\alpha\). Applying H\"older once more over the \(m\) optimal arms, with all H\"older exponents equal to \(m\), gives
			\begin{equation}\label{eq:full-opt-sum-exp-bound}
			\begin{aligned}
		\mathbb{E}\left[\exp\left\{2\alpha\gamma
		\sum_{\ell\in\mathcal S^\star}W_{\ell,t}^2
		\right\}\right]
			 \le
			\prod_{\ell\in\mathcal S^\star}
			\left(\mathbb{E}\left[\exp\{2m\alpha\gamma W_{\ell,t}^2\}\right]\right)^{1/m}
			\le
			C\left(\max_{h\in\mathcal S^\star}\mathbb{E}[1+x_{h,t}^{-\alpha}]\right)^{(2-\alpha)/2}.
			\end{aligned}
			\end{equation}
		Combining~\eqref{eq:full-exp-reduction-to-opt} and~\eqref{eq:full-opt-sum-exp-bound} gives
	\begin{equation}\label{eq:full-total-exp-bound}
	\mathbb{E}\left[
\exp\left\{\alpha\gamma\left(
\sum_{\ell\in\mathcal S^\star}W_{\ell,t}^2
+
\sum_{j\notin\mathcal S^\star}
\left(\sqrt{N_{j,t}}(\widehat\mu_{j,t}-\mu^\star)\right)_+^2
\right)\right\}
\right]
		\le
		C\left(\max_{h\in\mathcal S^\star}\mathbb{E}[1+x_{h,t}^{-\alpha}]\right)^{(2-\alpha)/4}.
		\end{equation}
		Taking expectations in~\eqref{eq:full-inv-drift}, using~\eqref{eq:full-total-exp-bound}, and maximizing over \(i\in\mathcal S^\star\) gives
	\[
	\max_{i\in\mathcal S^\star}\mathbb{E}[1+x_{i,t+1}^{-\alpha}]
\le
\left(1-\frac{c}{t}\right)
\max_{i\in\mathcal S^\star}\mathbb{E}[1+x_{i,t}^{-\alpha}]
+\frac{C}{t}
	\left(\max_{i\in\mathcal S^\star}\mathbb{E}[1+x_{i,t}^{-\alpha}]\right)^{(2-\alpha)/4}
	+\frac{C}{t}.
	\]
			Since \((2-\alpha)/4<1\), the scalar recursion argument used after~\eqref{eq:q-inv-drift} implies
			\begin{equation}\label{eq:full-opt-inv-moment}
			\sup_{K\le t\le T}\max_{i\in\mathcal S^\star}\mathbb{E}[1+x_{i,t}^{-\alpha}]\le C.
			\end{equation}

\vspace{3pt}
\noindent\textbf{Step 5: unweighted exponential moment of optimal standardized means.}
	Choose \(0<\kappa_m<\kappa_0\) so small that \(\kappa_m/(\kappa_0-\kappa_m)\le\alpha\).  By H\"older's inequality,~\eqref{eq:full-q-exp-tight}, and~\eqref{eq:full-opt-inv-moment}, for every \(i\in\mathcal S^\star\) and \(K\le t\le T\),
\[
\mathbb{E}[\exp\{\kappa_m W_{i,t}^2\}]
=
\mathbb{E}\left[
\left(x_{i,t}\exp\{\kappa_0 W_{i,t}^2\}\right)^{\kappa_m/\kappa_0}
x_{i,t}^{-\kappa_m/\kappa_0}
\right]
\le C_m.
\]
Since \(W_{i,t}=\sqrt{N_{i,t}}(\widehat\mu_{i,t}-\mu^\star)\) for \(i\in\mathcal S^\star\), this proves the desired exponential-moment bound.  The consequent first-moment bound follows from \(|x|\le C_{\kappa_m}\exp\{\kappa_m x^2\}\), where \(C_{\kappa_m}\) depends only on \(\kappa_m\).
The constants \(\kappa_m\) and \(C_m\) depend only on the fixed bandit instance, since \(\kappa_0,\kappa_1,\kappa_2\) are universal and \(\alpha,\lambda_\star\) were chosen using only that instance.

\end{proof}

We now   finish the proof of Theorem~\ref{thm:main}(a).  For \(i\in\mathcal S^\star\) , define the within-optimal proportions
  \[
y_{i,t}:=\frac{N_{i,t}}{N_{\mathrm{opt}}(t)},
\qquad
V_t^\star:=\sum_{i\in\mathcal S^\star}\left(y_{i,t}-\frac{1}m\right)^2.
\]
  If \(A_{t+1}\notin\mathcal S^\star\), then \((N_{i,t})_{i\in\mathcal S^\star}\) is unchanged and  therefore   \(V_{t+1}^\star=V_t^\star\).  If \(A_{t+1}=i\in\mathcal S^\star\), then
\[
y_{j,t+1}=y_{j,t}+\frac{1}{N_{\mathrm{opt}}(t)+1}\bigl(\mathbf 1\{j=i\}-y_{j,t}\bigr),
\qquad j\in\mathcal S^\star.
\]
Let \(p_{i,t}:=\mathbb P(A_{t+1}=i\mid\mathcal F_t)\) for \(i\in\mathcal S^\star\), and let \(\pi_{i,t}^{\star}\) denote the winner probability among the optimal arms only:
\[
\pi_{i,t}^{\star}:=
\mathbb P\left(
 i=\argmax_{j\in\mathcal S^\star}
 \left\{\widehat\mu_{j,t}+\sqrt{\frac{\variance}{N_{j,t}}}Z_{j,t+1}\right\}
\middle|\mathcal F_t
\right).
\]
If the global winner is an optimal arm \(i\), then \(i\) is also the winner among the optimal arms.  Thus \(0\le p_{i,t}\le \pi_{i,t}^{\star}\), and \(\sum_{i\in\mathcal S^\star}(\pi_{i,t}^{\star}-p_{i,t})=\mathbb P(A_{t+1}\notin\mathcal S^\star\mid\mathcal F_t)\).
Here is the square expansion in calendar time.  If the algorithm pulls \(i\in\mathcal S^\star\), then
\[
V_{t+1}^\star-V_t^\star
=
\frac{2}{N_{\mathrm{opt}}(t)+1}
\sum_{j\in\mathcal S^\star}y_{j,t}
\bigl(\mathbf 1\{j=i\}-y_{j,t}\bigr)
+
\frac{1}{(N_{\mathrm{opt}}(t)+1)^2}
\sum_{j\in\mathcal S^\star}
\bigl(\mathbf 1\{j=i\}-y_{j,t}\bigr)^2.
\]
If \(A_{t+1}\notin\mathcal S^\star\), then \(V_{t+1}^\star=V_t^\star\).  Therefore, taking conditional expectation gives
\begin{align}
\mathbb{E}[V_{t+1}^\star-V_t^\star\mid\mathcal F_t]
&=
\frac{2}{N_{\mathrm{opt}}(t)+1}
\left(
\sum_{i\in\mathcal S^\star}y_{i,t}p_{i,t}
-
\left(\sum_{i\in\mathcal S^\star}p_{i,t}\right)
\sum_{i\in\mathcal S^\star}y_{i,t}^2
\right) \notag\\
&\qquad
+
\frac{1}{(N_{\mathrm{opt}}(t)+1)^2}
\sum_{i\in\mathcal S^\star}p_{i,t}
\sum_{j\in\mathcal S^\star}
\bigl(\mathbf 1\{j=i\}-y_{j,t}\bigr)^2 \notag\\
&\le
\frac{2}{N_{\mathrm{opt}}(t)+1}
\left(
\sum_{i\in\mathcal S^\star}y_{i,t}p_{i,t}
-
\left(\sum_{i\in\mathcal S^\star}p_{i,t}\right)
\sum_{i\in\mathcal S^\star}y_{i,t}^2
\right)
+
\frac{2}{(N_{\mathrm{opt}}(t)+1)^2},
\label{eq:full-opt-drift-start-new}
\end{align}
The final line of~\eqref{eq:full-opt-drift-start-new} uses
\(\sum_{j\in\mathcal S^\star}(\mathbf 1\{j=i\}-y_{j,t})^2\le2\).

	Using \(p_{i,t}=\pi_{i,t}^{\star}-(\pi_{i,t}^{\star}-p_{i,t})\), we decompose
	\begin{align}
	&\sum_{i\in\mathcal S^\star}y_{i,t}p_{i,t}
	-
	\left(\sum_{i\in\mathcal S^\star}p_{i,t}\right)
	\sum_{i\in\mathcal S^\star}y_{i,t}^2 \notag\\
	& \qquad =
	\left(
	\sum_{i\in\mathcal S^\star}y_{i,t}\pi_{i,t}^{\star}
-
\sum_{i\in\mathcal S^\star}y_{i,t}^2
\right)
-
\left(
	\sum_{i\in\mathcal S^\star}y_{i,t}(\pi_{i,t}^{\star}-p_{i,t})
	-
	\left(\sum_{i\in\mathcal S^\star}(\pi_{i,t}^{\star}-p_{i,t})\right)
	\sum_{i\in\mathcal S^\star}y_{i,t}^2
	\right).
	\label{eq:full-opt-drift-decomposition}
	\end{align}
	Since \(0\le y_{i,t}\le1\), \(0\le\sum_i y_{i,t}^2\le1\), and
	\(\sum_{i\in\mathcal S^\star}(\pi_{i,t}^{\star}-p_{i,t})
	=\mathbb P(A_{t+1}\notin\mathcal S^\star\mid\mathcal F_t)\),
	\begin{equation}\label{eq:full-opt-subopt-error-absolute}
	\left|
	\sum_{i\in\mathcal S^\star}y_{i,t}(\pi_{i,t}^{\star}-p_{i,t})
	-
	\left(\sum_{i\in\mathcal S^\star}(\pi_{i,t}^{\star}-p_{i,t})\right)
	\sum_{i\in\mathcal S^\star}y_{i,t}^2
	\right|
	\le
	2\mathbb P(A_{t+1}\notin\mathcal S^\star\mid\mathcal F_t).
	\end{equation}
	Combining~\eqref{eq:full-opt-drift-decomposition} and~\eqref{eq:full-opt-subopt-error-absolute} gives
	\begin{align}
\sum_{i\in\mathcal S^\star}y_{i,t}p_{i,t}
-
\left(\sum_{i\in\mathcal S^\star}p_{i,t}\right)
\sum_{i\in\mathcal S^\star}y_{i,t}^2 \le
\left(\sum_{i\in\mathcal S^\star}y_{i,t}\pi_{i,t}^{\star}
-
\sum_{i\in\mathcal S^\star}y_{i,t}^2\right)
+
2\mathbb P(A_{t+1}\notin\mathcal S^\star\mid\mathcal F_t).
\label{eq:subopt-blocking-error-new}
\end{align}
  Let \(g_i(y_t)\) be the pure-noise winner map in~\eqref{eq:def:g:variance}, restricted to the optimal arms.  By Lemma~\ref{lem:dotprod},
\[
\sum_{i\in\mathcal S^\star}y_{i,t}g_i(y_t)
-
\sum_{i\in\mathcal S^\star}y_{i,t}^2
\le -V_t^\star.
\]
Applying Lemma~\ref{lem:weighted-perturb} to the optimal arms only gives the conditional bound
\[
\sum_{i\in\mathcal S^\star}y_{i,t}|\pi_{i,t}^{\star}-g_i(y_t)|
\le
\frac{C_m}{\sqrt{\variance}}
\sum_{i\in\mathcal S^\star}
\left|N_{i,t}^{1/2}(\widehat\mu_{i,t}-\mu^\star)\right|.
\]
Combining the preceding inequalities yields the conditional drift
\begin{align}
\mathbb{E}[V_{t+1}^\star\mid\mathcal F_t]
&\le
V_t^\star
-
\frac{2}{N_{\mathrm{opt}}(t)+1}V_t^\star
+
\frac{C_m}{(N_{\mathrm{opt}}(t)+1)\sqrt{\variance}}
\sum_{i\in\mathcal S^\star}
\left|N_{i,t}^{1/2}(\widehat\mu_{i,t}-\mu^\star)\right|
\notag\\
&\qquad
+
C_m\frac{\mathbb P(A_{t+1}\notin\mathcal S^\star\mid\mathcal F_t)}{N_{\mathrm{opt}}(t)+1}
+
\frac{C_m}{(N_{\mathrm{opt}}(t)+1)^2}.
\label{eq:full-opt-conditional-drift-new}
\end{align}
  We now localize only to ensure that \(N_{\mathrm{opt}}(t)\) is of order \(t\).  Let \(t_0:=\lceil 2\variance L_T^2\rceil\), and define
\[
\zeta:=\inf\left\{t\ge K:
N_{\mathrm{sub}}(t)>\variance L_T^2\right\},
\]
   with \(\inf\emptyset=\infty\).  By~\eqref{equ:209}, \(\mathbb P(\zeta\le T)\to0\).  On \(\{\zeta>t\}\), for \(t\ge t_0\), \(N_{\mathrm{opt}}(t)=t-N_{\mathrm{sub}}(t)\ge t/2\).
We spell out the stopped recursion.
Since
\(\{\zeta>t+1\}\subseteq\{\zeta>t\}\) and \(\{\zeta>t\}\in\mathcal F_t\), the tower property gives
\[
\mathbb{E}\bigl[V_{t+1}^\star\mathbf 1\{\zeta>t+1\}\bigr] \le
\mathbb{E}\bigl[V_{t+1}^\star\mathbf 1\{\zeta>t\}\bigr]
=
\mathbb{E}\left[
\mathbf 1\{\zeta>t\}
\mathbb{E}[V_{t+1}^\star\mid\mathcal F_t]
\right].
\]
On \(\{\zeta>t\}\), for \(t\ge t_0\), we have \(N_{\mathrm{opt}}(t)\ge t/2\).  Applying~\eqref{eq:full-opt-conditional-drift-new} inside this expectation and using Lemma~\ref{lem:full-opt-exp-tight} gives
\[
\mathbb{E}\bigl[V_{t+1}^\star\mathbf 1\{\zeta>t+1\}\bigr]
\le
\left(1-\frac{2}{t+1}\right)\mathbb{E}\bigl[V_t^\star\mathbf 1\{\zeta>t\}\bigr]
+
\frac{C_m}{t\sqrt{\variance}}
+
\frac{C_m}{t}\mathbb{E}\bigl[\mathbb P(A_{t+1}\notin\mathcal S^\star\mid\mathcal F_t)\bigr]
+
\frac{C_m}{t^2}.
\]
Multiplying both sides by \(t+1\) yields
\[
(t+1)\mathbb{E}\bigl[V_{t+1}^\star\mathbf 1\{\zeta>t+1\}\bigr]
\le
(t-1)\mathbb{E}\bigl[V_t^\star\mathbf 1\{\zeta>t\}\bigr]
+
\frac{C_m}{\sqrt{\variance}}
+
C_m\mathbb{E}\bigl[\mathbb P(A_{t+1}\notin\mathcal S^\star\mid\mathcal F_t)\bigr]
+
\frac{C_m}{t}.
\]
Since \(\mathbb{E}[V_t^\star\mathbf 1\{\zeta>t\}]\ge0\), the first term is at most \(t\mathbb{E}[V_t^\star\mathbf 1\{\zeta>t\}]\).
Summing from \(t_0\) to \(T-1\), and using \(0\le V_{t_0}^\star\le1\), we obtain
\[
\mathbb{E}\bigl[V_T^\star\mathbf 1\{\zeta>T\}\bigr]
\le
\frac{t_0}{T}
+
\frac{C_m}{\sqrt{\variance}}
+
\frac{C_m}{T}
\sum_{t=t_0}^{T-1}
\mathbb{E}\bigl[\mathbb P(A_{t+1}\notin\mathcal S^\star\mid\mathcal F_t)\bigr]
+
\frac{C_m\log T}{T}.
\]
  Moreover, \(\sum_{t=K}^{T-1}\mathbb{E}[\mathbb P(A_{t+1}\notin\mathcal S^\star\mid\mathcal F_t)]\le \mathbb{E}[N_{\mathrm{sub}}(T)]=O(\variance L_T)\).  Under the assumptions \(\variance\to\infty\) and \(\variance(\log T)^2/T\to0\), we have \(t_0/T\to0\), \(1/\sqrt{\variance}\to0\), \(\mathbb{E}[N_{\mathrm{sub}}(T)]/T\to0\), and \(\log T/T\to0\).  Thus \(\mathbb{E}[V_T^\star\mathbf 1\{\zeta>T\}]\to0\).  Since \(\mathbb P(\zeta\le T)\to0\), we conclude that
\[
V_T^\star\xrightarrow{\mathbb P}0,
\qquad
\frac{N_{i,T}}{N_{\mathrm{opt}}(T)}\xrightarrow{\mathbb P}\frac{1}m,
\quad i\in\mathcal S^\star.
\]
  Finally, \(N_{\mathrm{opt}}(T)/T\xrightarrow{\mathbb P}1\) implies
\[
\frac{N_{i,T}}{T}
=
\frac{N_{i,T}}{N_{\mathrm{opt}}(T)}\frac{N_{\mathrm{opt}}(T)}{T}
\xrightarrow{\mathbb P}\frac{1}m,
\qquad i\in\mathcal S^\star.
\]
  This proves Theorem~\ref{thm:main}(a) under the stated growth conditions.

 \subsection{Proof of Theorem~\ref{thm:main}(b)} \label{appendix:thm:main:b}
In this subsection, we fix a suboptimal arm $a\notin \mathcal{S}^\star$ with gap
$\Delta_a := \mu^\star-\mu_a>0$, and prove the sharp asymptotics
\begin{equation}\label{eq:sharp-target}
	\frac{N_{a,T}}{\variance L_T}\ \xrightarrow{\mathbb P}\ \frac{2}{\Delta_a^2}.
\end{equation}
Throughout, $\{\mathcal F_t\}_{t\ge 0}$ denotes the natural filtration, and
\(\theta_{i,t+1}\) is the Gaussian sampling index drawn at round \(t+1\), conditionally
on \(\mathcal F_t\), with center \(\widehat\mu_{i,t}\) and variance
\(\variance/N_{i,t}\).
We will bound the probability of selecting suboptimal arm $a$ after it has been pulled $k$ times.
This yields an approximate geometric waiting time for the $(k+1)$-th pull.
Then summing these waiting times and applying a log-growth law for geometric sums will pin down $N_{a,T}$ at order \(\variance L_T=\variance\log(T/\variance)\) with the sharp constant.

\subsubsection{Step 1. Bounding the suboptimal selection probability and waiting time.}
Use the same burn-in level \(t_0:=\lceil 2\variance L_T^2\rceil\), recalling that \(L_T=\log(T/\variance)\).
To pull a suboptimal arm $a$, its Thompson sample must beat all optimal-arm samples.
So we first show that at least one optimal arm is well-sampled after burn-in, making \(\max_{i\in\mathcal S^\star}\theta_{i,t+1}\) unlikely to fall far below $\mu^\star$.

\begin{lemma}[Lower tail of the maximum optimal-arm index]\label{lem:theta-star-lower-tail}
	Recall the concentration event \(\mathcal E_T\) in~\eqref{equ:event} and the few-suboptimal-pulls event \(\Omega_T\) in~\eqref{equ:209}.
	Fix $\varepsilon\in(0,\Delta_a/8)$, there exists a constant $c>0$ such that on $\mathcal{E}_T\cap \Omega_T$,
	for every $t\in\{t_0,\dots,T-1\}$,
	\begin{equation}\label{eq:theta-star-lower-tail}
		\mathbb P\left(\max_{i\in\mathcal S^\star}\theta_{i,t+1}<\mu^\star-2\varepsilon\mid \mathcal F_t\right)
		\le \exp\bigl(-c \varepsilon^2 L_T^2\bigr).
	\end{equation}
\end{lemma}

\begin{proof}[Proof of Lemma~\ref{lem:theta-star-lower-tail}]
	On $\Omega_T$, for any $t\in\{t_0,\dots,T-1\}$ we have \(N_{\mathrm{opt}}(t)=t-N_{\mathrm{sub}}(t)\ge t-N_{\mathrm{sub}}(T)\ge t/2\).
	By averaging over $\mathcal{S}^\star$, there exists \(i\in \mathcal{S}^\star\) with \(N_{i,t}\ge t/(2|\mathcal{S}^\star|)\ge \variance L_T^2/|\mathcal{S}^\star|\).
	On $\mathcal{E}_T$, for $T$ large enough,
	we have \(\widehat{\mu}_{i,t}\ge \mu^\star-\varepsilon\).
	Since \(\theta_{i,t+1}\sim \mathcal{N}(\widehat{\mu}_{i,t},\variance/N_{i,t})\) conditionally on $\mathcal F_t$,
	\[
	\mathbb P(\max_{i\in\mathcal S^\star}\theta_{i,t+1}<\mu^\star-2\varepsilon\mid\mathcal F_t)
	\le
	\mathbb P(\theta_{i,t+1}<\mu^\star-2\varepsilon\mid\mathcal F_t)
	=
		\Phi \left(-\varepsilon\sqrt{\frac{N_{i,t}}{\variance}}\right).
	\]
	Using the standard Gaussian tail bound \(\Phi(-x)\le \exp\{-x^2/2\}\) and the preceding lower bound on \(N_{i,t}\) yields~\eqref{eq:theta-star-lower-tail}.
\end{proof}

We next bound the one-step selection probability for the fixed suboptimal arm \(a\).  If \(k=N_{a,t}\), then conditionally on \(\mathcal F_t\), \(\theta_{a,t+1}\sim \mathcal{N}(\widehat{\mu}_{a,t},\variance/k)\).

Choose fixed constants \(0<c<C<\infty\), depending only on the fixed instance and \(\varepsilon\), with \(C\) sufficiently large and \(c\) sufficiently small for the bounds below.
On \(\mathcal E_T\), \(|\widehat\mu_{a,t}-\mu_a|\le \varepsilon/2\) whenever \(N_{a,t}\ge \lceil 4C\log\log T/\varepsilon^2\rceil\).
Recall \(L_T=\log(T/\variance)\) from~\eqref{eq:LT-def}. Because we no longer restrict \(\variance\) to be sub-polynomial in \(T\), the lower bound below keeps one deterministic competitor-loss factor:
\[
r_T
 :=
 C\left\{1+\frac{\log\log T}{\variance}+\log\left(1+\frac{\log\log T}{\variance}\right)\right\}
 =
 o(L_T).
\]
Indeed, \(r_T=o(L_T)\): if \(\variance\le\sqrt{\log\log T}\), then \(L_T\sim\log T\) and \(r_T=O(\log\log T/\variance+\log\log\log T)=o(L_T)\); if \(\variance>\sqrt{\log\log T}\), then \(r_T=O(\sqrt{\log\log T})=o(L_T)\), because \(\variance(\log T)^2/T\to0\) implies \(L_T-2\log\log T\to\infty\).
Define
\begin{equation}\label{eq:alpha-pm-def}
\alpha_+(\varepsilon):=\frac{(\Delta_a-3\varepsilon)^2}{2}, \qquad
\alpha_-(\varepsilon):=\frac{(\Delta_a+3\varepsilon)^2}{2},
\end{equation}
and
\begin{equation}\label{eq:k-star-def}
k_0(T):=\max\left\{\left\lceil 4C\log\log T/\varepsilon^2\right\rceil,
\left\lceil \frac{\variance\left(\sqrt{L_T}+r_T+\log L_T\right)}{4(\Delta_a+3\varepsilon)^2}\right\rceil\right\}.
\end{equation}
Then the cutoff in~\eqref{eq:k-star-def} satisfies
$
\sqrt{L_T}+r_T+\log L_T\to\infty,
\sqrt{L_T}+r_T+\log L_T=o(L_T),
$
and hence \(k_0(T)/(\variance L_T)\to0\) and \(k_0(T)/\variance\to\infty\).
The following conclusion gives the one-step selection bounds only on the range used below.
\begin{lemma}[Exponential bounds on one-step selection]\label{lem:one-step-bounds}
    With the fixed constants \(c,C\) chosen above, on \(\mathcal{E}_T\cap \Omega_T\), for all \(t\in\{t_0,\dots,T-1\}\) satisfying \(k_0(T)\le k=N_{a,t}\le 4\variance L_T/(\Delta_a-3\varepsilon)^2\),
    \begin{equation}
        p^-_{k,T}(\varepsilon)\le \mathbb P(A_{t+1}=a\mid \mathcal F_t)\le p^+_{k,T}(\varepsilon),
        \quad \text{where }
        \left\{
        \begin{aligned}
        p^+_{k,T}(\varepsilon)
        &:=C\sqrt{\frac{\variance}{k}}\cdot
        \exp\left\{-\alpha_+(\varepsilon)\frac{k}{\variance}\right\},\\
        p^-_{k,T}(\varepsilon)
        &:=c\exp\{-r_T\}\cdot\sqrt{\frac{\variance}{k}}\cdot
        \exp\left\{-\alpha_-(\varepsilon)\frac{k}{\variance}\right\}.
        \end{aligned}
        \right.
    \end{equation}
	Here \(k_0(T)\) is the cutoff defined in~\eqref{eq:k-star-def}; in particular \(k_0(T)/\variance\to\infty\) and \(k_0(T)=o(\variance L_T)\).
\end{lemma}

\begin{proof}[Proof of Lemma~\ref{lem:one-step-bounds}]
    For the upper bound,
    \[
    \{A_{t+1}=a\}\subseteq
    \{\theta_{a,t+1}\ge \mu^\star-2\varepsilon\}\cup\{\max_{i\in\mathcal S^\star}\theta_{i,t+1}<\mu^\star-2\varepsilon\}.
    \]
    By Lemma~\ref{lem:theta-star-lower-tail}, the second event has conditional probability at most
    \(\exp\{-c\varepsilon^2L_T^2\}\). On \(\mathcal E_T\) and \(k\ge \lceil 4C\log\log T/\varepsilon^2\rceil\),
    \(\widehat\mu_{a,t}\le \mu^\star-\Delta_a+\varepsilon\), so
    \[
    \mathbb P(\theta_{a,t+1}\ge \mu^\star-2\varepsilon\mid \mathcal F_t)
    \le \mathbb P\left(Z\ge(\Delta_a-3\varepsilon)\sqrt{\frac{k}{\variance}}\right).
    \]
    Mills' upper bound in Lemma~\ref{lem:normal-tail} gives the displayed upper bound.  Since \(k\le 4\variance L_T/(\Delta_a-3\varepsilon)^2\), this displayed upper bound is at least a constant multiple of \(L_T^{-1/2}\exp\{-2L_T\}\).  The lower-tail term above is \(o(L_T^{-1/2}\exp\{-2L_T\})\).  The chosen value of \(C\) absorbs this term, proving the upper bound.

    For the lower bound, use the event that \(\theta_{a,t+1}\ge \mu^\star+2\varepsilon\) and \(\theta_{b,t+1}\le \mu^\star+\varepsilon\) for every \(b\ne a\).  On this event, arm \(a\) is selected.  On \(\mathcal E_T\), for every \(b\ne a\), since \(\mu_b\le \mu^\star\), \((\mu^\star+\varepsilon-\widehat\mu_{b,t})/\sqrt{\variance/N_{b,t}}\ge -\sqrt{C\log\log T/\variance}\).
    By conditional independence of the Gaussian indices, \(\prod_{b\ne a}\mathbb P(\theta_{b,t+1}\le \mu^\star+\varepsilon\mid \mathcal F_t)\ge \Phi(-\sqrt{C\log\log T/\variance})^{K-1}\).
    The standard lower bound for \(\Phi(-x)\), applied separately for \(x\le1\) and \(x>1\), gives
    \[
    -\log\Phi\left(-\sqrt{\frac{C\log\log T}{\variance}}\right)
    \le C\left\{1+\frac{\log\log T}{\variance}
    +\log\left(1+\frac{\log\log T}{\variance}\right)\right\}.
    \]
    By the choice of \(C\) in \(r_T\), this gives \(\prod_{b\ne a}\mathbb P(\theta_{b,t+1}\le \mu^\star+\varepsilon\mid \mathcal F_t)\ge \exp\{-r_T\}\).
    For arm \(a\), \(\widehat\mu_{a,t}\ge \mu_a-\varepsilon=\mu^\star-\Delta_a-\varepsilon\), and hence
    \[
    \mathbb P(\theta_{a,t+1}\ge \mu^\star+2\varepsilon\mid \mathcal F_t)
    \ge \mathbb P\left(Z\ge(\Delta_a+3\varepsilon)\sqrt{\frac{k}{\variance}}\right).
    \]
    Because \(k\ge k_0(T)\) implies \(k/\variance\to\infty\), Mills' lower bound yields the displayed lower bound after decreasing the fixed constant \(c\) if necessary.
\end{proof}

Once arm \(a\) has been pulled \(N_{a,t}\) times, its Gaussian sampling index variance stays at \(\variance/N_{a,t}\) until the next pull of \(a\).  Let \(T_a(k)\) be the time of the \(k\)-th pull of arm \(a\), with \(T_a(k)=\infty\) if the \(k\)-th pull never occurs and \(T_a(0)=0\).

\begin{lemma}[Geometric coupling for waiting times]\label{lem:geom-coupling}
    Recall \(k_0(T)\) from~\eqref{eq:k-star-def} and \(t_0=\lceil 2\variance L_T^2\rceil\). On an enlarged probability space, first generate independent uniform blocks \(\{U_{k,s}:s\ge1\}\), independently over all integers \(k\) satisfying \(k_0(T)\le k\le 4\variance L_T/(\Delta_a-3\varepsilon)^2\), and define
    \[
    G_k^+\sim\mathrm{Geom}(p^+_{k,T}(\varepsilon)),\qquad
    G_k^-\sim\mathrm{Geom}(p^-_{k,T}(\varepsilon)),
    \]
    from these uniforms.  Then each family is independent over \(k\).  The signs refer to the one-step probability bounds, so \(G_k^+\) is the stochastically smaller waiting time.  On \(\mathcal E_T\cap\Omega_T\), if \(T_a(k)\in[t_0,T]\), the same uniforms give the pathwise monotone coupling with the remaining horizon written explicitly:
    \begin{equation}\label{eq:geom-coupling-bound}
        G_k^+\wedge (T-T_a(k)+1)\le (T_a(k+1)-T_a(k))\wedge (T-T_a(k)+1)\le G_k^-\wedge (T-T_a(k)+1)\qquad \text{a.s.}
    \end{equation}
\end{lemma}

\begin{proof}[Proof of Lemma~\ref{lem:geom-coupling}]
    The uniforms in the statement are fixed before restricting to \(\mathcal E_T\cap\Omega_T\), and
    \(G_k^+:=\inf\{s\ge1:U_{k,s}\le p^+_{k,T}(\varepsilon)\}\) and
    \(G_k^-:=\inf\{s\ge1:U_{k,s}\le p^-_{k,T}(\varepsilon)\}\).
    Fix \(k\) with \(T_a(k)\in[t_0,T]\). If \(T_a(k)=T\), the claim is immediate. Otherwise, on the enlarged space, use the same uniform block to realize the waiting-block pull indicators. For every step \(s\) before the next pull of arm \(a\) and before the horizon, set
    \[
    \mathbf 1\{A_{T_a(k)+s}=a\}
    =
    \mathbf 1\left\{
    U_{k,s}\le
    \mathbb P(A_{T_a(k)+s}=a\mid\mathcal F_{T_a(k)+s-1})
    \right\}.
    \]
    Conditionally on \(\mathcal F_{T_a(k)+s-1}\), the right-hand side is Bernoulli with success probability \(\mathbb P(A_{T_a(k)+s}=a\mid\mathcal F_{T_a(k)+s-1})\), so this realization has the same conditional law as the original algorithm's pull indicator. Before the next pull of arm \(a\), its current count remains \(k\). Lemma~\ref{lem:one-step-bounds} then gives
    \[
    p^-_{k,T}(\varepsilon)\le
    \mathbb P(A_{T_a(k)+s}=a\mid\mathcal F_{T_a(k)+s-1})
    \le p^+_{k,T}(\varepsilon).
    \]
    Therefore, on such steps,
    \[
    \mathbf 1\{U_{k,s}\le p^-_{k,T}(\varepsilon)\}
    \le \mathbf 1\{A_{T_a(k)+s}=a\}
    \le \mathbf 1\{U_{k,s}\le p^+_{k,T}(\varepsilon)\}.
    \]
    Taking first success times in this pathwise comparison, with truncation at \(T-T_a(k)+1\) if no success occurs before the horizon, yields~\eqref{eq:geom-coupling-bound}.
\end{proof}

\subsubsection{Step 2. Finishing the proof for a fixed suboptimal arm.}
Recall \(\alpha_\pm(\varepsilon)\) from~\eqref{eq:alpha-pm-def} and \(k_0(T)\) from~\eqref{eq:k-star-def}.
By Lemma~\ref{lem:one-step-bounds}, \(p^+_{k,T}(\varepsilon)\) and \(p^-_{k,T}(\varepsilon)\) have exponential rates \(\alpha_+(\varepsilon)\) and \(\alpha_-(\varepsilon)\), respectively.  The extra factor \(\exp\{-r_T\}\) in \(p^-_{k,T}\) is covered by the tail log-law below, because \(r_T=o(L_T)\).
Fix \(\delta\in(0,1)\) and define
\begin{equation}
n^-(T):=\left\lfloor \frac{(1-\delta)\variance L_T}{\alpha_-(\varepsilon)}\right\rfloor,
\qquad
n^+(T):=\left\lceil \frac{(1+\delta)\variance L_T}{\alpha_+(\varepsilon)}\right\rceil.
\end{equation}
Since \(k_0(T)=o(\variance L_T)\) by~\eqref{eq:k-star-def}, both counts are larger than \(k_0(T)\) for large \(T\). Also, because \(\delta<1\), \(n^+(T)\le 4\variance L_T/(\Delta_a-3\varepsilon)^2\) for all large \(T\).
For any \(n\), \(\{N_{a,T}\ge n\}\Longleftrightarrow\{T_a(n)\le T\}\).

\begin{lemma}[Reaching \(k_0(T)\) pulls takes negligible time]\label{lem:Ta-kstar-negligible}
Recalling~\eqref{eq:k-star-def},
\begin{equation}\label{eq:Tkstar-small}
    \frac{T_a(k_0(T))}{T}\xrightarrow{\mathbb P}0.
\end{equation}
\end{lemma}

\begin{proof}[Proof of Lemma~\ref{lem:Ta-kstar-negligible}]
Since \(t_0/T\to0\), it suffices to control the additional time after \(t_0\).  On \(\mathcal E_T\cap\Omega_T\), for all large \(T\) and all \(t\in\{t_0,\dots,T-1\}\) with \(N_{a,t}\le k_0(T)\), \(\mathbb P(A_{t+1}=a\mid\mathcal F_t)\ge p^-_{k_0(T),T}(\varepsilon)\).
Indeed, if \(N_{a,t}=k_0(T)\), this follows from Lemma~\ref{lem:one-step-bounds}.  If \(N_{a,t}=k<k_0(T)\), then on \(\mathcal E_T\),
\[
\frac{\mu^\star+2\varepsilon-\widehat\mu_{a,t}}{\sqrt{\variance/k}}
\le
(\Delta_a+2\varepsilon)\sqrt{\frac{k}{\variance}}
+
\sqrt{\frac{C\log\log T}{\variance}}
\le
(\Delta_a+3\varepsilon)\sqrt{\frac{k_0(T)}{\variance}},
\]
The step to \((\Delta_a+3\varepsilon)\sqrt{k_0(T)/\variance}\) uses \(k_0(T)\ge \lceil 4C\log\log T/\varepsilon^2\rceil\).  The same competitor bound as in Lemma~\ref{lem:one-step-bounds}, together with Mills' lower bound at \((\Delta_a+3\varepsilon)\sqrt{k_0(T)/\variance}\), gives \(\mathbb P(A_{t+1}=a\mid\mathcal F_t)\ge p^-_{k_0(T),T}(\varepsilon)\).

Fix any \(\eta>0\).  Up to time \(\lfloor \eta T\rfloor\), the lower bound \(\mathbb P(A_{t+1}=a\mid\mathcal F_t)\ge p^-_{k_0(T),T}(\varepsilon)\) allows the standard uniform coupling with the sum of \(k_0(T)\) independent geometric variables, each with success probability \(p^-_{k_0(T),T}(\varepsilon)\).  If the coupling sum plus \(t_0\) is at most \(\lfloor \eta T\rfloor\), then arm \(a\) reaches \(k_0(T)\) pulls by time \(\lfloor \eta T\rfloor\).  Therefore
\[
\mathbb P\left(T_a(k_0(T))>\eta T\right)
\le \mathbb P(\mathcal E_T^c\cup\Omega_T^c)
+\mathbb P\left(\text{the coupling sum exceeds }\eta T-t_0\right),
\]
Here the coupling sum denotes a sum of \(k_0(T)\) independent \(\mathrm{Geom}(p^-_{k_0(T),T}(\varepsilon))\) variables; using \(k_0(T)\) variables instead of the remaining number of pulls only makes the upper bound larger.
Moreover, by~\eqref{eq:k-star-def}, \(\alpha_-(\varepsilon)k_0(T)/\variance=O(\sqrt{L_T}+r_T+\log L_T)=o(L_T)\) and \(r_T=o(L_T)\).
The coupling sum of \(k_0(T)\) geometric variables with success probability \(p^-_{k_0(T),T}(\varepsilon)\) has expectation \(k_0(T)/p^-_{k_0(T),T}(\varepsilon)\), and
\[
\frac{k_0(T)}{p^-_{k_0(T),T}(\varepsilon)}
\le
	 C k_0(T)\sqrt{\frac{k_0(T)}{\variance}}\cdot
	 \exp\left\{r_T+\alpha_-(\varepsilon)\frac{k_0(T)}{\variance}\right\}
 =
 \variance\exp\{o(L_T)\}
=o(T),
\]
where the last step uses \(T=\variance\exp\{L_T\}\).  Markov's inequality gives~\eqref{eq:Tkstar-small}.
\end{proof}

We now prove the lower and upper brackets for \(N_{a,T}\).  Lemma~\ref{lem:geom-tail-sum-loglaw}, applied to \(\sum_{k=k_0(T)}^{n^-(T)-1}G_k^-\) with lower index \(m_T=k_0(T)\), terminal index \(n_T=n^-(T)-1\), exponent \(\alpha_-(\varepsilon)\), and the above \(r_T\), yields
\begin{equation}\label{eq:lower-bracket-geometric-loglaw}
\log\left(\sum_{k=k_0(T)}^{n^-(T)-1}G_k^-\right)
=\log\variance+\alpha_-(\varepsilon)\frac{n^-}{\variance}+o_{\mathbb P}(L_T).
\end{equation}

By~\eqref{eq:lower-bracket-geometric-loglaw} and the definition of \(n^-(T)\),
\(T^{-1}\sum_{k=k_0(T)}^{n^-(T)-1}G_k^-\to0\) in probability.
If some counts between \(k_0(T)\) and \(n^-\) have already been passed before \(t_0\), the corresponding geometric variables in~\eqref{eq:lower-bracket-geometric-loglaw} only add nonnegative slack.
\begin{equation}\label{eq:lower-bracket-time-slack}
\max\{T_a(k_0(T)),t_0\}+\sum_{k=k_0(T)}^{n^-(T)-1}G_k^-=o_{\mathbb P}(T).
\end{equation}
Together with Lemma~\ref{lem:geom-coupling},~\eqref{eq:lower-bracket-time-slack} makes the truncation inactive with probability tending to one along all waiting times used for the lower bracket, and gives
\begin{equation}\label{eq:lower-bracket-time-bound}
\frac{T_a(n^-)}{T}
\le
\frac{\max\{T_a(k_0(T)),t_0\}}{T}+\frac{1}{T}\sum_{k=k_0(T)}^{n^-(T)-1}G_k^-
\xrightarrow{\mathbb P}0.
\end{equation}
Since \(\{N_{a,T}\ge n^-(T)\}=\{T_a(n^-)\le T\}\),~\eqref{eq:lower-bracket-time-bound} implies
\begin{equation}\label{eq:lower-bracket}
\mathbb P(N_{a,T}\ge n^-(T))\to1.
\end{equation}
For the upper bracket, recall \(k_0(T)\) from~\eqref{eq:k-star-def} and set
\[
k_1^+(T):=\max\{k_0(T),\lceil \variance(\sqrt{L_T}+r_T+\log L_T)\rceil\}.
\]
The time-uniform version of Lemma~\ref{lem:lem3}, applied with \(u=t_0\), gives \(\mathbb{E}[N_{a,t_0}]=O(\variance\log L_T)\).  Since \(k_1^+(T)\ge \variance(\sqrt{L_T}+r_T+\log L_T)\) and \(\sqrt{L_T}+r_T+\log L_T\ge\sqrt{L_T}\), Markov's inequality gives
\begin{equation}\label{eq:upper-bracket-start-count}
\mathbb P\left(
N_{a,t_0}<
k_1^+(T)
\right)\to1.
\end{equation}
On the event
\(N_{a,t_0}<
k_1^+(T)\),
if \(T_a(n^+)\le T\), then all waiting times from this starting count to count \(n^+-1\) occur after \(t_0\) and before the horizon, so the truncated coupling in Lemma~\ref{lem:geom-coupling} reduces to the ordinary one.
\begin{equation}\label{eq:upper-bracket-time-lower}
\begin{aligned}
&\left\{
N_{a,t_0}<
k_1^+(T),
\ T_a(n^+)\le T
\right\}\subseteq
\left\{
T\ge
\sum_{k=k_1^+(T)}^{n^+(T)-1}G_k^+
\right\}.
\end{aligned}
\end{equation}
Since \(k_1^+(T)=o(n^+(T))\), Lemma~\ref{lem:geom-tail-sum-loglaw} applied to the sum in~\eqref{eq:upper-bracket-time-lower} with lower index \(m_T=k_1^+(T)\), terminal index \(n_T=n^+(T)-1\), exponent \(\alpha_+(\varepsilon)\), and the lemma's deterministic term set to \(0\), yields
\begin{equation}\label{eq:upper-bracket-geometric-loglaw}
\log\left(\sum_{k=k_1^+(T)}^{n^+(T)-1}G_k^+\right)
=\log\variance+\alpha_+(\varepsilon)\frac{n^+}{\variance}+o_{\mathbb P}(L_T).
\end{equation}
Since \(T=\variance\exp\{L_T\}\),~\eqref{eq:upper-bracket-geometric-loglaw} and the definition of \(n^+(T)\) imply
\begin{equation}\label{eq:upper-bracket-geometric-diverges}
\frac{1}{T}
\sum_{k=k_1^+(T)}^{n^+(T)-1}G_k^+
\xrightarrow{\mathbb P}\infty.
\end{equation}
Combining~\eqref{eq:upper-bracket-start-count},~\eqref{eq:upper-bracket-time-lower}, and~\eqref{eq:upper-bracket-geometric-diverges} gives
\begin{equation}\label{eq:upper-bracket}
\mathbb P(N_{a,T}<n^+(T))\to1.
\end{equation}
Combining~\eqref{eq:lower-bracket} and~\eqref{eq:upper-bracket}, with probability tending one,
\[
\frac{1-\delta}{\alpha_-(\varepsilon)}
\le
\frac{N_{a,T}}{\variance L_T}
\le
\frac{1+\delta}{\alpha_+(\varepsilon)}.
\]
Letting \(\delta\downarrow0\) and then \(\varepsilon\downarrow0\) gives
\[
\frac{N_{a,T}}{\variance L_T}\xrightarrow{\mathbb P}\frac{2}{\Delta_a^2}.
\]
This proves Theorem~\ref{thm:main}(b) for the fixed suboptimal arm \(a\). Since \(K\) is fixed, the same conclusion may be made simultaneous over all suboptimal arms by a finite union bound.

\section{Proofs for Mean-Bonus Thompson Sampling} \label{appendix:thm:main:2}

We work on the good event
\begin{equation}
	\mathcal{G}_T := \mathcal{G}_{T, 1} \cap \mathcal{G}_{T, 2},
	\label{eq:mean-bonus-good-event}
\end{equation}
where $\mathcal{G}_{T, 1}$ and $\mathcal{G}_{T, 2}$ are defined as follows:
\[
\begin{aligned}
\mathcal{G}_{T, 1} &:=  \left\{
\left|\widehat{\mu}_{a, t}-\mu_a\right| \leq \sqrt{\frac{C_1  \log T}{N_{a, t}}}, \text{ for all }t\leq T \text{ with } N_{a,t}\ge1 \text{ and all arms } a \in [K] \right\}, \\
\mathcal{G}_{T, 2} &:=  \left\{
\left|Z_{a,t}\right| \leq \sqrt{2\log(2KT^3)}, \text{ for all }t\leq T \text{ and all arms } a \in [K]
\right\},
\end{aligned}
\]
where $C_1$ is a sufficiently large constant.  The Gaussian variables
\(\{Z_{a,t}:a\in[K],1\le t\le T\}\) are mutually independent standard normals,
independent of all other random variables; equivalently, conditional on
the past history, the indices use fresh algorithmic randomization.
\begin{lemma} \label{lem:mean-bonus-good-event}
	It holds that $\mathbb{P}(\mathcal{G}_T)\ge 1-2T^{-2}$.
\end{lemma}
\begin{proof}[Proof of Lemma~\ref{lem:mean-bonus-good-event}]
	For each fixed arm \(a\), recall that
	$
	\widehat\mu_{a,t}-\mu_a
	=
	({1}/{N_{a,t}})\sum_{s\le t:\,A_s=a}\xi_{a,s},
	  $ for $ 1\le N_{a,t}\le T.
	$
	Applying Lemma~\ref{lem:TUB} and taking a union bound over
	\(a\in[K]\), with \(C_1\) sufficiently large,
	\begin{equation} \label{eq:event-1}
	\mathbb{P}(\mathcal{G}_{T, 1}^c)
	\le
	\sum_{a=1}^K
	\mathbb P\left\{\exists\,t\le T:\ N_{a,t}\ge1,\
	\left|\frac{1}{N_{a,t}}\sum_{s\le t:\,A_s=a}\xi_{a,s}\right|>\sqrt{\frac{C_1\log T}{N_{a,t}}}\right\}
	\le T^{-2}.
	\end{equation}
	Meanwhile, for any standard normal random variable \(Z\), Lemma~\ref{lem:normal-tail}-(\ref{lem:normal-tail-4}) gives \(\mathbb P(|Z|\ge x)\le2\exp(-x^2/2)\).
	This further implies that
	\begin{equation}\label{eq:event-2}
	\mathbb{P}(\mathcal{G}_{T, 2}^c) \le 2KT \cdot \exp\{-\log(2KT^3)\}=T^{-2}.
	\end{equation}
	Combining~\eqref{eq:event-1} and~\eqref{eq:event-2}, we have \(\mathbb P(\mathcal G_T)\ge1-2T^{-2}\), which completes the proof.
\end{proof}

Recall that our algorithm uses the following update rule:
\begin{equation}
	\Upsilon_{a,t+1}
	:= \widehat{\mu}_{a,t} + \frac{Z_{a,t+1}}{\sqrt{N_{a,t}}}
	+ \sqrt{\frac{2\bonus\log T}{N_{a,t}}} \qquad A_{t+1} = \argmax_{a\in [K]} \Upsilon_{a,t+1}.
	\label{eq:mb-index}
\end{equation}
Equivalently, \(\Upsilon_{a,t+1}\) is Gaussian with center \(\widehat\mu_{a,t}+B_{a,t}\) and variance \(1/N_{a,t}\); no further bonus is added after forming \(\Upsilon_{a,t+1}\).
Under the good event~\eqref{eq:mean-bonus-good-event}, we have
	\begin{equation}\label{eq:U:bounds}
	\Upsilon_{a,t+1}\in
\left[
\mu_a+\sqrt{\frac{2\bonus\log T}{N_{a,t}}}
\ \pm\
	\frac{\sqrt{2\log(2KT^3)}+\sqrt{C_1\log T}}{\sqrt{N_{a,t}}}
	\right].
	\end{equation}
	Our proof will frequently use the following two lemmas and the fact that as $T\to\infty$,
	\begin{equation}\label{eq:gamma:constraint}
	\bonus \to \infty \qquad \text{and} \qquad \bonus \log T = o(T).
	\end{equation}
Set
\[
\varepsilon_T:=\bonus^{-1/4}+\left(\frac{\bonus\log T}{T}\right)^{1/4}.
\]
Then, for all sufficiently large \(T\), \(\varepsilon_T\in(0,1/2)\), \(\varepsilon_T\to0\),
\[
\bonus^{-1/2}=o(\varepsilon_T),\qquad
\sqrt{\frac{\bonus\log T}{T}}=o(\varepsilon_T),\qquad
\sqrt{\frac{\log T}{T}}=o(\varepsilon_T),\qquad
\varepsilon_T\sqrt{\bonus}\to\infty.
\]

\begin{lemma}
	\label{lem:subopt-negligible}
	Fix any suboptimal arm $a\notin \mathcal{S}^\star$ with gap $\Delta_a>0$. We define
	\begin{equation}\label{eq:Nstar-def}
	N_{a, T}^\star:=\frac{2\bonus\log T}{\Delta_a^2},
	\qquad
	N_{a, T}^+:=\left\lceil (1+\varepsilon_T) N_{a, T}^\star\right\rceil \qquad N_{a,T}^-:=\left\lfloor (1-\varepsilon_T) N_{a, T}^\star\right\rfloor.
	\end{equation}
	On \(\mathcal G_T\), for all sufficiently large \(T\), we have $N_{a,T} \le N_{a,T}^+$, which further implies that
	\begin{equation}
		\frac{N_{\mathrm{sub}}(T)}{T}\xrightarrow{\mathbb{P}}0,
		\qquad
		\frac{N_{\mathrm{opt}}(T)}{T}\xrightarrow{\mathbb{P}}1,
		\label{eq:subopt-negligible}
	\end{equation}
	where $N_{\mathrm{opt}}(T)=\sum_{i\in \mathcal{S}^\star} N_{i,T}$ and $N_{\mathrm{sub}}(T)=T-N_{\mathrm{opt}}(T)$. The probability statements follow from \(\mathbb P(\mathcal G_T)\to1\).
\end{lemma}

\begin{proof}[Proof of Lemma~\ref{lem:subopt-negligible}]
	Fix a suboptimal arm $a\notin \mathcal{S}^\star$ with gap $\Delta_a>0$ and large enough $T$. Suppose at some time $t\le T-1$ we have $N_{a,t}\ge N_{a, T}^+$. Under $\mathcal{G}_T$ defined in~\eqref{eq:mean-bonus-good-event}, the upper bound in~\eqref{eq:U:bounds} yields
	\[
	\Upsilon_{a,t+1}\le \mu_a+\sqrt{\frac{2\bonus\log T}{N_{a,t}}}+\frac{\sqrt{2\log(2KT^3)}+\sqrt{C_1\log T}}{\sqrt{N_{a,t}}}.
	\]
		Since $N_{a,t}\ge N_{a, T}^+$, we have
		\begin{equation}\label{eq:subopt-bonus-upper-scale}
		\sqrt{\frac{2\bonus\log T}{N_{a,t}}}
		\le \sqrt{\frac{2\bonus\log T}{N_{a, T}^+}}
		\le
		\sqrt{
		\frac{2\bonus\log T}
		{(1+\varepsilon_T)N_{a,T}^\star}
		}
		=
		\frac{\Delta_a}{\sqrt{1+\varepsilon_T}} \le
		\Delta_a\left(1-\frac{\varepsilon_T}{4}\right),
		\end{equation}
		The inequality
	\((1+\varepsilon_T)^{-1/2}\le1-\varepsilon_T/4\)
	used in~\eqref{eq:subopt-bonus-upper-scale} follows from the first inequality in~\eqref{eq:sqrt-ineq} in Lemma~\ref{lem:sqrt-ineq}.
	Combining~\eqref{eq:subopt-bonus-upper-scale} with \(\mu_a=\mu^\star-\Delta_a\) gives
	\begin{equation}\label{eq:subopt-upper-mid}
	\mu_a+\sqrt{\frac{2\bonus\log T}{N_{a,t}}}
	\le
	(\mu^\star-\Delta_a)+\Delta_a\left(1-\frac{\varepsilon_T}{4}\right)
	=\mu^\star-\frac{\Delta_a\varepsilon_T}{4}.
	\end{equation}
	Since \(N_{a,t}\ge N_{a,T}^+ \ge N_{a, T}^\star\),
	\begin{equation}\label{eq:subopt-error}
	\begin{aligned}
	\frac{\sqrt{2\log(2KT^3)}+\sqrt{C_1\log T}}{\sqrt{N_{a,t}}}
	& \le
	\frac{\sqrt{2\log(2KT^3)}+\sqrt{C_1\log T}}{\sqrt{N_{a, T}^\star}} \\
	& =
	\frac{(\sqrt{2\log(2KT^3)}+\sqrt{C_1\log T}) \cdot \Delta_a}{\sqrt{2\bonus\log T}},
	\end{aligned}
	\end{equation}
	where the last equality uses the definition of $N_{a,T}^\star$ in~\eqref{eq:Nstar-def}.
	Because $\sqrt{2\log(2KT^3)}+\sqrt{C_1\log T}=O(\sqrt{\log T})$ and $\bonus\to\infty$, the right-hand side of~\eqref{eq:subopt-error} is $O(\Delta_a/\sqrt{\bonus})$,
	which is $o(\Delta_a\varepsilon_T)$ since \(\bonus^{-1/2}=o(\varepsilon_T)\).
	Thus there exists large enough $T$ such that \((\sqrt{2\log(2KT^3)}+\sqrt{C_1\log T})/\sqrt{N_{a,t}}\le \Delta_a\varepsilon_T/8\).
	Combining~\eqref{eq:subopt-upper-mid} and~\eqref{eq:subopt-error} gives, for all $T$ large enough and all such $t$, \(\Upsilon_{a,t+1}\le \mu^\star-\Delta_a\varepsilon_T/8<\mu^\star\).
	By~\eqref{eq:U:bounds}, for all $T$ large enough on $\mathcal{G}_T$ every optimal arm $i\in \mathcal{S}^\star$ satisfies
	\begin{equation}\label{eq:exceed:mu-star}
	\Upsilon_{i,t+1}\ge \mu^\star+ \frac{\sqrt{2\bonus\log T} - \sqrt{2\log(2KT^3)} - \sqrt{C_1\log T}}{\sqrt{N_{i,t}}} \ge \mu^\star.
	\end{equation}
	Hence the algorithm cannot select arm $a$ at time $t+1$.
	Therefore, once $N_{a,t}$ reaches $N_{a, T}^+$ it cannot increase further. Summing the bound over all suboptimal arms gives, on \(\mathcal G_T\) defined in~\eqref{eq:mean-bonus-good-event},
		\begin{equation}
			N_{\mathrm{sub}}(T):=\sum_{a\notin \mathcal{S}^\star} N_{a,T}
			\le \sum_{a\notin \mathcal{S}^\star} N_{a, T}^+
			=O(\bonus\log T),
			\label{eq:Nsub-bound}
		\end{equation}
		Since \(\mathbb P(\mathcal G_T)\to1\), this further implies that
	\[
	\frac{N_{\mathrm{sub}}(T)}{T} \le \frac{O(\bonus\log T)}{T} \xrightarrow{\mathbb{P}}0, \qquad \frac{N_{\mathrm{opt}}(T)}{T} = 1-\frac{N_{\mathrm{sub}}(T)}{T} \xrightarrow{\mathbb{P}}1.
	\]
	which completes the proof.
\end{proof}

\subsection{Proof of the Mean-Bonus Part of Theorem~\ref{thm:regret-main}}
\label{sec:regret-mean-bonus}

\begin{proof}[Proof of the mean-bonus part of Theorem~\ref{thm:regret-main}]
If all arms are optimal, then \(\mathcal R(T)=0\), so the claim is trivial.  Hence assume that \([K]\setminus\mathcal S^\star\) is nonempty.
First bound the pathwise pseudo-regret on \(\mathcal G_T\).  Take \(T\) large enough so that Lemma~\ref{lem:subopt-negligible} holds for every suboptimal arm, \(\varepsilon_T\le1/2\), and \(N_{a,T}^\star\ge2\) for every suboptimal arm.  On \(\mathcal G_T\), for every \(a\notin\mathcal S^\star\), we have
\[
N_{a,T}
\le
N_{a,T}^+
\le
2N_{a,T}^\star
=
\frac{4\bonus\log T}{\Delta_a^2}.
\]
Therefore, we further obtain
\[
\begin{aligned}
\sum_{t=1}^T(\mu^\star-\mu_{A_t})
&=
\sum_{a\notin\mathcal S^\star}\Delta_a N_{a,T}  \le
4\sum_{a\notin\mathcal S^\star}\frac{\bonus\log T}{\Delta_a}.
\end{aligned}
\]
Since the reward noises have mean zero, \(\mathcal R(T)=\mathbb{E}[\sum_{t=1}^T(\mu^\star-\mu_{A_t})]\). On \(\mathcal G_T^c\), the pathwise pseudo-regret is at most \(T\max_{b\in[K]}(\mu^\star-\mu_b)\). Together with Lemma~\ref{lem:mean-bonus-good-event}, for sufficiently large \(T\), we have
\[
\mathcal R(T)
\le
4\sum_{a\notin\mathcal S^\star}\frac{\bonus\log T}{\Delta_a}
+
T\max_{b\in[K]}(\mu^\star-\mu_b)\mathbb P(\mathcal G_T^c)
\le
6\sum_{a\notin\mathcal S^\star}\frac{\bonus\log T}{\Delta_a}
.
\]
This completes the proof of the mean-bonus part of Theorem~\ref{thm:regret-main}.
\end{proof}

\subsection{Proof of Theorem~\ref{thm:main:2}(a)}

\begin{proof}[Proof of Theorem~\ref{thm:main:2}(a)]
	Our proof relies on the following lemma.
	\begin{lemma}
		\label{lem:opt-order}
		For $T$ large enough, on $\mathcal{G}_T$ the following holds:
		for any two distinct optimal arms $i,j\in \mathcal{S}^\star$ and any time $t\le T-1$, if $N_{i,t}\ge (1+\varepsilon_T)N_{j,t}$, then
		\begin{equation}\label{eq:opt-order}
		\Upsilon_{i,t+1}<\Upsilon_{j,t+1}.
		\end{equation}
	\end{lemma}

	\begin{proof}[Proof of Lemma~\ref{lem:opt-order}]
		Fix distinct $i,j\in \mathcal{S}^\star$, so $\mu_i=\mu_j=\mu^\star$.
		On $\mathcal{G}_T$, use the lower bound in~\eqref{eq:U:bounds} for $\Upsilon_{j,t+1}$ and the upper bound in~\eqref{eq:U:bounds} for $\Upsilon_{i,t+1}$:
		\begin{align*}
			& \Upsilon_{j,t+1}-\Upsilon_{i,t+1} \\
			&\ge
			\sqrt{2\bonus\log T}\cdot\left(\frac{1}{\sqrt{N_{j,t}}}-\frac{1}{\sqrt{N_{i,t}}}\right)
			-\left(\sqrt{2\log(2KT^3)}+\sqrt{C_1\log T}\right)\cdot\left(\frac{1}{\sqrt{N_{j,t}}}+\frac{1}{\sqrt{N_{i,t}}}\right).
		\end{align*}
		If $N_{i,t}\ge (1+\varepsilon_T)N_{j,t}$, then $1/\sqrt{N_{i,t}}\le (1/\sqrt{1+\varepsilon_T}) \cdot(1/\sqrt{N_{j,t}})$, and by the first inequality in~\eqref{eq:sqrt-ineq} in Lemma~\ref{lem:sqrt-ineq},
		\begin{equation}\label{eq:opt-order-count-gap}
		\frac{1}{\sqrt{N_{j,t}}}-\frac{1}{\sqrt{N_{i,t}}}
		\ge
		\left(1-\frac{1}{\sqrt{1+\varepsilon_T}}\right)\frac{1}{\sqrt{N_{j,t}}}
		\ge
		\frac{\varepsilon_T}{4\sqrt{N_{j,t}}}.
		\end{equation}
		The assumption \(N_{i,t}\ge (1+\varepsilon_T)N_{j,t}\) also gives \(1/\sqrt{N_{j,t}}+1/\sqrt{N_{i,t}}\le 2/\sqrt{N_{j,t}}\).
		Together with~\eqref{eq:opt-order-count-gap}, the inequality \(1/\sqrt{N_{j,t}}+1/\sqrt{N_{i,t}}\le 2/\sqrt{N_{j,t}}\) gives
		\begin{equation}\label{eq:opt-order-gap}
	\Upsilon_{j,t+1}-\Upsilon_{i,t+1}
		\ge
			\frac{1}{\sqrt{N_{j,t}}}\cdot
			\left\{\frac{\varepsilon_T\sqrt{2\bonus\log T}}{4}
		-2\left(\sqrt{2\log(2KT^3)}+\sqrt{C_1\log T}\right)\right\}.
		\end{equation}
			 Note that \(\varepsilon_T\sqrt{\bonus}\ge \bonus^{1/4}\to\infty\) and \(\sqrt{2\log(2KT^3)}+\sqrt{C_1\log T}=O(\sqrt{\log T})\), the right hand side in~\eqref{eq:opt-order-gap} is positive for all sufficiently large \(T\); this proves Lemma~\ref{lem:opt-order}.
		\end{proof}

	Back to the proof of Theorem~\ref{thm:main:2}(a). If \(m=1\), then the claim follows directly from Lemma~\ref{lem:subopt-negligible}. Hence assume \(m\ge2\). Assume $T$ is large so that Lemma~\ref{lem:opt-order} holds.
	Let $i_{\max}\in \mathcal{S}^\star$ and $i_{\min}\in \mathcal{S}^\star$ satisfy
	\[
	N_{i_{\max},T}=\max_{i\in \mathcal{S}^\star} N_{i,T},
	\qquad
	N_{i_{\min},T}=\min_{i\in \mathcal{S}^\star} N_{i,T}.
	\]
		Suppose, toward a contradiction, that
		\[
		N_{i_{\max},T}> (1+\varepsilon_T)N_{i_{\min},T}+1.
		\]
		Let $t$ be the last time arm $i_{\max}$ was pulled. Then at time $t-1$,
		\[
		N_{i_{\max},t-1}=N_{i_{\max},T}-1
		>
		(1+\varepsilon_T)N_{i_{\min},T}
		\ge
		(1+\varepsilon_T)N_{i_{\min},t-1},
		\]
		since $N_{i_{\min},t-1}\le N_{i_{\min},T}$.
		By Lemma~\ref{lem:opt-order} we have $\Upsilon_{i_{\max},t}<\Upsilon_{i_{\min},t}$ on $\mathcal{G}_T$.
		But if the algorithm pulled $i_{\max}$ at time $t$, then $\Upsilon_{i_{\max},t} \ge \Upsilon_{i_{\min},t}$. This is a contradiction.
			Therefore,
			\begin{equation} \label{eq:nmax-nmin}
				N_{i_{\max},T}\le (1+\varepsilon_T)N_{i_{\min},T}+1.
			\end{equation}
		By the definition of \(N_{\mathrm{opt}}(T)\),
		\begin{equation}
			m N_{i_{\min},T}\le N_{\mathrm{opt}}(T)\le m N_{i_{\max},T},
			\label{eq:Nstar-bounds}
		\end{equation}
		where $m = |\mathcal{S}^\star|$ is the number of optimal arms.
		Combining~\eqref{eq:nmax-nmin} and~\eqref{eq:Nstar-bounds}, together with \(\varepsilon_T\to 0\) and \(1/T\to 0\), gives
		\begin{equation}
			\max_{i\in \mathcal{S}^\star}\left|\frac{N_{i,T}}{T}-\frac{N_{\mathrm{opt}}(T)}{mT}\right|\to 0.
		\label{eq:opt-deviation}
	\end{equation}
	If there is no suboptimal arm, then \(N_{\mathrm{opt}}(T)=T\). Otherwise, by Lemma~\ref{lem:subopt-negligible}, $N_{\mathrm{opt}}(T)/T\xrightarrow{\mathbb{P}}1$.
	Combining with~\eqref{eq:opt-deviation} and $\mathbb{P}(\mathcal{G}_T)\to 1$ yields $N_{i,T}/T\xrightarrow{\mathbb{P}}1/m$. Hence, we finish the proof of Theorem~\ref{thm:main:2}(a).
\end{proof}

\subsection{Proof of Theorem~\ref{thm:main:2}(b)}

\begin{proof}[Proof of Theorem~\ref{thm:main:2}(b)]
	The proof below uses the deterministic terminal-balance bound
	\eqref{eq:nmax-nmin} established in the proof of part~(a); this is a
	pathwise consequence on \(\mathcal G_T\), not merely the convergence
	statement of part~(a).
	We first prove the following lemma, which states that $N_{a,T}$ is lower bounded by $N_{a,T}^-$ defined in~\eqref{eq:Nstar-def}.
	\begin{lemma}
		\label{lem:subopt-lower}
		For the fixed suboptimal arm $a\notin \mathcal{S}^\star$, for all sufficiently large $T$, on $\mathcal{G}_T$ we have
		\begin{equation}
			N_{a,T}\ge N_{a,T}^-.
			\label{eq:subopt-lower}
		\end{equation}
	\end{lemma}

	\begin{proof}[Proof of Lemma~\ref{lem:subopt-lower}]
		Fix $T$ large enough and work on $\mathcal{G}_T$.
		By Lemma~\ref{lem:subopt-negligible}, on \(\mathcal G_T\) we have
		\(N_{\mathrm{sub}}(T)\le\sum_{b\notin\mathcal S^\star}N^+_{b,T}=O(\bonus\log T)=o(T)\). Hence, for all sufficiently large \(T\),
		\begin{equation}
			N_{\mathrm{sub}}(T)< \lfloor T/(8m)\rfloor.
			\label{eq:Nsub-less-ell}
		\end{equation}
		Assume for contradiction that $N_{a,T}\le N_{a,T}^-$. Then for every $t\le T$ we also have
		\begin{equation}\label{eq:Nsub-lower:contradiction}
		N_{a,t} \le N_{a, T} \le N_{a,T}^-.
		\end{equation}
		Consider any time $t\in\{T-\lfloor T/(8m)\rfloor,\dots,T-1\}$.
		We will show that on $\mathcal{G}_T$ and for all large $T$,
		\begin{equation}
			\Upsilon_{a,t+1}>\max_{i\in \mathcal{S}^\star} \Upsilon_{i,t+1}.
			\label{eq:subopt-dominates-lastblock}
		\end{equation}
			\paragraph{Upper bound for \(\max_{i\in \mathcal{S}^\star} \Upsilon_{i,t+1}\).} By~\eqref{eq:Nsub-less-ell}, \(N_{\mathrm{opt}}(T)=T-N_{\mathrm{sub}}(T)\ge T/2\) for sufficiently large \(T\). This starts the upper bound used in~\eqref{eq:subopt-dominates-lastblock}.
		If \(m=1\), then \(\min_{i\in\mathcal S^\star}N_{i,T}=N_{\mathrm{opt}}(T)\ge T/2\). If \(m\ge2\), let \(i_{\max}\) and \(i_{\min}\) denote optimal arms with maximal and minimal terminal counts, respectively; by~\eqref{eq:nmax-nmin}, optimal counts differ by at most a factor \(1+\varepsilon_T\) up to \(+1\). In either case, \(\min_{i\in \mathcal S^\star}N_{i,T}\ge T/(4m)\) for all sufficiently large \(T\).
		Now for any $t\in\{T-\lfloor T/(8m)\rfloor,\dots,T-1\}$ and any optimal $i\in \mathcal{S}^\star$, \(N_{i,t}\ge N_{i,T}-(T-t)\ge T/(4m)-\lfloor T/(8m)\rfloor\ge T/(8m)\).
		Apply the upper bound in~\eqref{eq:U:bounds} to $\Upsilon_{i,t+1}$:
		\[
		\Upsilon_{i,t+1}\le \mu^\star+\sqrt{\frac{2\bonus\log T}{N_{i,t}}}
		+\frac{\sqrt{2\log(2KT^3)}+\sqrt{C_1\log T}}{\sqrt{N_{i,t}}}.
		\]
		Combining the above with this count lower bound gives
		\[
		\begin{aligned}
		\Upsilon_{i,t+1} & \le \mu^\star + \sqrt{\frac{16 m \bonus  \log T}{T}}
		+ \frac{\sqrt{16m\log(2KT^3)}+\sqrt{16C_1 m \log T}}{\sqrt{T}} \\
		& < \mu^\star + \frac{\Delta_a\varepsilon_T}{4},
		\label{eq:opt-index-upper}
		\end{aligned}
		\]
		where we use \(\sqrt{\bonus\log T/T}=o(\varepsilon_T)\) and \(\sqrt{\log T/T}=o(\varepsilon_T)\).

		\paragraph{Lower bound $\Upsilon_{a,t+1}$ in~\eqref{eq:subopt-dominates-lastblock}.}
		On $\mathcal{G}_T$,~\eqref{eq:U:bounds} gives
		\begin{equation}\label{eq:Ua:lower:1}
		\begin{aligned}
		\Upsilon_{a,t+1} & \ge \mu_a+\sqrt{\frac{2\bonus\log T}{N_{a,t}}}-\frac{\sqrt{2\log(2KT^3)}+\sqrt{C_1\log T}}{\sqrt{N_{a,t}}} \\
		& \ge \mu_a+\sqrt{\frac{2\bonus\log T}{N_{a,T}^-}}-\frac{\sqrt{2\log(2KT^3)}+\sqrt{C_1\log T}}{\sqrt{N_{a,T}^-}},
		\end{aligned}
		\end{equation}
		The bound \(N_{a,t}\le N_{a,T}^-\) from~\eqref{eq:Nsub-lower:contradiction}, together with the fact that \(\sqrt{2\bonus\log T}-\sqrt{2\log(2KT^3)}-\sqrt{C_1\log T}\ge0\) for all sufficiently large \(T\), is used in the second line of~\eqref{eq:Ua:lower:1}.
		By the definitions of $N_{a,T}^-$ and $N_{a,T}^\star$ in~\eqref{eq:Nstar-def}, we have
		\begin{equation}\label{eq:Nstar-recall}
		(1-\varepsilon_T)N_{a,T}^\star - 1 \le N_{a,T}^-\le (1-\varepsilon_T)N_{a,T}^\star
		\end{equation}
		with $N_{a,T}^\star=2\bonus\log T/\Delta_a^2$,
		which implies that
		\begin{equation}\label{eq:Ua:lower:2}
		\sqrt{\frac{2\bonus\log T}{N_{a,T}^-}}\ge \frac{\Delta_a}{\sqrt{1-\varepsilon_T}}
		\ge \Delta_a\left(1+\frac{\varepsilon_T}{2}\right).
		\end{equation}
		The inequality \((1-\varepsilon_T)^{-1/2}\ge1+\varepsilon_T/2\) used in~\eqref{eq:Ua:lower:2} follows from the second inequality in~\eqref{eq:sqrt-ineq} in Lemma~\ref{lem:sqrt-ineq}.

		Meanwhile, by~\eqref{eq:Nstar-recall}, \(\varepsilon_T\to0\), and \(N_{a,T}^\star\to\infty\), we have \(N_{a,T}^-\ge (1-\varepsilon_T)N_{a,T}^\star-1\ge (1-\varepsilon_T)N_{a,T}^\star/2\ge N_{a,T}^\star/4\).
		Using \(N_{a,T}^\star=2\bonus\log T/\Delta_a^2\) gives
		\begin{equation}\label{eq:Ua:lower:3}
		\begin{aligned}
		\frac{\sqrt{2\log(2KT^3)}+\sqrt{C_1\log T}}{\sqrt{N_{a,T}^-}} & \le \frac{\sqrt{8\log(2KT^3)}+2\sqrt{C_1\log T}}{\sqrt{N_{a,T}^\star}} \\
		& = \frac{(\sqrt{8\log(2KT^3)}+2\sqrt{C_1\log T})\Delta_a}{\sqrt{2\bonus\log T}} \\
		& < \frac{\Delta_a \varepsilon_T}{8}
		\end{aligned}
		\end{equation}
		because the middle expression in~\eqref{eq:Ua:lower:3} is \(O(\Delta_a/\sqrt{\bonus})=o(\Delta_a\varepsilon_T)\).
		Combining~\eqref{eq:Ua:lower:1},~\eqref{eq:Ua:lower:2}, and~\eqref{eq:Ua:lower:3} gives
		\begin{equation}
			\Upsilon_{a,t+1}\ge \mu_a+\Delta_a\left(1+\frac{\varepsilon_T}{2}\right)-\frac{\Delta_a\varepsilon_T}{8}
			=\mu^\star+\frac{3\Delta_a\varepsilon_T}{8}.
			\label{eq:Ia-lower-final}
		\end{equation}
		Finally,~\eqref{eq:opt-index-upper} and~\eqref{eq:Ia-lower-final} imply \(\Upsilon_{a,t+1}>\mu^\star+3\Delta_a\varepsilon_T/8>\max_{i\in \mathcal{S}^\star}\Upsilon_{i,t+1}\),
		which proves~\eqref{eq:subopt-dominates-lastblock}.

		\paragraph{Contradiction based on~\eqref{eq:subopt-dominates-lastblock}.} If~\eqref{eq:subopt-dominates-lastblock} holds for every $t\in\{T-\lfloor T/(8m)\rfloor,\dots,T-1\}$, then the maximizer at time $t+1$ cannot belong to $\mathcal{S}^\star$,
		so the chosen arm is suboptimal at every time in this block with size $\lfloor T/(8m)\rfloor$. In other words, \(N_{\mathrm{sub}}(T)\ge \lfloor T/(8m)\rfloor\), contradicting~\eqref{eq:Nsub-less-ell}.
		This completes the proof of Lemma~\ref{lem:subopt-lower}.
	\end{proof}

	Back to the proof of Theorem~\ref{thm:main:2}(b). Fix $a\notin \mathcal{S}^\star$.
	On $\mathcal{G}_T$ and for large $T$, Lemma~\ref{lem:subopt-negligible} and Lemma~\ref{lem:subopt-lower} give
	\[
	N_{a,T}^-\le N_{a,T}\le N_{a,T}^+.
	\]
	Since $N_{a,T}^\pm=(1\pm \varepsilon_T)N_{a,T}^\star+O(1)$, \(\varepsilon_T\to 0\), and \(N_{a,T}^\star\to\infty\), it follows that on $\mathcal{G}_T$,
	\[
	N_{a,T}/N_{a,T}^\star\to 1 \qquad \text{ with } \qquad N_{a,T}^\star=2\bonus\log T/\Delta_a^2.
	\]
	Because $\mathbb{P}(\mathcal{G}_T)\to 1$ (Lemma~\ref{lem:mean-bonus-good-event}),
	we conclude
	\[
	N_{a,T}/N_{a,T}^\star\xrightarrow{\mathbb{P}}1 \qquad \text{ with } \qquad N_{a,T}^\star=2\bonus\log T/\Delta_a^2.
	\]
	We finish the proof of Theorem~\ref{thm:main:2}(b).
\end{proof}

\section{Technical Lemmas}

\begin{lemma}[Standard normal tail bounds]\label{lem:normal-tail}
    For the standard normal density function $\phi$, cumulative distribution function $\Phi$, and upper tail function $\overline{\Phi}$ defined in~\eqref{eq:def:gaussian}, the following properties hold.

    \begin{enumerate}
    \item \textbf{Symmetry.} For all $x\in\mathbb R$,
    \[
    \Phi(-x)=\overline\Phi(x).
    \] \label{lem:normal-tail-1}
    \item \textbf{Mills ratio upper bound.} For all $x>0$,
    \[
    \overline\Phi(x)\le \frac{\phi(x)}{x}
    =\frac{1}{x\sqrt{2\pi}}\exp\{-x^{2}/2\}.
    \] \label{lem:normal-tail-2}
    \item \textbf{Mills ratio lower bound.} For all $x>0$,
    \[
    \overline\Phi(x)\ge \frac{x}{1+x^{2}}\phi(x)
    =\frac{x}{(1+x^{2})\sqrt{2\pi}}\exp\{-x^{2}/2\}.
    \]
    In particular, for all $x\ge 1$,
    \[
    \overline\Phi(x)\ge \frac{\phi(x)}{2x}
    =\frac{1}{2x\sqrt{2\pi}}\exp\{-x^{2}/2\}.
    \] \label{lem:normal-tail-3}
    \item \textbf{Crude sub-Gaussian tail bound.} For all $x\ge 0$,
    \[
    \overline\Phi(x)\le \exp\{-x^{2}/2\}, \quad
    \Phi(-x)\le \exp\{-x^{2}/2\}.
    \] \label{lem:normal-tail-4}
    \end{enumerate}
    \end{lemma}

\begin{lemma}[Time-uniform bound for thinned sub-Gaussian martingales \citep{howard2021time}]
\label{lem:TUB}
Let \((X_t,\mathcal H_t)_{t\ge1}\) be a martingale difference sequence satisfying $\mathbb E[\exp\{\lambda X_t\}\mid\mathcal H_{t-1}] \le \exp\{\lambda^2/2\}$ for all $\lambda\in\mathbb R.$ Let \(I_t\in\{0,1\}\) be \(\mathcal H_{t-1}\)-measurable and write \(V_t:=\sum_{i=1}^t I_i\). Then there exists an absolute constant \(c>0\) such that, for any \(\delta\in(0,1)\), with probability at least \(1-\delta\), the following holds simultaneously for all \(t\) with \(V_t\ge3\): \[ \left|\frac{1}{V_t}\sum_{i=1}^t I_iX_i\right| \le c\,\sqrt{\frac{1}{V_t}\big(\log\log V_t+\log(1/\delta)\big)}. \]
\end{lemma}

\begin{lemma}[Tail log-law for variance-scale geometric sums]\label{lem:geom-tail-sum-loglaw}
	Let \(\variance\to\infty\) and \(L_T:=\log(T/\variance)\to\infty\).  Let \(r_T\ge0\) be deterministic with \(r_T=o(L_T)\).  Let \(m_T\) and \(n_T\) be integers satisfying
		\[
		n_T=\Theta(\variance L_T),\qquad
		m_T=o(n_T),\qquad
		m_T/\variance\to\infty.
		\]
		For \(m_T\le k\le n_T\), let \(G_{k,T}\) be independent geometric random variables on the positive integers \(\mathbb N^+\) with success probabilities \(p_{k,T}\in(0,1)\), namely \(\mathbb P(G_{k,T}=\ell)=(1-p_{k,T})^{\ell-1}p_{k,T}\) for \(\ell \in \mathbb N^+\).
		Suppose that for some fixed \(\alpha>0\), uniformly over \(m_T\le k\le n_T\),
		\[
		-\log p_{k,T}=
		\alpha\frac{k}{\variance}
		+O\left(1+\log\frac{k}{\variance}+r_T\right).
		\]
		Then, as \(T\to\infty\),
		\begin{equation}\label{eq:geom-tail-corrected}
			\log\left(\sum_{k=m_T}^{n_T}G_{k,T}\right)
			=
			\log \variance+\alpha\frac{n_T}{\variance}+o_{\mathbb P}(L_T).
		\end{equation}
		Equivalently, \(\log(\variance^{-1}\sum_{k=m_T}^{n_T}G_{k,T})=\alpha n_T/\variance+o_{\mathbb P}(L_T)\).
\end{lemma}

\begin{proof}[Proof of Lemma~\ref{lem:geom-tail-sum-loglaw}]
		Since \(n_T/\variance=\Theta(L_T)\), the assumption implies \(\log(k/\variance)=O(\log L_T)\) uniformly on the relevant range, and \(\log L_T+r_T=o(L_T)\).  For the upper bound, using \(\mathbb{E}[G_{k,T}]=1/p_{k,T}\),
			\begin{equation}\label{eq:geom-tail-reciprocal-sum}
			\sum_{k=m_T}^{n_T}\frac1{p_{k,T}}
			\le
				\exp\{o(L_T)\}\cdot\sum_{k=m_T}^{n_T}\exp\left\{\alpha\frac{k}{\variance}\right\}
			\le
			C\variance\exp\left\{\alpha\frac{n_T}{\variance}+o(L_T)\right\}.
			\end{equation}
			Using \(\mathbb E[\sum_{k=m_T}^{n_T}G_{k,T}]=\sum_{k=m_T}^{n_T}1/p_{k,T}\),~\eqref{eq:geom-tail-reciprocal-sum} gives
			\begin{equation}\label{eq:geom-tail-mean-upper}
			\mathbb{E}\left[\sum_{k=m_T}^{n_T}G_{k,T}\right]
			\le
			C\variance\exp\left\{\alpha\frac{n_T}{\variance}+o(L_T)\right\}.
			\end{equation}
		Markov's inequality gives
		\[
		\log\left(\sum_{k=m_T}^{n_T}G_{k,T}\right)
		\le
		\log \variance+\alpha\frac{n_T}{\variance}+o_{\mathbb P}(L_T).
		\]
		For the lower bound, fix \(\delta>0\).  If \(\log \variance\le \delta L_T/2\), then the last variable alone suffices: the assumed bound gives
		\[
		p_{n_T,T}\le \exp\left\{-\alpha\frac {n_T}{\variance}+\delta L_T/4\right\},
		\]
		so, using \(\mathbb P(G_{n_T,T}\le x)\le xp_{n_T,T}\),
		\[
		\mathbb P\left(G_{n_T,T}\le \exp\left\{\alpha\frac {n_T}{\variance}-\delta L_T/2\right\}\right)\to0.
		\]
		On this high-probability event,
		\[
		\sum_{k=m_T}^{n_T}G_{k,T}\ge G_{n_T,T}
		\ge \exp\left\{\alpha\frac {n_T}{\variance}-\delta L_T/2\right\}
		\ge \exp\left\{\log \variance+\alpha\frac {n_T}{\variance}-\delta L_T\right\}.
		\]
		If \(\log \variance>\delta L_T/2\), set \(J_T:=\{n_T-\lfloor\variance\rfloor+1,\ldots,n_T\}\).  Since \(m_T=o(n_T)\) and \(n_T=\Theta(\variance L_T)\), we have \(J_T\subseteq\{m_T,\ldots,n_T\}\) for all large \(T\).  For \(k\in J_T\), the assumption gives \(p_{k,T}^{-1}\ge \exp\{\alpha n_T/\variance-o(L_T)\}\) and \(p_{k,T}^{-2}\le \exp\{2\alpha n_T/\variance+o(L_T)\}\).  Therefore
		\[
		\mathbb{E}\left[\sum_{k\in J_T}G_{k,T}\right]
		\ge
		\exp\left\{\log \variance+\alpha\frac {n_T}{\variance}-o(L_T)\right\}.
		\]
		Moreover, using \(\operatorname{Var}(G_{k,T})\le p_{k,T}^{-2}\),
		\[
		\frac{\operatorname{Var}\left(\sum_{k\in J_T}G_{k,T}\right)}
		{\left(\mathbb{E}\left[\sum_{k\in J_T}G_{k,T}\right]\right)^2}
		\le \frac{\exp\{o(L_T)\}}{\variance}\to0
		\]
		in the present case, because \(\log\variance>\delta L_T/2\).  Chebyshev's inequality gives the matching lower bound.  Since \(\delta\) is arbitrary,~\eqref{eq:geom-tail-corrected} follows.
\end{proof}

\begin{lemma}
	\label{lem:sqrt-ineq}
	For every $x\in[0,1]$, the first inequality below holds; for every
	\(x\in[0,1)\), the second inequality below holds:
	\begin{equation}
		\frac{1}{\sqrt{1+x}}\le 1-\frac{x}{4},
		\qquad
		\frac{1}{\sqrt{1-x}}\ge 1+\frac{x}{2}.
		\label{eq:sqrt-ineq}
	\end{equation}
\end{lemma}

\begin{proof}[Proof of Lemma~\ref{lem:sqrt-ineq}]
	For the first inequality, consider $f(x)=(1+x)^{-1/2}$ on $[0,1]$. The function is convex and satisfies
	$f(0)=1$ and $f(1)=1/\sqrt{2}\le 3/4$. Therefore $f(x)$ lies below the chord from $(0,1)$ to $(1,3/4)$, i.e. $f(x)\le 1-x/4$.
	For the second inequality, let $g(x)=(1-x)^{-1/2}$ on $[0,1)$. The function is convex and $g(0)=1$, $g'(0)=1/2$,
	so $g(x)\ge g(0)+g'(0)x=1+x/2$ by convexity.
\end{proof}

\end{document}